\documentclass[10pt]{article} 
\usepackage[preprint]{tmlr}
\usepackage{graphicx}
\usepackage{subcaption}
\usepackage{amsmath,amsthm,thmtools}
\usepackage{accents}
\usepackage{multirow}
\usepackage{booktabs}

\usepackage{algorithm}
\usepackage{etoolbox}

\let\classAND\AND
\let\AND\relax
\usepackage{algorithmic}

\let\AND\classAND
\AtBeginEnvironment{algorithmic}{\let\AND\algoAND}

\usepackage{amsmath,amsfonts,bm}

\def\eqref#1{equation~\ref{#1}}

\def\1{\bm{1}}

\def\rb{{\textnormal{b}}}
\def\rc{{\textnormal{c}}}

\def\rr{{\textnormal{r}}}

\def\rv{{\textnormal{v}}}
\def\rw{{\textnormal{w}}}

\def\rz{{\textnormal{z}}}

\def\rvg{{\mathbf{g}}}

\def\rvu{{\mathbf{i}}}

\def\rvu{{\mathbf{u}}}

\def\rvx{{\mathbf{x}}}
\def\rvy{{\mathbf{y}}}

\def\ervu{{\textnormal{u}}}

\def\vmu{{\bm{\mu}}}

\def\vg{{\bm{g}}}

\def\vk{{\bm{k}}}

\def\vu{{\bm{u}}}
\def\vv{{\bm{v}}}

\def\vx{{\bm{x}}}
\def\vy{{\bm{y}}}
\def\vz{{\bm{z}}}

\def\evu{{u}}
\def\evv{{v}}

\def\evz{{z}}

\def\mI{{\bm{I}}}

\def\mK{{\bm{K}}}

\def\mX{{\bm{X}}}

\DeclareMathAlphabet{\mathsfit}{\encodingdefault}{\sfdefault}{m}{sl}
\SetMathAlphabet{\mathsfit}{bold}{\encodingdefault}{\sfdefault}{bx}{n}

\def\sR{{\mathbb{R}}}

\def\sU{{\mathbb{U}}}
\def\sV{{\mathbb{V}}}

\def\sX{{\mathbb{X}}}
\def\sY{{\mathbb{Y}}}

\newcommand{\E}{\mathbb{E}}

\newcommand{\R}{\mathbb{R}}

\DeclareMathOperator*{\argmax}{arg\,max}

\def\sXd{{\sX_\mathrm{disc}}}
\def\sYd{{\sY_\mathrm{disc}}}
\def\sUmc{{\sU_\mathrm{MC}}}
\def\sVmc{{\sV_\mathrm{MC}}}
\def\sVmcbase{{\sV_\mathrm{MC}^\mathrm{base}}}

\let\oldforall\forall
\renewcommand{\forall}{\oldforall \, }
\let\oldexist\exists
\renewcommand{\exists}{\oldexist \: }

\newlength{\dhatheight}
\newcommand{\doublehat}[1]{%
    \settoheight{\dhatheight}{\ensuremath{\hat{#1}}}%
    \addtolength{\dhatheight}{-0.25ex}%
    \hat{\vphantom{\rule{1pt}{\dhatheight}}%
    \smash{\hat{#1}}}}

\usepackage{hyperref}
\usepackage{url}
\usepackage[capitalize,noabbrev]{cleveref}

\theoremstyle{plain}
\newtheorem{theorem}{Theorem}[section]
\newtheorem{proposition}[theorem]{Proposition}
\newtheorem{lemma}[theorem]{Lemma}

\theoremstyle{definition}
\newtheorem{definition}[theorem]{Definition}

\theoremstyle{remark}
\newtheorem{remark}[theorem]{Remark}
\newtheorem*{claim}{Claim}

\title{Bayesian Optimization for\\ Non-Convex Two-Stage Stochastic Optimization Problems}

\author{\name Jack M. Buckingham \email jack.buckingham@warwick.ac.uk \\
      \addr EPSRC Centre for Doctoral Training in Mathematics for Real-World Systems\\
      Mathematics Institute\\
      University of Warwick, UK
      \AND
      \name Ivo Couckuyt \email ivo.couckuyt@ugent.be \\
      \addr Faculty of Engineering and Architecture\\
      Ghent University - imec, Belgium
      \AND
      \name Juergen Branke \email juergen.branke@wbs.ac.uk\\
      \addr Warwick Business School\\
      University of Warwick, UK}

\def\month{MM}  
\def\year{YYYY} 
\def\openreview{\url{https://openreview.net/forum?id=XXXX}} 

\begin{document}

\maketitle

\begin{abstract}
Bayesian optimization is a sample-efficient method for solving expensive, black-box optimization problems. Stochastic programming concerns optimization under uncertainty where, typically, average performance is the quantity of interest.
In the first stage of a two-stage problem, here-and-now decisions must be made in the face of uncertainty, while in the second stage, wait-and-see decisions are made after the uncertainty has been resolved.
Many methods in stochastic programming assume that the objective is cheap to evaluate and linear or convex.
We apply Bayesian optimization to solve non-convex, two-stage stochastic programs which are black-box and expensive to evaluate as, for example, is often the case with simulation objectives.
We formulate a knowledge-gradient-based acquisition function to jointly optimize the first- and second-stage variables, establish a guarantee of asymptotic consistency, and provide a computationally efficient approximation.
We demonstrate comparable empirical results to an alternative we formulate with fewer approximations, which alternates its focus between the two variable types, and superior empirical results over the state of the art and the standard, na\"ive, two-step benchmark.
\end{abstract}

\section{Introduction}\label{sec:introduction}

Many optimization problems involve decisions that must be made with incomplete information. This uncertainty can be handled in several ways depending on the risk appetite of the decision maker, ranging from optimizing for the worst-case outcome, to optimizing for average performance. The Bayesian optimization community refers to all these approaches as \emph{robust} optimization problems; however, the operations research community mostly reserves the term robust optimization for worst-case settings. Optimization of the average case is the default for \emph{stochastic programming} \citep{birge2011stochprog,shapiro2021stochprog}.

The notion of a two-stage decision process was first introduced to stochastic programming by \citet{dantzig1955linear}, who considered linear programs. In contrast, we consider non-convex problems but the setting is the same.
We wish to maximize some objective in expectation over a stochastic environment. In the first stage of the optimization, we choose the \emph{here-and-now} variables or \emph{fixed design} before the uncertainty in the environment is resolved. In the second stage, we choose the \emph{wait-and-see} or \emph{adjustable variables} once the specific environment has been revealed. The aim is to learn both an optimal fixed design and a control policy mapping environments to adjustable variables to maximize the expected objective. This is illustrated in \cref{fig:example-problem}.

\begin{figure}
    \centering
    \includegraphics[width=\textwidth]{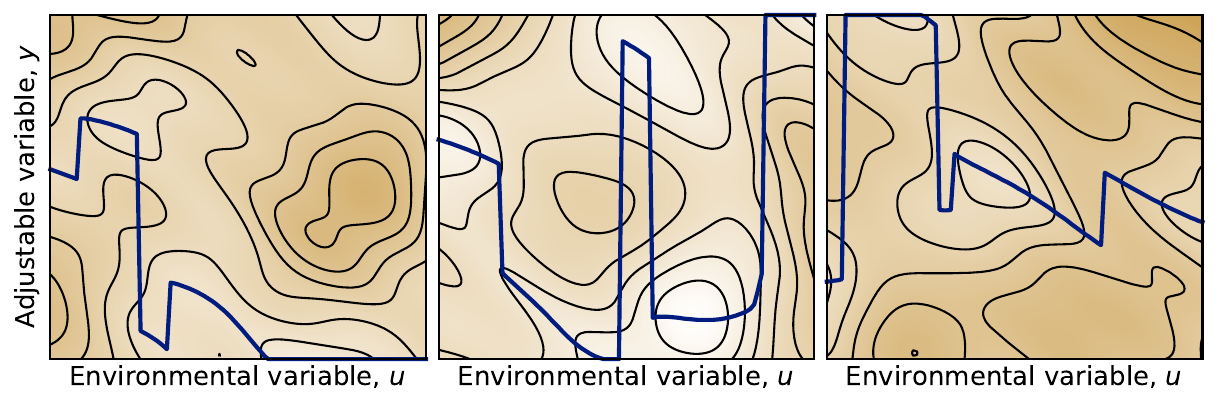}
    \caption{A contour plot for a non-convex two-stage stochastic optimization problem at three different values of the fixed design, \(\vx\). The solid blue line shows the optimal second-stage decision, \(\vy = \vg(\vu)\). The aim is to find the fixed design with the best objective value after taking the expectation over the environmental variable.}
    \label{fig:example-problem}
\end{figure}

In many cases, including cases from simulation optimization, the objective is a black box. That is, gradients are typically not available and the simulation is so complicated that there are no insights exploitable to make the optimization easier. The only available tool is to evaluate the objective at specific inputs.
Furthermore, these evaluations are often expensive. For example, they may require running time-consuming fluid dynamics or finite element analysis simulations, or be a physical experiment that not only takes time but consumes costly resources.
This variation has so far received relatively little attention in the literature and to the best of the authors' knowledge, \citep{xie2021globallocal} is the only other work to address it.

Examples of these problems arise throughout operations research and engineering, including in the design of wind farms \citep{chen2022windfarm} and the management of power stations \citep{phan2014renewable}.
Similar problems are popular in optimal control, specifically in control co-design, where the objective is to jointly optimize a fixed plant design and a control signal in the presence of a \emph{feedback loop}.
Example applications include hybrid electric vehicles \citep{yan2020hybridbus,xun2022hybridtruck}, batteries \citep{cui2021batteries} and wind turbines \citep{cui2021wind}.

Since evaluations of the objective are expensive, it pays to put more computational effort into reducing the number of evaluations required compared with a typical stochastic approximation approach with numerical gradients \citep[Section 8.2]{shapiro2021stochprog}.
Bayesian optimization is a sample-efficient, global optimization technique designed to tackle such problems. A probabilistic model of the objective -- usually a Gaussian process -- is maintained, and an acquisition function is used to measure the value of collecting a sample at a proposed input.

We can na\"ively apply Bayesian optimization to solve two-stage stochastic programs, by first optimizing the wait-and-see control policy for an arbitrary value of the fixed design, then switching to optimize the fixed design for the control policy found in the first phase. The two phases could be alternated until convergence. However, each phase requires a new initial sample, waiting for each step to converge is wasteful and slow, and the iteration is not guaranteed to converge to the global optimum. There is therefore potential to improve greatly upon this by jointly optimizing both sets of variables.

In this work, we formulate the knowledge gradient acquisition function for two-stage stochastic optimization problems without coupled constraints. We establish asymptotic consistency of the estimators of the maximum of the expected performance associated with the fixed design and control policy which would be recommended with this acquisition function at each iteration, under the assumption that the objective is drawn from a Gaussian process with known mean and covariance function. Further, we provide a computationally efficient approximation scheme to optimize this acquisition function and present an alternative algorithm that alternates its focus between the here-and-now and wait-and-see variables, making fewer approximations. We demonstrate empirically that both of these perform similarly on a suite of synthetic and real-world problems, with and without observation noise, and significantly outperform the standard `two-step' algorithm which alternates between fully optimizing the fixed design and fully optimizing the control policy.
Increasing the dimension and decreasing the intrinsic length scales of the fixed and adjustable variables makes different aspects of the problem harder and we investigate the effect of this on the three algorithms. Finally, we establish superior performance of the joint algorithm over the state of the art on a real-world supply chain example.

\subsection{Problem statement}\label{sec:problem-statement}

Let \(\sX \subset \R^{d_x}\), \(\sY \subset \R^{d_y}\) and \(\sU \subset \R^{d_u}\) denote the spaces of fixed, adjustable and environmental variables respectively. Let \(h: \sX \times \sY \times \sU \to \R\) be an unknown function, say given by a simulation, and suppose that we can make expensive observations of \(h\) at points \((\vx, \vy, \vu) \in \sX \times \sY \times \sU\) of our choosing, possibly corrupted by Gaussian noise,

\begin{equation}\label{eq:obs-model}
    \rv = h(\vx, \vy, \vu) + \varepsilon, \quad \varepsilon \sim \mathcal{N}(0, \sigma^2).
\end{equation}

This work tackles the problem of finding fixed parameters \(\vx \in \sX\) and a control policy \(\vg: \sU \to \sY\) for the adjustable parameters which maximize the expectation of \(h\) over the uncertain environment, \(\rvu\). That is, it solves
\begin{equation}\label{eq:problem-statement}
    \argmax_{\substack{\vx \in \sX \\ \vg: \sU \to \sY}} \E_\rvu \bigl[ h(\vx, \vg(\rvu), \rvu) \bigr].
\end{equation}
We assume here that the distribution of \(\rvu\) is known, say from historical observations, but we make no assumptions on what form that distribution takes.

Note that the problem can also be phrased classically as a stochastic bi-level optimization problem, by moving the optimization of \(\vg\) inside the expectation,
\begin{subequations}\label{eq:problem-statement-bilevel}
\begin{gather}
    \max_{\vx \in \sX} \E_\rvu \bigl[ h^*(\vx, \rvu) \bigr] \label{eq:problem-statement-bilevel-outer} \\
    \text{where}\quad \forall \vx \in \sX\, \forall \vu \in \sU \;
        h^*(\vx, \vu) = \max_{\vy \in \sY} h(\vx, \vy, \vu). \label{eq:problem-statement-bilevel-inner}
\end{gather}
\end{subequations}

Any constraints on \(\vx\) and \(\vy\) are assumed part of the definition of \(\sX\) and \(\sY\). Except for the supply chain problem in \cref{sec:experiments-supply-chain}, we do not consider coupled constraints where the space \(\sY\) of admissible \(\vy\) depends on \(\vx\), nor when either of \(\sX\) and \(\sY\) depend on \(\rvu\).

Throughout this work, we use an italic font, \(\vu\), to represent deterministic quantities and a straight font, \(\rvu\), for random vectors. The exception to this are Greek letters, such as the observation noise \(\varepsilon\), and the Gaussian process, \(f\), to be introduced later.

\subsection{Related work}

For worst-case robustness, the two-stage problem is known as \emph{adjustable robust optimization} \citep{bental2004adjustable, poursoltani2021adjustable}. Solution methods generally focus on finding explicit functional forms for the control policy (or `decision rule') and objectives are usually linear. \citet{yanikoglu2019adjustable} observe that the case of nonlinear objectives is under-studied.
A highly related concept is that of \emph{recoverable robustness} \cite{liebchen2009recoverablerobustness}, where the aim is to optimize a deterministic objective and find a computationally feasible `recovery algorithm' to modify the solution to ensure that scenario-dependent constraints are satisfied. The recovery algorithm takes the place of the control policy in adjustable robustness. Again, most solution methods target linear problems. A third concept is that of \emph{distributional robustness}, where the objective is now an expectation over the environmental variables, but the distribution of these variables is not well understood and the worst case over plausible distributions is considered. Again, the literature tends to focus on linear problems \citep{bertsimas2021samplerobust}.

The problem of \emph{Control Co-Design (CCD)} is very similar to two-stage stochastic optimization, and tackles non-convex problems. The role of the uncertain variables \(\rvu\) is replaced by a state vector evolving deterministically according to a difference or differential equation. The search over adjustable variables, \(\rvy\), is replaced with a search for an optimal control signal as a function of time, which can affect the state evolution. \citet{wauters2022comparative} apply Bayesian optimization in the outer loop of a nested optimization scheme to solve the deterministic CCD problem. \citet{azad2023uccdreview} review different ways of introducing uncertainty into the CCD problem, the closest of which to the problem tackled in this paper is \emph{stochastic in expectation, uncertain CCD (SE-UCCD)}. \citet{azad2022seuccd} use a nested formulation along with Monte Carlo simulation and generalized polynomial chaos expansions to solve an SE-UCCD problem. The authors note a nested scheme is necessary to decouple the optimization of the control variables for different realizations of the uncertainty.

Non-convex, two-stage stochastic optimization problems are a class of uncertain bi-level optimization where the inner and outer objectives are identical. Many more sources of uncertainty are possible in this setup beyond the uncertainty from \(\rvu\), arising when the first and second level decision makers do not know how each other will react \citep{beck2023uncertainbilevelreview}. Bayesian optimization with the Upper Confidence Bound (UCB) acquisition function has been applied to bi-level problems in \citep{fu2024bbo}, while a surrogate-assisted evolutionary algorithm was employed in \citep{sinha2022bilevel} for bi-level problems.

\citet{salomon2019activeRobustness} explore non-convex, two-stage optimization problems under uncertainty as \emph{active robust optimization}. Various risk measures are considered, including expectation, worst-case performance, value-at-risk, and target-based. Both single-objective and multi-objective problems are considered, and evolutionary algorithms are proposed for the multi-objective case.

Finally, two-stage stochastic optimization problems are closely related to standard stochastic optimization (the first-stage decision) and multi-task optimization (the second-stage decision), and Bayesian optimization has been widely applied in both cases.

Stochastic, or robust, Bayesian optimization aims to solve \(\max_{\vx \in \sX} \E_\rvu[h(\vx, \rvu)]\) and has received a lot of attention, with authors proposing acquisition strategies based on expected improvement \citep{williams2000sequential, groot2010robust} and knowledge gradient \citep{pearce2017robustkg, pearce2022crn, toscanopalmerin2022integrands} to name a few.
The related problem of robustness to perturbed or noisy inputs has also received much attention \citep{nogueira2016unscented, beland2017uncertain, oliveira2019uncertaininputs, frohlich2020robust, le2021robust, le2024robust}, as has the study of other risk measures such as value-at-risk \citep{cakmak2020riskmeasures,daulton2022robustmobo} and mean-variance trade-offs \citep{qing2022spectral}. Generally, it is assumed that the distribution of the environmental variables is known, but recent work exists to infer this distribution \citep{huang2024stochastic}.

In multi-task Bayesian optimization, the goal is to solve a finite set or continuum of optimization problems, \(\max_\vy h(\vy, \vu)\), parameterized by \(\vu\) as in \cref{eq:problem-statement-bilevel-inner}.
Acquisition functions based on entropy \citep{swersky2013multitask}, expected improvement \citep{tesch2011robust, ginsbourger2014multitask} and knowledge gradient \citep{pearce2018multitask,ding2022multitask} have been proposed previously.

The closest work to our own is that of \citet{xie2021globallocal} who tackle expensive, non-convex, two-stage stochastic optimization problems by adopting a nested approach. They use a Gaussian process surrogate model both on the inner level to optimize the control policy \(\vy = \vg(\vx)\) and on the outer level to optimize the fixed design \(\vx\). The problem setting differs slightly from the one considered in \cref{eq:problem-statement} in that \citet{xie2021globallocal} only aim to find the optimal first-stage decision \(\vx\), while we also seek to return an optimal control policy \(\vy = \vg(\vx)\). Nonetheless, the setting is sufficiently similar for us to provide a numerical comparison in \cref{sec:experiments}.

\section{Background} \label{sec:background}

In this section, we introduce the relevant background material from the field of Bayesian optimization for the case of single-stage, deterministic problems, \(h : \sX \to \sR\). These will be generalized to two-stage stochastic problems in \cref{sec:algo-acqf}.

\subsection{Gaussian processes}

A \emph{Gaussian process (GP)} on \(\sX \subset \R^d\) is a collection of random variables \((f(\vx) : \vx \in \sX)\) such that for every finite subset \(\{\vx^1, \dots, \vx^n\} \subset \sX\), the variables \(f(\vx^1), \dots, f(\vx^n)\) are jointly normal. Their mean and covariance are given by a mean function \(\mu : \sX \to \R\) and a covariance function, or kernel, \(k : \sX \times \sX \to [0, \infty)\). These functions completely determine the finite-dimensional distributions of the GP, and we write \(f \sim \mathcal{GP}(\mu, k)\).

Suppose we have made \(n\) noisy observations of the Gaussian process, \(v^1 = f(\vx^1) + \varepsilon^1, \dots, v^n = f(\vx^n) + \varepsilon^n\) for independent \(\varepsilon^1, \dots, \varepsilon^n \sim \mathcal{N}(0, \sigma^2)\).
Then the posterior distribution of \(f\) conditional on these observations is again a Gaussian process whose mean and covariance functions we shall denote by \(\mu^n\) and \(k^n\).

When used to model an unknown function, the prior mean and covariance functions of a GP should be chosen to fit the function being modeled. Popular choices for the covariance function include the squared exponential kernel and Mat\'ern family of kernels. These are parameterized by a length scale for each dimension and an output scale. The mean function is often chosen to be zero or a constant. The length scales, output scale, and constant mean, together with the variance of the observation noise are known as hyperparameters of the GP. They can be fitted by placing prior distributions on the hyperparameters and using maximum a posteriori (MAP) estimates. \citet{rasmussen2006gpml} provide an excellent introduction to Gaussian processes for machine learning.

\subsection{Bayesian optimization}

Bayesian optimization is a technique for efficiently solving expensive, black-box, global optimization problems, \(h : \sX \to \sR\). Initially, \(h\) is evaluated on a space-filling design such as a scrambled Sobol' sequence \citep{sobol1967sobolseq} or Latin hypercube (LHC) \citep{mckay2000lhc}. Afterward, additional samples are collected either one by one or in small batches, focusing on the more promising regions of the input space.

At each step, a Bayesian model is placed on the objective function and an acquisition function, \(\alpha : \sX \to \R\), is optimized to select the next sample location, \(\vx^{n+1} \in \sX\). Typically, it is assumed that \(\sX\) is compact and \(\alpha\) is continuous at all but finitely many locations, so that it has a well-defined maximum in \(\sX\).

The role of \(\alpha\) is to balance exploration of regions of the input space with high uncertainty, and exploitation of regions already known to contain good values.
Most commonly, the Bayesian model consists of a GP prior, along with an observation model with additive Gaussian noise such as \cref{eq:obs-model}.

When the evaluation budget is exhausted a recommendation is returned, for which common choices are the best input evaluated so far (a risk-averse approach) or the maximum of the posterior mean of the GP (a risk-neutral approach).

\Cref{alg:bo-basic} summarizes the Bayesian optimization algorithm to maximize a function \(h: \sX \to \sR\) using the knowledge gradient acquisition function. The review papers by \citet{shahriari2016bayesopt} and \citet{frazier2018tutorial}, and the book by \citet{garnett2023bayesian} all provide a good introduction to Bayesian optimization.

\subsection{Knowledge gradient}

A popular acquisition function capable of handling observation noise and partial information is the knowledge gradient (KG) \citep{frazier2009kg, scott2011ctskg}. The grounding principle is that at the end of the data collection process, the recommended solution will be the maximum of the posterior mean. Write \(\mu^n(\vx) = \E[f(\vx) \,|\, \mathcal{D}^n]\) for the posterior mean of the GP conditional on the \(n\) observations made during the optimization. The final recommendation is
\begin{equation}\label{eq:bg-knowledge-gradient-recommendation}
    \vx^{*n} \in \argmax_{\vx \in \sX} \mu^n(\vx).
\end{equation}

The knowledge gradient acquisition function is then constructed to be one-step Bayes-optimal. That is, maximizing it corresponds to maximizing the expected value at the recommended solution, \(\E[f(\rvx^{*(n+1)}) \,|\, \mathcal{D}^n]\), where here the expectation is taken over both \(f\) and the next observation. Let \(\rv^{n+1} = f(\vx^{n+1}) + \varepsilon^{n+1}\) be the random variable representing the next observation. The knowledge gradient is given by
\begin{equation}\label{eq:bg-knowledge-gradient}
    \alpha_\text{KG}^n(\vx) = \E_{\rv^{n+1}}\biggl[
        \max_{\rvx' \in \sX} \mu^{n+1}(\rvx')
        \,\bigg|\,
        \vx^{n+1} = \vx
    \biggr]
        - \max_{\vx' \in \sX} \mu^n(\vx').
\end{equation}
This is the expected change in the maximum of the posterior mean after making an additional observation at \(\vx\).

\begin{algorithm}
\caption{Bayesian optimization using knowledge gradient}
\label{alg:bo-basic}
\begin{algorithmic}[1]
    \REQUIRE Initial sample size \(n_0\), evaluation budget \(n_\mathrm{tot}\)
    \STATE Evaluate \(h\) at \(n_0\) points, chosen according to a space-filling design
    \FOR {\(n = n_0\) \TO \(n_\mathrm{tot} - 1\)}
        \STATE Fit hyperparameters of GP prior, \(f \sim \mathcal{GP}(\mu, k)\), using MAP estimates and all available data
        \STATE Optimize acquisition function, \(\vx^{n+1} \gets \argmax_{\vx} \alpha_\text{KG}^n(\vx)\) \label{alg-line:bo-basic-optimizeacqf}
        \STATE Evaluate expensive function, \(v^{n+1} \gets h(\vx^{n+1}) + \varepsilon^{n+1}\)
    \ENDFOR
    \STATE Fit hyperparameters of GP prior, \(f \sim \mathcal{GP}(\mu, k)\), using MAP estimates and all available data
    \STATE Compute recommendation \(\vx^{*n_\mathrm{tot}} \gets \argmax_{\vx} \mu^{n_\mathrm{tot}}(\vx)\)
    \ENSURE Recommendation, \(\vx^{*n_\mathrm{tot}}\)
\end{algorithmic}
\end{algorithm}

\subsection{Discrete approximation of knowledge gradient}\label{sec:bg-discretekg}

Optimizing the knowledge gradient on Line \ref{alg-line:bo-basic-optimizeacqf} of \cref{alg:bo-basic} is a non-trivial task \cite{ungredda2022efficient}.
One long-standing method is to replace the search space for the inner optimizations with a finite approximation, \(\sXd \subset \sX\). Then the outer expectation can either be computed analytically \citep[Algorithm 2]{frazier2009kg} or via a quasi-Monte Carlo approximation, which is the approach taken in this paper.

Write \(\mu^{n+1}(\vx'; \vx, v)\) for the posterior mean after \(n+1\) samples, where the \((n+1)\)\textsuperscript{th} sample is \(v\) and made at location \(\vx\).
Conditional on the \(n\) observations made so far, the \((n+1)\)\textsuperscript{th} observation, \(\rv^{n+1} = f(\vx) + \varepsilon^{n+1}\), is normally distributed, \(\rv^{n+1} \,|\, (\vx^1, v^1, \dots, \vx^n, v^n) \sim \mathcal{N}(\mu^n(\vx), k^n(\vx, \vx) + \sigma^2)\).
Let \(\tilde{v}^{n+1,1}, \dots, \tilde{v}^{n+1,N_v}\) be a quasi-Monte Carlo sample of \(\rv^{n+1}\) of size \(N_v\).
Then \(\alpha_\text{KG}^n(\vx)\) is approximated by
\begin{equation} \label{eq:bg-kg-discrete}
    \hat{\alpha}_\text{KG}^n(\vx)
    = \frac{1}{N_v} \sum_{i=1}^{N_v} \max_{\vx' \in \sXd} \mu^{n+1}(\vx' ; \vx, \tilde{v}^{n+1,i}) - \max_{\vx' \in \sXd} \mu^n(\vx').
\end{equation}

Using the reparameterization trick \citep{kingma2013vae, wilson2017reparam}, we can write
\begin{equation} \label{eq:reparameterization-trick}
    \rv^{n+1} = \mu^n(\vx) + \rz \sqrt{k^n(\vx, \vx) + \sigma^2}
\end{equation}
where \(\rz \sim \mathcal{N}(0, 1)\). Fixing \(N_v\) \emph{base samples}, \(\tilde{z}^1, \dots, \tilde{z}^{N_v}\), of \(\rz\) allows us to differentiate the samples \(\{v^{n+1,i}\}_i\) in \cref{eq:bg-kg-discrete} by the input \(\vx\). Therefore, we can optimize the knowledge gradient with a deterministic, gradient-based optimizer such as multi-start L-BFGS-B. A Monte Carlo approximation of an expression of the form \(\E_\rz[m(\vx, \rz)]\) is known as sample average approximation (SAA) \citep[Section 5.1]{shapiro2021stochprog}.

Rather than using independent samples for the base samples, \(\rz\), we instead use a scrambled Sobol' sequence to make \cref{eq:bg-kg-discrete} a quasi-Monte Carlo approximation. These base samples are randomly regenerated at each iteration of the Bayesian optimization loop, to avoid over-fitting to any particular set of samples.

\section{Knowledge gradient for two-stage problems}\label{sec:algo-acqf}

We return now to the setting of two-stage stochastic optimization problems, as set out in the problem statement in \cref{sec:problem-statement}.
We consistently use the following notation. A bar such as in \(\bar{\mu}(\vx, \vg)\) will be used to indicate functions or functionals formed by taking an expectation over the environmental variables \(\rvu\), a star such as in \(\vx^{*n}\) or \(\mu^{*n}\) will be used to indicate an optimal quantity (using data available at iteration \(n\)) and a hat such as in \(\hat{\alpha}_{\text{jKG}}\), \(\hat{\vx}^{*n}\) and \(\hat{\bar{\mu}}^{*n}\) will denote an approximation. Additionally, the specific symbol \(\tilde{\vx} = (\vx, \vy, \vu)\) will be used to compactly denote a point in \(\tilde{\sX} = \sX \times \sY \times \sU\). Throughout, we will also remind the reader in words of the meanings of the important quantities.

Expected improvement \citep{jones1998ego} is possibly the most common choice of acquisition function in Bayesian optimization. It compares the proposed sample location \((\vx, \vy, \vu)\) with the best-tested value so far. However, in two-stage stochastic programming we are searching for a fixed design \(\vx \in \sX\) and control policy \(\vg : \sU \to \sY\), which maximizes the quantity \(\bar{f}(\vx, \vg) = \E_\rvu[f(\vx, \vg(\rvu), \rvu)]\) as in \cref{eq:problem-statement}. Expected improvement does not readily generalize to this problem because we are making evaluations of \(f\) rather than \(\bar{f}\) - a problem of observations giving partial information. This problem also occurs in robust and multi-task Bayesian optimization, where variants of expected improvement have been proposed \citep{williams2000sequential,groot2010robust,tesch2011robust,ginsbourger2014multitask}. The knowledge gradient acquisition function has also been used to tackle problems of partial information in multi-fidelity and multi-objective optimization, the optimization of integral expressions, and the optimization of function networks \citep{daulton2023hvkg, buckingham2023latencies, toscanopalmerin2022integrands, buathong2023functionnetworks}, and is arguably a simpler way to handle partial information. Hence, we advocate the use of knowledge gradient for two-stage stochastic Bayesian optimization.

Given a Gaussian process model \(f \sim \mathcal{GP}(\mu, k)\) of the objective \(h\) and its expectation \(\bar{f}\) over a random \(\rvu\), write \(\bar{\mu}^n(\vx, \vg) = \E[\bar{f}(\vx, \vg) \,|\, \mathcal{D}^n]\) for the posterior mean of \(\bar{f}\) after \(n\) observations of \(f\). Then, by analogy with \cref{eq:bg-knowledge-gradient-recommendation}, after \(n\) observations we would recommend a fixed design and control strategy given by
\begin{equation}\label{eq:recommendation-conceptual}
    \vx^{*n},\, \vg^{*n} = \argmax_{\substack{\vx \in \sX \\ \vg: \sU \to \sY}} \bar{\mu}^n(\vx, \vg).
\end{equation}

Writing \(\tilde{\vx} = (\vx, \vy, \vu)\) for brevity, the equivalent of \cref{eq:bg-knowledge-gradient} gives what we call the \emph{joint knowledge gradient (jKG)},
\begin{equation}\label{eq:acqf-joint-kg-conceptual}
\alpha_{\text{jKG}}^n(\tilde{\vx}) = \E_{\rv^{n+1}} \Biggl[
        \underbrace{
            \max_{\substack{\rvx' \in \sX \\ \rvg': \sU \to \sY}} \, \bar{\mu}^{n+1}(\rvx', \rvg')
        }_{\substack{\text{Maximum of posterior mean} \\ \text{after next observation, } \rv^{n+1}}}
        \Bigg|\,
        \tilde{\vx}^{n+1} = \tilde{\vx}
        \Biggr]
        \quad-\!\!\!\!\!
        \underbrace{
            \max_{\substack{\vx' \in \sX \\ \vg':\sU \to \sY}} \, \bar{\mu}^n(\vx', \vg')
        }_{\substack{\text{Current maximum of posterior} \\ \text{mean (constant wrt } \tilde{\vx} \text{)}}}
\end{equation}
where now \(\rv^{n+1} = f(\vx, \vy, \vu) + \varepsilon^{n+1}\) is the next observation (a random variable).

As has been pointed out in previous works, the second term here is constant and can be neglected when optimizing the knowledge gradient. However, as proved in \cref{thm:joint-kg-nonnegative}, \(\alpha_\text{jKG}^n\) is non-negative, which can be used to improve the choice of initial locations when optimizing it in \cref{sec:acqf-joint-approximation}. Therefore, in our experiments, we do not neglect the constant second term.

\begin{proposition}\label{thm:joint-kg-nonnegative}
    The joint knowledge gradient is non-negative,
    \begin{equation*}
        \forall \tilde{\vx} \in \sX \times \sY \times \sU \quad \alpha_\text{jKG}^n(\tilde{\vx}) \geq 0.
    \end{equation*}
\end{proposition}

This result is proved in \cref{thm:app-jointkg-nonneg} in \cref{app:theoretical-results}.

\subsection{Asymptotic consistency of the recommendations as estimators of the supremum} \label{sec:acqf-consistency}

It is desirable to know that when using the joint knowledge gradient as an acquisition function, the value, \(\bar{h}(\vx^{*n}, \vg^{*n}) = \E_\rvu[h(\vx^{*n}, \vg^{*n}(\rvu), \rvu)]\), of the solution recommended after \(n\) samples converges to the true maximum, \(\bar{h}(\vx^{*}, \vg^{*}) = \E_\rvu[h(\vx^{*}, \vg^{*}(\rvu), \rvu)]\), as the number of samples tends to infinity.
\citet{bect2019supermartingale} established almost sure convergence with knowledge gradient in the vanilla Bayesian optimization case assuming that the hyperparameters of the Gaussian process surrogate model remain fixed.
With the same assumption on the hyperparameters, \citet{toscanopalmerin2022integrands} showed convergence of the expected values in the risk-neutral, single-stage, stochastic setting if the environmental variable space \(\sU\) is finite.
The following consistency result extends that found in \citep{bect2019supermartingale} to the case of two-stage stochastic optimization. The space \(\sU\) need not be finite.

\begin{restatable}{theorem}{jointkgconsistency}\label{thm:jointkg-consistency}
    Suppose that the objective, \(h\), is a sample from a Gaussian process on \(\sX \times \sY \times \sU\) with continuous sample paths, and that \(\sX, \sY, \sU\) are compact subsets of \(\R^{d_x}, \R^{d_y}, \R^{d_u}\), respectively.
    Assume further that the surrogate \(f \sim \mathcal{GP}(\mu, k)\) has the same hyperparameters as \(h\) so that we may identify \(f = h\).
    Then, viewing the recommendations \(\rvx^{*n}\) and \(\rvg^{*n}\) as random elements depending on the initial sample, the observation noise, and the function \(f\), we have
    \begin{equation*}
        \bar{f}(\rvx^{*n}, \rvg^{*n}) \to \sup_{\substack{\rvx \in \sX \\ \rvg : \sU \to \sY}} \bar{f}(\rvx, \rvg) \qquad\text{as}\qquad n \to \infty
    \end{equation*}
    almost surely and in mean.
\end{restatable}

The proof of this theorem can be found in \cref{app:theoretical-results} and is based on the work by \citet{bect2019supermartingale} and \cite{buckingham2023latencies}.
It begins by showing that the joint knowledge gradient converges almost surely, uniformly to zero. Further, the posterior mean and covariance of \(\bar{f}\) are shown to converge uniformly to continuous limits by observing that the posterior mean and covariance of \(f\) are martingales in Banach spaces. These two facts are combined to show that the uniform limit, \(\bar{\mu}^\infty\), of the posterior means, \(\bar{\mu}^n\), has sample paths that are almost surely equal to the sample paths of \(\bar{f}\) (up to a constant). Whence, convergence of the values of the recommendations can be established.

\subsection{Efficient computation and optimization}\label{sec:acqf-joint-approximation}

Performing optimizations over a function space is challenging. By making use of the bi-level formulation introduced in \cref{eq:problem-statement-bilevel}, we obtain more tractable expressions.
Indeed, by exchanging the order of \(\E_{\rv^{n+1}}\) and \(\E_\rvu\), and moving the maximization over \(\vg\) inside the expectation over \(\rvu\) in \cref{eq:recommendation-conceptual}, we obtain
\begin{subequations}\label{eq:recommendation-bilevel}
\begin{gather}
    \vx^{*n} \in \argmax_{\vx \in \sX} \E_\rvu \Bigl[ \max_{\rvy \in \sY} \mu^n(\vx, \rvy, \rvu) \Bigr] \label{eq:recommendation-bilevel-x} \\
    \text{s.t.}\qquad \forall \vu \in \sU \quad \vg^{*n}(\vu) \in \argmax_{\vy \in \sY} \mu^n(\vx^{*n}, \vy, \vu), \label{eq:recommendation-bilevel-g}
\end{gather}
\end{subequations}
where \(\mu^n(\vx, \vy, \vu) = \E[f(\vx, \vy, \vu) \,|\, \mathcal{D}^n]\) is the posterior mean of the Gaussian process model.

Performing the same manipulation to \cref{eq:acqf-joint-kg-conceptual} gives
\begin{equation}\label{eq:acqf-joint-kg-bilevel}
    \alpha_{\text{jKG}}^n(\tilde{\vx}) = \E_{\rv^{n+1}} \Biggl[\;
        \underbrace{
            \max_{\rvx' \in \sX} \, \E_{\rvu'} \! \biggl[ \max_{\rvy' \in \sY} \mu^{n+1}(\rvx', \rvy', \rvu') \biggr]
        }_{\substack{\text{Maximum of posterior mean} \\ \text{after next observation, } \rv^{n+1}}}
        \Bigg|\,
        \tilde{\vx}^{n+1} = \tilde{\vx}
        \Biggr]
        - \underbrace{
            \max_{\vx' \in \sX} \, \E_{\rvu'} \! \biggl[ \max_{\rvy' \in \sY} \mu^n(\vx', \rvy', \rvu') \biggr]
        }_{\substack{\text{Current maximum of posterior} \\ \text{mean (constant wrt } \tilde{\vx} \text{)}}}.
\end{equation}

Running a nested optimization is computationally expensive and slow. Instead, we propose a computational strategy based on successive discretization and quasi-Monte Carlo (qMC) approximation similar to that in \cref{sec:bg-discretekg}.
Specifically, we take discrete approximations \(\sXd\) and \(\sYd\) in the inner maximizations over \(\sX\) and \(\sY\), and qMC approximations \(\sUmc\) and \(\sVmc\) for the expectations over \(\rvu' \in \sU\) and \(\rv^{n+1} = f(\tilde{\vx}) + \varepsilon^{n+1} \in \R\).
The samples of \(\rv^{n+1}\) are taken conditional on the \(n\) observations made so far.
To differentiate the samples \(v \in \sVmc\) by the proposed next query location \(\tilde{\vx}\), we use the reparameterization trick as described in \cref{eq:reparameterization-trick} to write each \(v = v(\tilde{\vx}, z)\) as a function of \(\tilde{\vx}\) and a corresponding base sample \(z \in \sVmcbase\), where \(\sVmcbase\) is a Sobol' qMC approximation of a standard normal variable.
Analogously to \cref{eq:bg-kg-discrete}, write \(\mu^{n+1}(\vx', \vy', \vu' ; \tilde{\vx}, v(\tilde{\vx}, z))\) for the posterior mean of \(f(\vx', \vy', \vu')\) after \(n+1\) samples, where the \((n+1)\)\textsuperscript{th} sample is \(v(\tilde{\vx}, z)\) and made at location \(\tilde{\vx}\). Then, writing \(N_v = |\sVmc| = |\sVmcbase|\) and \(N_u = |\sUmc|\) for the sizes of the qMC approximation sets, this gives the approximation
\begin{multline}\label{eq:acqf-joint-kg-approx}
    \hat{\alpha}_{\text{jKG}}^n(\tilde{\vx})
    = \frac{1}{N_v} \sum_{z \in \sVmcbase} \max_{\vx' \in \sXd} \frac{1}{N_u} \sum_{\vu' \in \sUmc} \max_{\vy' \in \sYd} \mu^{n+1}(\vx', \vy', \vu' ; \tilde{\vx}, v(\tilde{\vx}, z)) \\
    - \underbrace{\max_{\vx' \in \sXd} \frac{1}{N_u} \sum_{\vu' \in \sUmc} \max_{\vy' \in \sYd} \mu^n(\vx', \vy', \vu')}_{\text{constant wrt } \tilde{\vx}}.
\end{multline}

As written, this is a SAA over both the environmental variable \(\rvu\) and the base samples for the next observation \(\rv^{n+1}\). It is differentiable by the proposed query location \(\tilde{\vx}\) and thus we can optimize \(\alpha_\text{jKG}^n(\tilde{\vx})\) using a gradient-based optimizer such as multi-start L-BFGS-B.

In order to speed up the optimization of the acquisition function, it is important to cache some covariances that are not cached by default in GPyTorch at the time of writing. See \cref{app:caching} for details on the caching and \cref{app:optimization-of-acqf} for a description of the choice of starts for the multi-start L-BFGS-B.

The discretizations \(\sXd\) and \(\sYd\), and qMC approximations \(\sUmc\) and \(\sVmcbase\) are chosen randomly and independently at each iteration to avoid over-fitting to a particular discretization. Samples from a random Latin hypercube (LHC) are used for \(\sXd\) and \(\sYd\), while samples of a scrambled Sobol' sequence are used for \(\sUmc\) and \(\sVmcbase\).

When the evaluation budget is exhausted, a final recommendation must be given. This is computed by generating a further Sobol' sample \(\sU_{\mathrm{MC},\text{rec}}\) and optimizing
\begin{equation}\label{eq:recommendation-bilevel-x-approx}
    \hat{\vx}^{*n} \in \argmax_{\vx \in \sX} \sum_{\vu \in \sU_{\mathrm{MC}, \text{rec}}} \max_{\vy \in \sY} \mu^{n}(\vx, \vy, \vu).
\end{equation}
The optimal control policy is the function \(\vg^{*n}(\vu) \approx \hat{\vg}^{*n}(\vu) \in \argmax_{\vy} \mu^n(\hat{\vx}^{*n}, \vy, \vu)\).
The optimization of \(\hat{\vx}^{*n}\) is done with multi-start L-BFGS-B in a fashion similar to the one-shot knowledge gradient \citep{balandat2020botorch}, while the optimization within \(\hat{\vg}^*\) is done using single-start L-BFGS-B. Care must be taken over the choice of starts for the optimization of \(\hat{\vx}^{*n}\) since many local optima exist in the landscape. See \cref{app:optimization-of-recommendation} for full details.

\cref{alg:bo-codesign-joint-kg} shows the full Bayesian optimization loop for the joint knowledge gradient acquisition function. Hyperparameters are fitted at each step using maximum a posteriori (MAP) estimates.
Further details are given in \cref{app:further-exp-details-gpsurrogate}.

\begin{algorithm}
\caption{Bayesian optimization for two-stage stochastic programs using the joint knowledge gradient}
\label{alg:bo-codesign-joint-kg}
\begin{algorithmic}[1]
    \REQUIRE Initial sample size \(n_0\), evaluation budget \(n_\mathrm{tot}\), discretization sizes \(N_x\) and \(N_y\), qMC sample sizes \(N_u\), \(N_v\) and \(N_{u,\text{rec}}\) (powers of two)
    \STATE Evaluate \(h\) at \(n_0\) points, chosen according to a scrambled Sobol' sequence on \(\sX \times \sY \times \sU\)
    \FOR {\(n = n_0\) \TO \(n_\mathrm{tot} - 1\)}
        \STATE Fit hyperparameters of GP prior, \(f \sim \mathcal{GP}(\mu, k)\), using MAP estimates
        \STATE Generate \(\sXd \gets \mathrm{LHC}(N_x)\), \(\sYd \gets \mathrm{LHC}(N_y)\), \(\sUmc \gets \mathrm{Sobol}(N_u)\), \(\sVmcbase \gets \mathrm{Sobol}(N_v)\)
        \STATE \(\vx^{n+1}, \vy^{n+1}, \vu^{n+1} \gets \argmax_{\vx, \vy, \vu} \hat{\alpha}_\text{jKG}^n(\vx, \vy, \vu;\, \sXd, \sYd, \sUmc, \sVmcbase)\)
        \STATE Evaluate \(h(\vx^{n+1}, \vy^{n+1}, \vu^{n+1})  + \varepsilon^{n+1}\) (expensive)
    \ENDFOR
    \STATE Fit hyperparameters of GP prior, \(f \sim \mathcal{GP}(\mu, k)\), using MAP estimates
    \STATE Generate \(\sU_{\mathrm{MC}, \text{rec}} \gets \mathrm{Sobol}(N_{u,\text{rec}})\)
    \STATE Optimize \(\hat{\vx}^{*n_\mathrm{tot}} \in \argmax_\vx \frac{1}{N_{u,\text{rec}}} \sum_{\vu \in \sU_{\mathrm{MC}, \text{rec}}} \max_{\vy} \mu^{n_\mathrm{tot}}(\vx, \vy, \vu)\)
    \STATE Let \(\hat{\vg}^{*n_\mathrm{tot}} : \sU \to \sY\) given by \(\hat{\vg}^{*n_\mathrm{tot}}(\vu) \in \argmax_{\vy} \mu^{n_\mathrm{tot}}(\hat{\vx}^{*n_\mathrm{tot}}, \vy, \vu)\)
    \ENSURE \(\hat{\vx}^{*n_\mathrm{tot}}, \hat{\vg}^{*n_\mathrm{tot}}\)
\end{algorithmic}
\end{algorithm}

\subsection{An alternative, alternating policy}\label{sec:acqf-alternating}

The optimization of the joint knowledge gradient policy requires several layers of approximation. Previous works on robust \citep{pearce2017robustkg, toscanopalmerin2022integrands} and multi-task \citep{pearce2018multitask} knowledge gradient, which used a discrete approximation for the inner maximization, have avoided the need for the quasi-Monte Carlo (qMC) approximation of the expectation over the next observation.
The tricks used cannot be extended to our formulation of the joint KG acquisition function. However, we can implement it in an alternating strategy switching between a pair of acquisition functions that aim to improve the fixed design, \(\vx\), and adjustable variables, \(\vy\), separately. Both acquisition functions use a GP model of the objective as a function of the whole space \(\sX \times \sY \times \sU\), and rely on the knowledge gradient principle.
\Cref{alg:bo-codesign-alternating-kg} outlines the algorithm, and the two acquisition functions are explained in detail in the following two sections.

\begin{algorithm}
\caption{Bayesian optimization for two-stage stochastic programs using the alternating knowledge gradient}
\label{alg:bo-codesign-alternating-kg}
\begin{algorithmic}[1]
    \REQUIRE Initial sample size \(n_0\), evaluation budget \(n_\mathrm{tot}\), discretization sizes \(N_x\) and \(N_y\), qMC sample sizes \(N_u\) and \(N_{u,\text{rec}}\) (powers of two)
    \STATE Evaluate \(h\) at \(n_0\) points, chosen according to a scrambled Sobol' sequence on \(\sX \times \sY \times \sU\)
    \FOR {\(n = n_0\) \TO \(n_\mathrm{tot} - 1\)}
        \STATE Fit hyperparameters of GP prior, \(f \sim \mathcal{GP}(\mu, k)\), using MAP estimates
        \STATE Generate \(\sXd \gets \mathrm{LHC}(N_x)\), \(\sYd \gets \mathrm{LHC}(N_y)\), \(\sUmc \gets \mathrm{Sobol}(N_u)\)
        \IF {\(n - n_0\) even}
            \STATE \(\vx^{n+1}, \vy^{n+1}, \vu^{n+1} \gets \argmax_{\vx, \vy, \vu} \hat{\alpha}_\text{aKG-fix}^n(\vx, \vy, \vu;\, \sXd, \sYd, \sUmc)\)
        \ELSE
            \STATE \(\vx^{n+1}, \vy^{n+1}, \vu^{n+1} \gets \argmax_{\vx, \vy, \vu} \hat{\alpha}_\text{aKG-adj}^n(\vx, \vy, \vu;\, \sXd, \sYd, \sUmc)\)
        \ENDIF
        \STATE Evaluate \(h(\vx^{n+1}, \vy^{n+1}, \vu^{n+1})\) (expensive)
    \ENDFOR
    \STATE Fit hyperparameters of GP prior, \(f \sim \mathcal{GP}(\mu, k)\), using MAP estimates
    \STATE Generate \(\sU_{\mathrm{MC}, \text{rec}} \gets \mathrm{Sobol}(N_{u,\text{rec}})\)
    \STATE Optimize \(\hat{\vx}^{*n_\mathrm{tot}} \in \argmax_\vx \frac{1}{N_{u,\text{rec}}} \sum_{\vu \in \sU_{\mathrm{MC}, \text{rec}}} \max_{\vy} \mu^{n_\mathrm{tot}}(\vx, \vy, \vu)\)
    \STATE Let \(\hat{\vg}^{*n_\mathrm{tot}} : \sU \to \sY\) given by \(\hat{\vg}^{*n_\mathrm{tot}}(\vu) \in \argmax_{\vy} \mu^{n_\mathrm{tot}}(\hat{\vx}^{*n_\mathrm{tot}}, \vy, \vu)\)
    \ENSURE \(\hat{\vx}^{*n_\mathrm{tot}}, \hat{\vg}^{*n_\mathrm{tot}}\)
\end{algorithmic}
\end{algorithm}

\subsubsection{Improving the fixed design}

The acquisition function designed to improve the fixed design asks `if we must choose the control policy \(\vg\) now but can make one more observation before choosing the fixed design \(\vx\), how much will the next sample improve the maximum of the posterior mean, in expectation?'. While this is a less natural question than the one that motivates the joint strategy, it effectively reduces our problem to a standard stochastic optimization problem, for which we can apply similar techniques to existing robust knowledge gradient formulations using discretizations \citep{pearce2017robustkg, toscanopalmerin2022integrands}. Note that when we recommend the final solution, we still use \cref{eq:recommendation-bilevel,eq:recommendation-bilevel-x-approx}, making use of all information available.

Recall the definitions of the recommendations \(\vx^{*n}, \vg^{*n}\) in \cref{eq:recommendation-bilevel}. The acquisition function is obtained from \cref{eq:acqf-joint-kg-conceptual} and \cref{eq:acqf-joint-kg-bilevel} by replacing the maximizations over \(\vg'\) with \(\vg^{*n}\) -- the optimal policy before the next observation,
\begin{subequations}\label{eq:acqf-alternating-kg-fixed}
\begin{alignat}{2}
    \alpha_\text{aKG-fix}^n(\tilde{\vx}) &= \E_{\rv^{n+1}}\biggl[\;
            \max_{\rvx' \in \sX} \bar{\mu}^{n+1}(\rvx', \vg^{*n})
        \,\bigg|\,
            \tilde{\vx}^{n+1} = \tilde{\vx}
        \biggr]
        &&-
        \max_{\vx' \in \sX} \bar{\mu}^n(\vx', \vg^{*n}) \\
    &= \E_{\rv^{n+1}}\Biggl[\;
    \underbrace{
        \max_{\rvx' \in \sX} \,\E_{\rvu'}\!\biggl[
            \mu^{n+1}(\rvx', \vg^{*n}(\rvu'), \rvu')
        \biggr]
    }_{\substack{
        \text{Maximum of posterior mean} \\
        \text{after next observation, } \rv^{n+1}, \\
        \text{if we can only change } \rvx' = \rvx^{*(n+1)}
    }}
    \,\Bigg|\,
        \tilde{\vx}^{n+1} = \tilde{\vx}
    \Biggr]
    &&- 
    \underbrace{
        \max_{\vx' \in \sX} \,\E_{\rvu'}\!\biggl[\mu^n(\vx', \vg^{*n}(\rvu'), \rvu')\biggr]
    }_{\substack{
        \text{Current maximum of posterior mean} \\
        \text{(constant wrt } \tilde{\vx} \text{)}
    }}.
\end{alignat}
\end{subequations}
This is approximated by
\begin{subequations}
\begin{align}
    \doublehat{\vx}^{*n} &\in \argmax_{\vx' \in \sXd} \frac{1}{N_u} \sum_{\vu' \in \sUmc} \max_{\vy' \in \sYd} \mu^n(\vx', \vy', \vu'), \label{eq:acqf-alternating-kg-x-approx} \\
    \forall \vu' \in \sUmc \quad \doublehat{\vg}^{*n}(\vu') &\in \argmax_{\vy' \in \sYd} \mu^n(\doublehat{\vx}^{*n}, \vy', \vu'), \label{eq:acqf-alternating-kg-g-approx} \\
    \begin{split}
    \hat{\alpha}_\text{aKG-fix}^n(\tilde{\vx}) &= \E_{\rv^{n+1}}\Biggl[\;
        \max_{\rvx' \in \sXd} \underbrace{\frac{1}{N_u} \sum_{\vu' \in \sUmc} \mu^{n+1}(\rvx', \doublehat{\vg}^{*n}(\vu'), \vu')}_{\text{Varies linearly with } \rv^{n+1}}
    \,\Bigg|\,
        \tilde{\vx}^{n+1} = \tilde{\vx}
    \,\Biggr] \\
    &\qquad\qquad\qquad\qquad\qquad\qquad\qquad\qquad - \frac{1}{N_u} \sum_{\vu' \in \sUmc} \mu^n(\vx', \doublehat{\vg}^{*n}(\vu'), \vu'). \label{eq:acqf-alternating-kg-fixed-approx}
    \end{split}
\end{align}
\end{subequations}

We observe that the posterior mean \(\mu^{n+1}(\cdot)\) varies linearly with \(\rv^{n+1}\) in \cref{eq:acqf-alternating-kg-fixed-approx}. Therefore, the expectation over \(\rv^{n+1}\) can be computed exactly using Algorithm~2 in \citep{frazier2009kg}, see \citep{pearce2017robustkg,toscanopalmerin2022integrands}. The resulting expression is deterministic and differentiable with respect to \(\tilde{\vx}\), so can be optimized using a gradient-based optimizer such as multi-start L-BFGS-B.

\subsubsection{Improving the adjustable variables}

Complementing this, the acquisition function designed to improve the adjustable variables asks the question `if we must choose the fixed design \(\vx\) now but can make one more observation before before choosing the control policy \(\vg\), how much will the next sample improve the maximum of the posterior mean, in expectation?'. By pretending that the fixed design \(\vx\) cannot be changed, the problem is reduced to a multi-task optimization problem and we can apply the multi-task knowledge gradient strategy proposed in \citep{pearce2018multitask}.

The acquisition function is obtained from \cref{eq:acqf-joint-kg-conceptual} and \cref{eq:acqf-joint-kg-bilevel} by replacing the maximizations over \(\vx'\) with \(\vx^{*n}\) from \cref{eq:recommendation-bilevel-x},
\begin{subequations}\label{eq:acqf-alternating-kg-adjustable}
\begin{alignat}{2}
    \alpha_\text{aKG-adj}^n(\tilde{\vx}) &= \E_{\rv^{n+1}} \biggl[
      \max_{\rvg': \sU \to \sY} \bar{\mu}^{n+1}(\vx^{*n}, \rvg')
      \,\bigg|\, \tilde{\vx}^{n+1} = \tilde{\vx}
      \biggr]
      &&-
      \max_{\vg': \sU \to \sY} \bar{\mu}^n(\vx^{*n}, \vg') \\
    &= \E_{\rv^{n+1}} \Biggl[\;
        \underbrace{
            \E_{\rvu'}\!\biggl[ \max_{\rvy' \in \sY} \mu^{n+1}(\vx^{*n}, \rvy', \rvu') \biggr]
        }_{\substack{
            \text{Maximum of posterior mean} \\
            \text{after next observation, } \rv^{n+1}, \\
            \text{if we can only change } \rvy' = \rvg^{*(n+1)}(\rvu')
        }}
    \,\Bigg|\,
        \tilde{\vx}^{n+1} = \tilde{\vx}
    \Biggr]
    &&-
    \underbrace{
        \E_{\rvu'}\!\biggl[ \max_{\rvy' \in \sY} \mu^n(\vx^{*n}, \rvy', \rvu') \biggr]
    }_{\substack{
        \text{Current maximum of posterior} \\
        \text{mean (constant wrt } \vx, \vy \text{ and } \vu \text{)}
    }}.
\end{alignat}
\end{subequations}
It is approximated by
\begin{multline}\label{eq:acqf-alternating-kg-adjustable-approx}
    \hat{\alpha}_\text{aKG-adj}^n(\tilde{\vx}) = \frac{1}{N_u} \sum_{\vu' \in \sUmc} \ \E_{\rv^{n+1}}\Biggl[ \max_{\rvy' \in \sYd} \!\! \underbrace{\mu^{n+1}(\doublehat{\vx}^{*n}, \rvy', \vu')}_{\text{Varies linearly with } \rv^{n+1}}
    \,\Bigg|\,
    \tilde{\vx}^{n+1} = \tilde{\vx}
    \,\Biggr] \\
    - \frac{1}{N_u} \sum_{\vu' \in \sUmc} \max_{\vy' \in \sYd} \mu^n(\doublehat{\vx}^{*n}, \vy', \vu')
\end{multline}
where \(\doublehat{\vx}^{*n}\) is given by \cref{eq:acqf-alternating-kg-x-approx}.
Again, the expectation over \(\rv^{n+1}\) can be computed exactly, the result is differentiable with respect to \(\tilde{\vx}\), and the acquisition function can be optimized using a gradient-based optimizer such as multi-start L-BFGS-B.

\subsection{The two-step incumbent method}\label{sec:acqf-twostep}
The standard approach of a particular industry partner is to design the fixed design \(\vx\) and controller \(\vg\) separately. That is, engineers must alternate between the following two steps:
\begin{enumerate}
    \item Fix the fixed design, \(\vx\), and optimize the control strategy, \(\vg\);
    \item Fix the control strategy, \(\vg\), and optimize the fixed design, \(\vx\).
\end{enumerate}
This is inefficient because there is no immediate convergence criterion to determine when to switch steps, and because the GP model must `forget' past data each time it does switch steps. For this reason, we consider a benchmark algorithm that only executes each step once and consider that future iterations of the steps would happen beyond the evaluation budget.

As in the previous sections, we employ the knowledge gradient acquisition function. In the first step, an \(\vx_\text{step-1} \in \sX\) is chosen arbitrarily, the restriction of \(h\) to \(\{\vx_\text{step-1}\} \times \sY \times \sU\) is modeled using a Gaussian process \(f_\text{step-1} : \sY \times \sU \to \R\) and the knowledge gradient (named REVI in \citep{pearce2018multitask}) is maximized to select the \(\vy\) and \(\vu\) for the next observation. At the end of the first step, an optimal policy \(\vg^{*n_{\mathrm{tot},1}}\) is extracted. In the second step, the restriction of \(h\) to \(\{(\vx, \vy, \vu) \in \sX \times \sY \times \sU : \vy = \vg^{*n_{\mathrm{tot},1}}(\vu)\}\) is modeled with a Gaussian process \(f_\text{step-2} : \sX \times \sU \to \R\) and the knowledge gradient is maximized to select the \(\vx\) and \(\vu\) for the next observation. At the end of the second step, an optimal fixed design \(\vx^{*n_{\mathrm{tot},2}}\) is extracted.

The two acquisition functions are
\begin{subequations} \label{eq:acqf-twostep-kg}
\begin{align}
    \alpha_\text{2sKG-1}^n(\vy, \vu) &= \E_{\rv_\text{step-1}^{n+1}} \Biggl[\;
        \E_{\rvu'} \biggl[ \max_{\rvy' \in \sY} \mu_\text{step-1}^{n+1}(\rvy', \rvu') \biggr]
    \,\Bigg|\,
        (\vy_\text{step-1}^{n+1}, \vu_\text{step-1}^{n+1}) = (\vy, \vu)
    \Biggr]
    -
    \E_{\rvu'} \biggl[ \max_{\rvy' \in \sY} \mu_\text{step-1}^n(\rvy', \rvu') \biggr], \\
    \alpha_\text{2sKG-2}^n(\vx, \vu) &= \E_{\rv_\text{step-2}^{n+1}} \Biggl[\;
        \max_{\rvx' \in \sX} \E_{\rvu'} \bigl[\mu_\text{step-2}^{n+1}(\rvx', \rvu')\bigr]
    \,\Bigg|\,
        (\vx_\text{step-2}^{n+1}, \vu_\text{step-2}^{n+1}) = (\vx, \vu)
    \Biggr]
    -
    \max_{\vx' \in \sX} \E_{\rvu'} \bigl[\mu_\text{step-2}^n(\vx', \rvu')\bigr].
\end{align}
\end{subequations}
Here, the hypothesized next observation is \(\rv_\text{step-1}^{n+1} = f_\text{step-1}(\vy_\text{step-1}^{n+1}, \vu_\text{step-1}^{n+1}) + \varepsilon_\text{step-1}^{n+1}\) in step one and \(\rv_\text{step-2}^{n+1} = f_\text{step-2}(\vx_\text{step-2}^{n+1}, \vu_\text{step-2}^{n+1}) + \varepsilon_\text{step-2}^{n+1}\) in step two, where \(\varepsilon_\text{step-1}^{n+1}, \varepsilon_\text{step-2}^{n+1} \sim \mathcal{N}(0, \sigma^2)\) represent observation noise.

These are approximated similarly to the alternating algorithm from \cref{sec:acqf-alternating}.
\begin{subequations} \label{eq:acqf-twostep-kg-approx}
\begin{align}
\begin{split}
    \hat{\alpha}_\text{2sKG-1}^n(\vy, \vu) &= \frac{1}{N_u} \sum_{\vu'\in \sUmc}
    \E_{\rv_\text{step-1}^{n+1}} \Biggl[\;
    \max_{\rvy' \in \sYd} \mu_\text{step-1}^{n+1}(\rvy', \vu')
    \,\Bigg|\,
        (\vy_\text{step-1}^{n+1}, \vu_\text{step-1}^{n+1}) = (\vy, \vu)
    \Biggr] \\
    &\qquad\qquad\qquad\qquad\qquad\qquad\qquad\qquad\qquad - \frac{1}{N_u} \sum_{\vu' \in \sUmc} \max_{\vy' \in \sYd} \mu_\text{step-1}^n(\vy', \vu'),
\end{split}\\
\begin{split}
    \hat{\alpha}_\text{2sKG-2}^n(\vx, \vu) &= \E_{\rv_\text{step-2}^{n+1}} \Biggl[\;
        \max_{\rvx' \in \sXd} \frac{1}{N_u}\sum_{\vu' \in \sUmc} \mu_\text{step-2}^{n+1}(\rvx', \vu')
    \,\Bigg|\,
        (\vx_\text{step-2}^{n+1}, \vu_\text{step-2}^{n+1}) = (\vx, \vu)
    \Biggr] \\
    &\qquad\qquad\qquad\qquad\qquad\qquad\qquad\qquad\qquad - \max_{\vx', \in \sXd} \frac{1}{N_u} \sum_{\vu' \in \sUmc} \mu_\text{step-2}^n(\vx', \vu').
\end{split}
\end{align}
\end{subequations}
As with the alternating algorithm, the expectations over \(\rv_\text{step-1}^{n+1}\) and \(\rv_\text{step-2}^{n+1}\) can be computed exactly using Algorithm~2 in \citep{frazier2009kg}, the result is deterministic and differentiable, and can be optimized using a gradient-based optimizer such as multi-start L-BFGS-B.

\begin{algorithm}
\caption{Bayesian optimization for two-stage stochastic programs using a two-step knowledge gradient}
\label{alg:bo-twostep}
\begin{algorithmic}[1]
    \REQUIRE Initial sample sizes \(n_1, n_2\) for each step, evaluation budget \(n_{\mathrm{tot},1}, n_{\mathrm{tot},2}\) for each step, discretization sizes \(N_x\) and \(N_y\), qMC sample sizes \(N_u\) and \(N_{u,\mathrm{rec}}\)
    \item[]
    \STATE\COMMENT{Step one (fix \(\vx\) and optimize \(\vg\))}
    \STATE Choose \(\vx_\text{step-1} \in \sX\) arbitrarily
    \STATE Evaluate \(h\) at \(n_1\) points, chosen by combining \(\vx_\text{step-1}\) with a scrambled Sobol' sequence on \(\sY \times \sU\)
    \FOR {\(n=n_1\) \TO \(n_{\mathrm{tot},1}-1\)}
        \STATE Fit MAP hyperparameters of GP prior, \(f_\text{step-1} \sim \mathcal{GP}\), on functions \(\sY \times \sU \to \R\) 
        \STATE Generate \(\sYd \gets \mathrm{LHC}(N_y)\), \(\sUmc \gets \mathrm{Sobol}(N_u)\)
        \STATE \(\vy_\text{step-1}^{n+1}, \vu_\text{step-1}^{n+1} \gets \argmax_{\vy, \vu} \alpha_\text{2sKG-1}^n(\vy, \vu;\, \sYd, \sUmc)\)
        \STATE Evaluate \(h(\vx_\text{step-1}, \vy_\text{step-1}^{n+1}, \vu_\text{step-1}^{n+1})\) (expensive)
    \ENDFOR
    \STATE Fit MAP hyperparameters of GP prior, \(f_\text{step-1} \sim \mathcal{GP}\), on functions \(\sY \times \sU \to \R\) 
    \STATE Let \(\hat{\vg}^{*n_{\mathrm{tot},1}}: \sU \to \sY\) given by \(\hat{\vg}^{*n_{\mathrm{tot},1}}(\vu) \in \argmax_{\vy} \mu_\text{step-1}^{n_{\mathrm{tot},1}}(\vy, \vu)\)
    \item[]
    \STATE\COMMENT{Step two (fix \(\vg\) and optimize \(\vx\))}
    \STATE Evaluate \(h\) at \(n_2\) points, chosen using scrambled Sobol' sequence on \(\sX \times \sU\) augmented by \(\vy\)-values given by \(\hat{\vg}^{*n_{\mathrm{tot},1}}(\cdot)\)
    \FOR {\(n=n_2\) \TO \(n_{\mathrm{tot},2}-1\)}
        \STATE Fit MAP hyperparameters of GP prior, \(f_\text{step-2} \sim \mathcal{GP}\), on functions \(\sX \times \sU \to \R\) 
        \STATE Generate \(\sXd \gets \mathrm{LHC}(N_x)\), \(\sUmc \gets \mathrm{Sobol}(N_u)\)
        \STATE \(\vx_\text{step-2}^{n+1}, \vu_\text{step-2}^{n+1} \gets \argmax_{\vx, \vu} \alpha_\text{2sKG-2}^n(\vx, \vu;\, \sXd, \sUmc)\)
        \STATE Evaluate \(h\bigl(\vx_\text{step-2}^{n+1},\, \hat{\vg}^{*n_{\mathrm{tot},1}}(\vu_\text{step-2}^{n+1}),\, \vu_\text{step-2}^{n+1} \bigr)\) (expensive)
    \ENDFOR
    \STATE Fit MAP hyperparameters of GP prior, \(f_\text{step-2} \sim \mathcal{GP}\), on functions \(\sX \times \sU \to \R\) 
    \STATE Generate \(\sU_{\mathrm{MC}, \text{rec}} \gets \mathrm{Sobol}(N_{u,\mathrm{rec}})\)
    \STATE Optimize \(\hat{\vx}^{*n_{\mathrm{tot},2}} \in \argmax_\vx \frac{1}{N_{u,\mathrm{rec}}} \sum_{\vu \in \sU_{\mathrm{MC}, \text{rec}}} \mu_\text{step-2}^{n_{\mathrm{tot},2}}(\vx, \vu)\)
    \item[]
    \ENSURE \(\hat{\vx}^{*n_{\mathrm{tot},2}}\), \(\hat{\vg}^{*n_{\mathrm{tot},1}}\)
\end{algorithmic}
\end{algorithm}

\section{Experiments} \label{sec:experiments}

To demonstrate the value of jointly optimizing the fixed design and adjustable variables, experiments are run on a variety of synthetic and real-world examples. Using the synthetic test problems sampled from Gaussian processes we investigate the effects of dimension and length scales, and demonstrate the ability of the algorithms to cope with observation noise. This setting allows us to isolate these effects without potential interference from model mismatch. The examples of optical table design and supply-chain optimization validate the algorithms in real-world optimization landscapes, and allow us to compare to the current state of the art from \citep{xie2021globallocal}.

In all experiments, the initial sample is a scrambled Sobol' sequence, whose size depends on the problem dimension and is detailed separately in \cref{app:further-exp-details-params}.

The main algorithm in this paper, presented in \cref{alg:bo-codesign-joint-kg,eq:acqf-joint-kg-approx}, jointly optimizes the fixed design \(\vx\) and control strategy \(\vg\) and we refer to it as \emph{joint KG (jKG)}. We refer to the alternating alternative from \cref{sec:acqf-alternating} as \emph{alternating KG (aKG)} and the two-step incumbent from \cref{sec:acqf-twostep} as \emph{two-step KG (2sKG)}.
In addition to comparing the three knowledge gradient algorithms, we also compare with a joint and a two-step random sampling policy, which we call \emph{joint RS (jRS)} and \emph{two-step RS (2sRS)} respectively. These policies use the same models and recommendation procedures as the joint and two-step knowledge gradient but recommend samples according to a scrambled Sobol' sequence.
The fixed design \(\vx_\text{step-1}\) used in the first step of the two-step algorithms is in the center of the space \(\sX\).

\subsection{Performance metric}

We measure performance using \emph{simple regret}, also known as \emph{opportunity cost}, which is the expected difference in the value of the solution we would recommend after \(n\) samples and the best possible value for that problem. As in previous sections, we write \(\vx^*\) and \(\vg^*\) for the optimal fixed design and control policy, and \(\rvx^{*n}\) and \(\rvg^{*n}\) for the recommendations after step \(n\). The recommendations \(\rvx^{*n}\) and \(\rvg^{*n}\) depend on the initial data and the observation noise, and so are random. The simple regret is given by

\begin{equation}\label{eq:simple-regret}
    r^n = \E_\rvu[h(\vx^*, \vg^*(\rvu), \rvu)] - \E_{\rvx^{*n}, \rvg^{*n}}\bigl[\E_\rvu[h(\rvx^{*n}, \rvg^{*n}(\rvu), \rvu]\bigr].
\end{equation}

For synthetic functions sampled from a GP, we take expectation over the test problem as well and call the result \emph{expected simple regret}, \(\E_h[\rr^n]\).

We estimate the simple regret by repeating experiments \(M=100\) times with different random seeds and taking the mean. We use a scrambled Sobol' sample of size \(N_{u,\text{rec}} = 128\) for the quasi-Monte Carlo approximation over \(\rvu\). Writing \(\vx^{*n,i}, \vg^{*n,i}\) for the recommendations in the \(i\)\textsuperscript{th} experiment after \(n\) data points have been collected, and \(\vu^1, \dots, \vu^{N_{u,\text{rec}}}\) for the qMC sample of \(\rvu\), the approximate simple regret is
\begin{equation}\label{eq:simple-regret-approx}
    \hat{r}^n = \frac{1}{M N_{u,\text{rec}}} \sum_{i=1}^M \sum_{j=1}^{N_{u,\text{rec}}} \biggl(
        h(\vx^*, \vg^{*}(\vu^j), \vu^j) - h(\vx^{*n,i}, \vg^{*n,i}(\vu^j), \vu^j)
    \biggr).
\end{equation}
For the synthetic functions sampled from a GP, each of the \(M\) experiments is run with a different sample \(h^i\) from the test problem GP.

In the results that follow, we plot this estimate as a function of \(n\) for each algorithm and indicate the region of two standard errors on either side of the mean. By the Central Limit Theorem, this will approximate a 95\% confidence interval.

\subsection{Synthetic examples: random GP samples} \label{sec:experiments-gp}

The synthetic examples are generated by sampling a GP. By using GP samples, we remove model-mismatch issues from the picture, such as landscapes with heterogeneous length scales. For each family of problems, characterized by dimension and length scale, 100 different test problems are generated by taking 100 independent samples of the GP. The samples are generated by randomly sampling the weights and basis vectors for 1024 random Fourier features as in \citep{rahimi2007rff}, using the implementation in BoTorch \citep{balandat2020botorch}.
Note that while the GP test problems are sampled from a known distribution, we still fit the hyperparameters of the surrogate GP to the available observations at each iteration using MAP estimates, as described in \cref{app:further-exp-details-gpsurrogate}.

\Cref{fig:results-dimensions} shows results for test problems generated from six-dimensional GPs, with different distributions of the dimensions between the variable types. Joint KG and alternating KG perform well in all cases, with the joint version marginally outperforming the alternating one in some places and both significantly outperforming the random and two-step benchmarks.

The two-step KG performs at its best when the higher dimensionality is in the adjustable variable space. This is because it optimizes the adjustable variables first, which is the harder part of the problem. It makes almost no improvements in the first step when the higher dimensionality is in the fixed design space.

\begin{figure}[htbp]
    \centering
    \includegraphics[width=\textwidth]{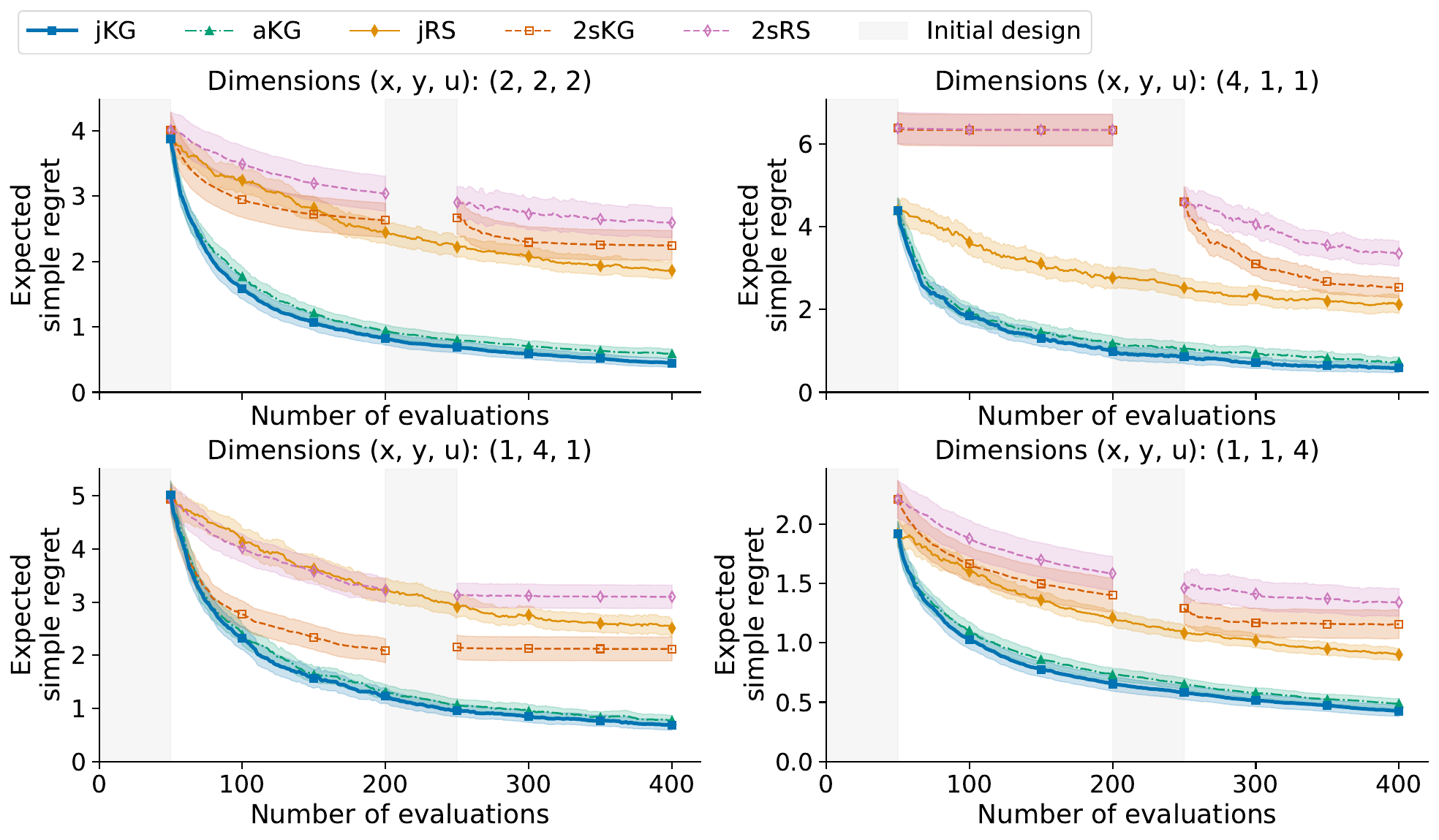}
    \caption{Evolution of the expected simple regret for GP sampled problems of different dimensions, with gray shading showing the evaluations attributed to the initial design. Note that the joint and alternating algorithms only require the first initial design phase, while the two-step algorithms require both. The colored shaded regions around each line indicate two standard errors either side of the mean (an approximate 95\% confidence interval). Joint KG and alternating KG perform well in all cases, with joint KG approaching zero regret slightly faster than alternating KG. Both algorithms outperform the two-step and random sampling benchmark algorithms. The two-step algorithms do particularly badly when the \((\vx,\vy,\vu)\)-dimension is \((4,1,1)\) since the first step only optimizes the adjustable variables while the bulk of the optimization problem is in the fixed design space.}
    \label{fig:results-dimensions}
\end{figure}

We see a similar pattern when examining problems of different length scales. \Cref{fig:results-lengthscales} shows results for test problems generated from three-dimensional GPs -- an \((\vx, \vy, \vu)\)-dimension of \((1, 1, 1)\) -- where one of the variables has a short length scale and the other two have a very long length scale. Again, the joint and alternating KG variants are consistently the best. The two-step algorithms are particularly poor when the short length scale is in the \(\vx\)-dimension, since the fixed design is not optimized until the second step.

\begin{figure}[htbp]
    \centering
    \includegraphics[width=\textwidth]{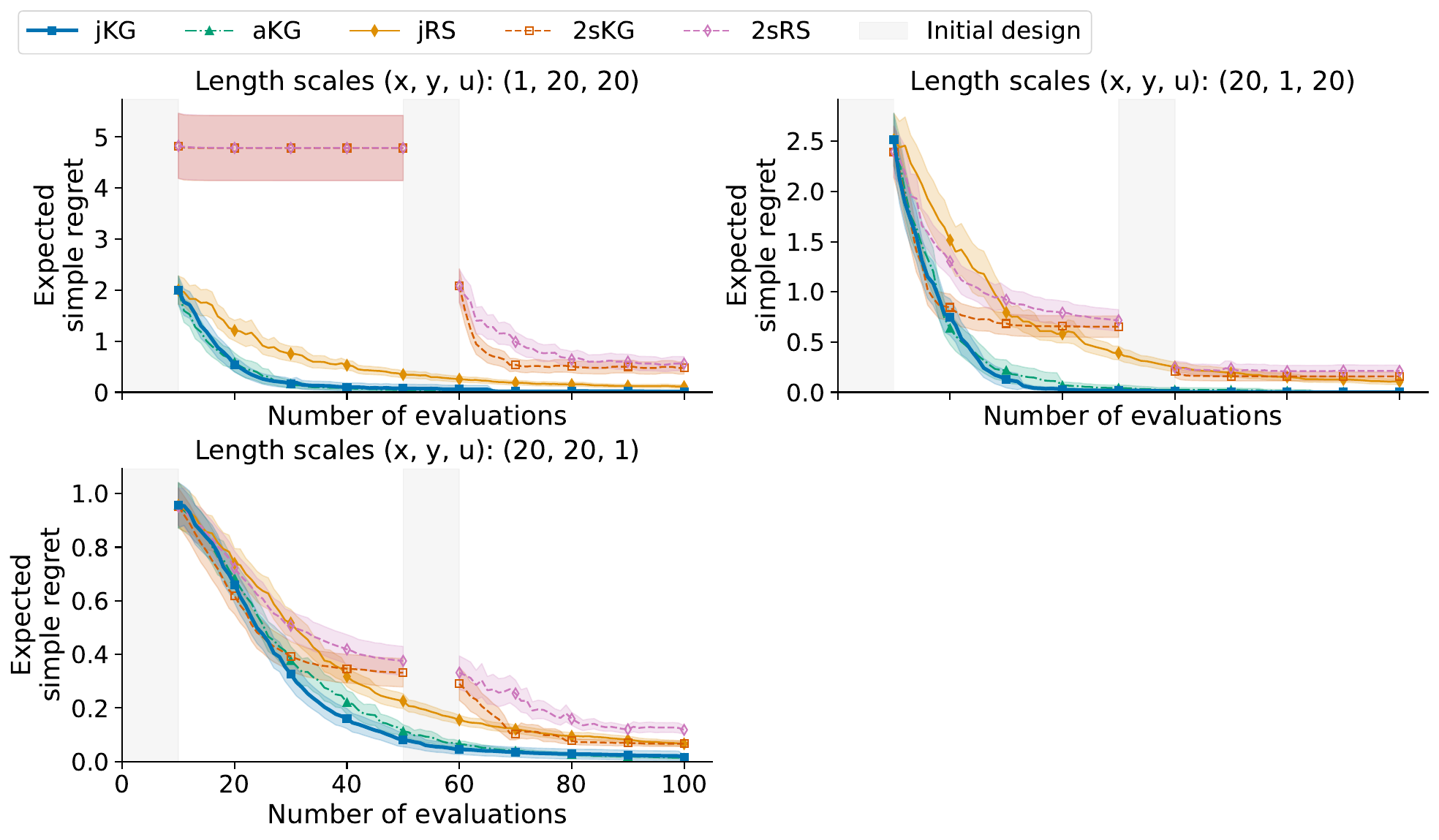}
    \caption{Evolution of the expected simple regret for GP sampled problems of different length scales, with gray shading showing the evaluations attributed to the initial design. Note that the joint and alternating algorithms only require the first initial design phase, while the two-step algorithms require both. The colored shaded regions around each line indicate two standard errors either side of the mean (an approximate 95\% confidence interval). Joint KG and alternating KG are consistently the best performers, and the two-step algorithms perform particularly poorly when the short length scale is in the \(\vx\)-dimension.}
    \label{fig:results-lengthscales}
\end{figure}

One motivation for using knowledge gradient is that it naturally handles observation noise. \Cref{fig:results-noisy} shows results for GP generated test problems of \((\vx, \vy, \vu)\)-dimension \((2, 2, 2)\) with additive Gaussian observation noise. It confirms that the algorithms are not adversely affected by the inclusion of observation noise and that joint and alternating KG continue to significantly outperform the Sobol' and two-step benchmark algorithms.

\begin{figure}[htbp]
    \centering
    \includegraphics[width=0.5\textwidth]{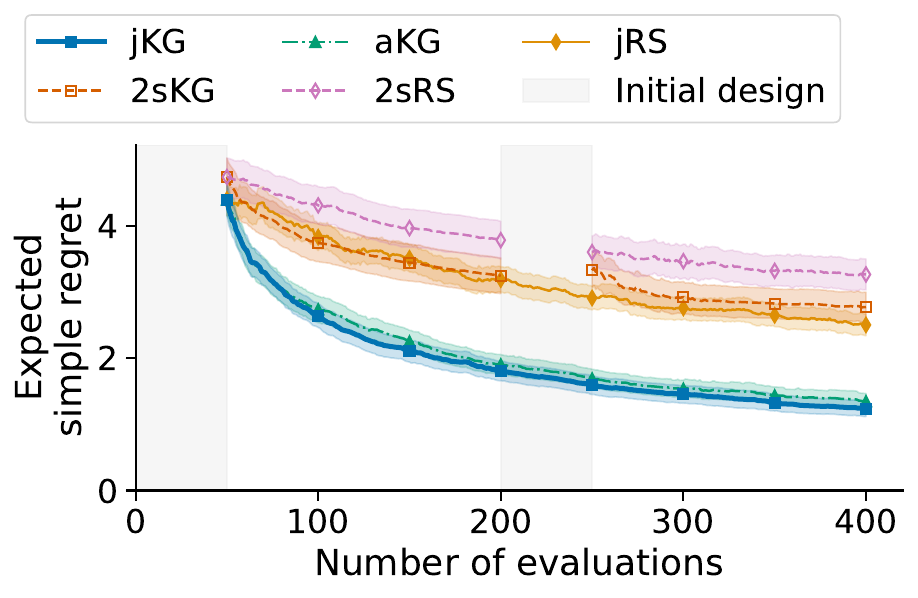}
    \caption{Evolution of the expected simple regret for GP sampled problems with additive Gaussian observation noise of standard deviation \(\sigma=2\). The problems have \((\vx, \vy, \vu)\)-dimension \((2, 2, 2)\) and come from the same distribution as those which appear in \cref{fig:results-dimensions}. Gray shading shows the evaluations attributed to the initial design. Note that the joint and alternating algorithms only require the first initial design phase, while the two-step algorithms require both. The colored shaded regions around each line indicate two standard errors either side of the mean (an approximate 95\% confidence interval). While the performance of all algorithms gets worse with increased levels of observation noise, the joint and alternating KG continue to outperform the Sobol' sequence and two-step benchmarks.}
    \label{fig:results-noisy}
\end{figure}

\subsection{Design of an optical table} \label{sec:experiments-optical-table}

The optimization landscapes found in real-world problems are unlikely to be seen as samples from a stationary GP. To confirm that the joint and alternating KG continue to be the best performers in the presence of model-mismatch, we consider the design of an optical table.
This example is a simplification of the example in \citep{salomon2019activeRobustness}.
An optical table is used in optics experiments to reduce the amplitude of vibrations propagated from the environment to the equipment.

We model the table as a rectangle of mass \(m_1 = 200\mathrm{kg}\), supported symmetrically by four springs with spring constant \(k\) -- one at each corner of the table. A damper with adjustable coefficient \(c\) is located centrally, and the equipment of mass \(m_2 = 20\mathrm{kg}\) is placed with its center of mass in the center. It is assumed that the equipment is symmetric so that rotational vibration can be neglected. The floor is modeled to vibrate with simple harmonic motion at uncertain frequency \(\omega\).
\cref{fig:opticaltable-diagram} contains a diagram of the set-up.

The system satisfies an ordinary differential equation, and it can be shown (see \cref{app:further-exp-details-optical-table}) that the steady-state solution is another simple harmonic motion with the same frequency but a different amplitude and phase. The objective is to minimize the ratio of the amplitudes of the vibrations of the table and the floor.

For the stochastic optimization problem, the spring constants \(k\) are taken as the fixed design, the vibration frequency of the floor \(\omega\) is the uncertain environmental variable, and the damping coefficient \(c\)  is the adjustable variable.
For the GP modeling, we approximate the negative logarithm of the amplitude ratio, \(-\log_{10}(B/A)\), in place of the amplitude ratio as a function of \(k\), \(c\) and \(\log_{10}(\omega)\).

\Cref{fig:opticaltable-regret} shows that joint KG and alternating KG approach zero much faster than the Sobol' and two-step algorithms. The two-step KG is initially the fastest decreasing because, in the first step, it is solving the simpler problem of finding the best damping coefficient \(c\) for a fixed spring constant \(k\). However, because it is not known a priori for how long to run each step, and because a new set of initial samples must be collected each time the algorithm switches steps, it cannot maintain this fast convergence rate.

While this problem is simple enough to have an analytic solution, it is illustrative of more complicated problems which lack symmetry and perfect components and thus where time-consuming physical experiments are to be run in the lab.

\begin{figure}[htbp]
\centering
\begin{subfigure}[b]{0.49\linewidth}
    \centering
    \def\svgwidth{\linewidth}
    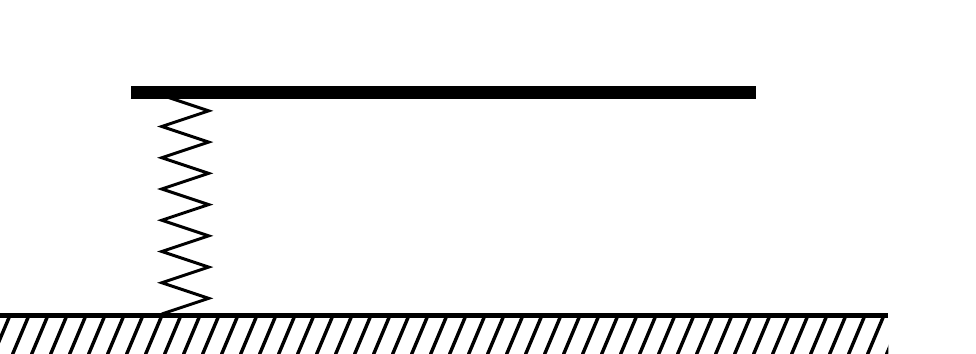
    \caption{Schematic diagram}
    \label{fig:opticaltable-diagram}
\end{subfigure}
\begin{subfigure}[b]{0.49\linewidth}
    \centering
    \includegraphics[width=\linewidth]{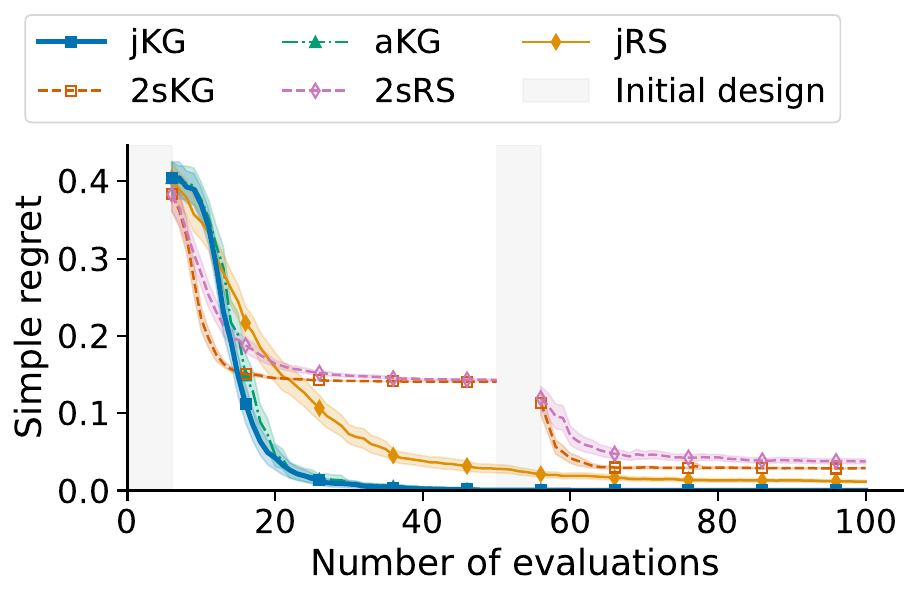}
    \caption{Simple regret}
    \label{fig:opticaltable-regret}
\end{subfigure}
\caption{Schematic diagram and evolution of the simple regret for the optical table experiment. Gray shading shows the evaluations attributed to the initial design. Note that the joint and alternating algorithms only require the first initial design phase, while the two-step algorithms require both. The colored shaded regions around each line indicate two standard errors either side of the mean (an approximate 95\% confidence interval). The joint and alternating KG algorithms are the fastest to converge to zero regret.}
\label{fig:opticaltable}
\end{figure}

\subsection{Supply chain optimization} \label{sec:experiments-supply-chain}
Supply chain management problems are a common application of stochastic programs in operations research. In this example, taken from \citep{xie2021globallocal}, we compare with the state-of-the-art algorithm proposed in that work.

The supply chain in question concerns a simplified bio-pharmaceutical process in which the production of a clinical product consumes raw materials, including soy and a raw chemical material. Orders for soy have a long lead time and must be made before uncertainty on the week-to-week demand is resolved. The lead time for the raw chemical is much shorter and can be made after we have a good prediction of the demand. The initial soy order \(x\) is the fixed design, the daily production quantity \(y_1\) along with the \((s, S)\) ordering policy for the raw chemical form the adjustable variables $\vy$ and the demand for the next four weeks are the uncertain environmental variables $\rvu$.
The objective is to minimize the costs that arise from the initial soy order, orders of the raw chemical product, the storage cost of surplus product, and the large cost to subcontract to fulfill unmet demand. These are 10 per unit of soy, 5 per unit of raw chemical, 5 per unit of manufactured product stored until the next week, and 100 per unit of manufactured product which must be subcontracted. It takes a single unit of soy and a single unit of the raw chemical to make one unit of the manufactured product.
\Cref{tab:supply-chain-variables} lists the different variables in the supply chain problem and \cref{alg:supply-chain} details the simulation.

The problem is constrained so that the daily production cannot exceed $1/20$ of the soy quantity. Further, the \((s, S)\) ordering policy for the raw chemical is constrained such that \(S > s\). When modeling the objective with a GP, we model \((s, S-s)\) rather than \((s, S)\) to convert the constraint \(s < S\) into a box constraint. The fixed design and adjustable variables are further constrained to the finite sets specified in \cref{tab:supply-chain-variables}. To handle the constraint \(y_1 \leq x/20\), the acquisition function is optimized using multi-start constrained SLSQP, and the result is then rounded to the nearest feasible value. As in \citep{xie2021globallocal}, when generating a recommended fixed design \(x^{*n}\) as well as when evaluating the recommended control policy \(\vg^{*n}\), an exhaustive search is performed over all possible \(x\) and \(\vy = (y_1, s, S)\). The marginal distribution for the environmental variables in the initial sample is normal instead of uniform to reflect that \(\rvu\) has a normal distribution.
Full details are provided in \cref{app:further-exp-details-supply-chain}.

\begin{table}[htbp]
    \centering
    \caption{Variables used in the supply chain management problem}
    \label{tab:supply-chain-variables}
    \begin{tabular}{cccl}
    \toprule
        Variable & Symbol & Type & Feasible values / Distribution\\
    \midrule
         Initial soy & \(x\) & fixed design & \(x \in \{0, 20, \dots, 5000\}\) \\
         Daily production & \(y_1\) & adjustable & \(y_1 \in [0, x/20] \cap \mathbb{Z}\) \\
         Raw chemical ordering policy & \((s, S)\) & adjustable & \(s, S \in \{100, 200, 300, 400, 500\}\) s.t. \(s < S\) \\
         Weekly demand & \(u_1, u_2, u_3, u_4\) & environmental & \(u_1, \dots, u_4 \sim \mathcal{N}(150, 10^2)\) \\
    \bottomrule
    \end{tabular}

\end{table}

\begin{algorithm}
\caption{Supply chain management simulation}
\label{alg:supply-chain}
\begin{algorithmic}[1]
    \REQUIRE Soy order quantity~\(x\), target daily production level~\(y_1\), raw chemical ordering policy parameters~\((s, S)\), weekly demand~\(u_1, u_2, u_3, u_4\)
    \STATE\COMMENT{Other variables: raw chemical inventory \(r\), clinical product inventory \(w\), cost objective \(C\)}
    \STATE \(r \gets 100\), \(C \gets 10 x\), \(w \gets 0\)
    \FOR {\(i = 1,2,3,4\)}
        \STATE\COMMENT{For each week}
        \FOR {\(j = 1,2,3,4,5\)}
            \STATE\COMMENT{For each working day}
            \STATE\COMMENT{Top up the raw chemical inventory}
            \IF {\(r < s\)}
                \STATE \(C \gets C + 5 (S - r)\)
                \STATE \(r \gets S\)
            \ENDIF
            \STATE\COMMENT{Produce clinical product}
            \STATE \(w_\text{new} \gets \min(y_1, x, r)\)
            \STATE \(x \gets x - w_\text{new}\), \(r \gets r - w_\text{new}\), \(w \gets w + w_\text{new}\)
        \ENDFOR
        \STATE\COMMENT{Calculate costs for unmet demand / surplus product}
        \IF {\(w \geq u_i\)}
            \STATE \(w \gets w - u_i\)
            \STATE \(C \gets C + 5 * w\)
        \ELSE
            \STATE \(C \gets C + 100 * (u_i - w)\)
            \STATE \(w \gets 0\)
        \ENDIF
    \ENDFOR
    \ENSURE Cost, \(C\)
\end{algorithmic}
\end{algorithm}

The objective is to minimize costs in the supply chain and the average cost over 100 repetitions is presented in \cref{fig:results-supply-chain}. \citet{xie2021globallocal} only seek to recommend the best fixed design \(x^{*n}\), and measure its performance using the best possible \(\vy\) for \(x^{*n}\) (found by exhaustive search). For a fair comparison, the dotted lines show the same metric for our algorithms. The solid lines show our usual metric, which measures the combination of the recommended fixed design \(x^{*n}\) and control policy \(\vg^{*n}\). The values for the algorithm of \citep{xie2021globallocal} are taken directly from their paper.

The joint knowledge gradient attains the best value found by \citep{xie2021globallocal} (after 2000 evaluations) after just 80 evaluations using the optimal control policy (their metric) and after 140 evaluations when using the recommended control policy. We note that even a random (Sobol') sampling strategy beats the best value found in \citep{xie2021globallocal} after 480 evaluations which shows the benefit of modeling the whole design space with a single GP. In \citep{xie2021globallocal}, expensive function evaluations made on the inner loop can only be used to inform the surrogate model used for that specific \(x\), while a joint model can be informed by evaluations at all nearby locations.

\begin{figure}[htbp]
    \centering
    \includegraphics[width=\linewidth]{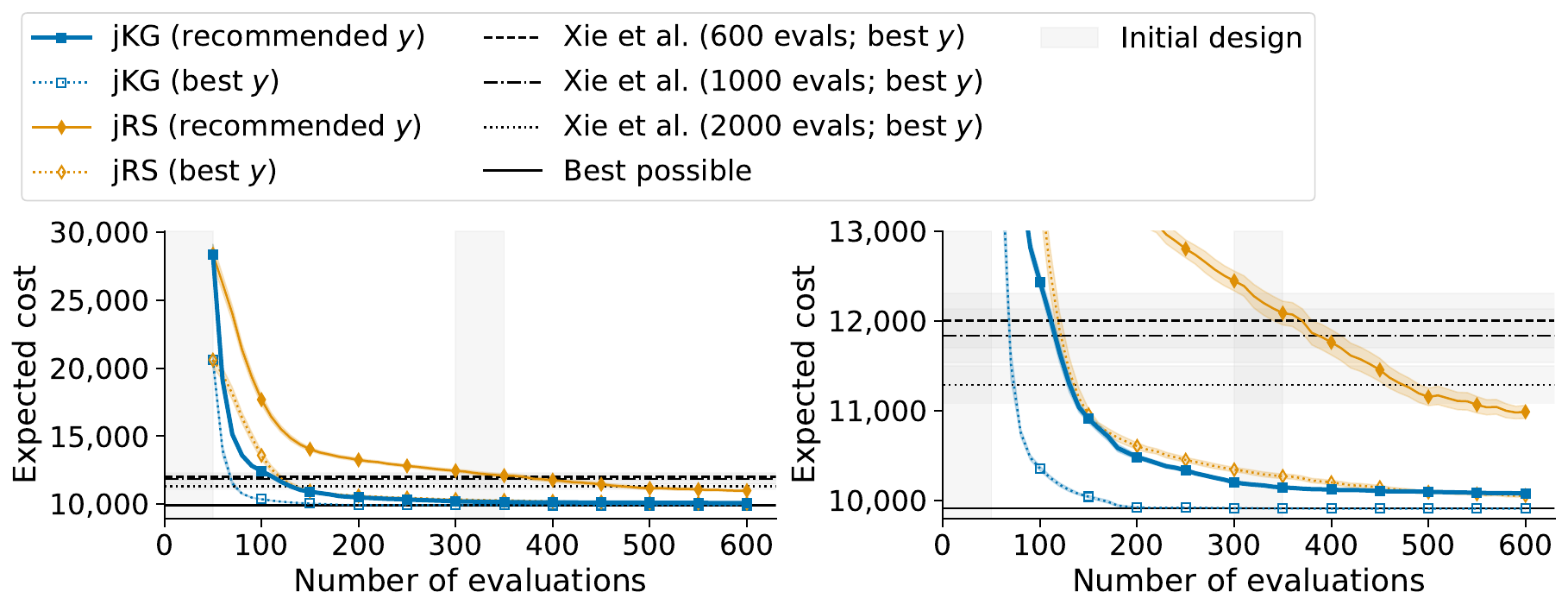}
    \caption{Evolution of the supply chain costs over 100 repetitions of the experiment. Solid lines show the cost associated with the recommended fixed design \(\vx\) and control policy \(\vy = \vg(\vx)\) as in the previous problems. Dotted lines show the cost associated with the recommended fixed design \(\vx\) if the optimal \(\vy\)-values can be chosen later. This is the metric used in \citep{xie2021globallocal}, which we include to facilitate a fair comparison. The right-hand figure shows a magnified view of the left-hand figure. Gray shading shows the evaluations attributed to the initial design. Note that the joint and alternating algorithms only require the first initial design phase, while the two-step algorithms require both. The colored shaded regions around each line indicate two standard errors either side of the mean (an approximate 95\% confidence interval).
    }
    \label{fig:results-supply-chain}
\end{figure}

\subsection{The cost of optimizing the acquisition function} \label{sec:experiments-timings}

The time required to optimize the acquisition function must be negligible compared with the time or cost required to evaluate the expensive objective. As an example, on a (2,2,2)-dimensional test problem, running with 6 CPU cores on an Intel Xeon Platinum 826 2.9GHz processor, the median time to optimize the joint KG acquisition function was just 13.5 seconds.

\Cref{fig:timings-dims-222} shows the median wall clock time for a single optimization of the acquisition function for each of the three knowledge gradient variants on the (2,2,2)-dimensional GP test problems.
The average is taken over the 100 repeats and 350 iterations (or 300 iterations for the case of the two-step KG).

Thanks to the Monte-Carlo approximation of the outer expectation in \cref{eq:acqf-joint-kg-approx}, the joint KG algorithm is immediately parallelized when implemented in PyTorch. However, the exact expectation in the alternating (\cref{eq:acqf-alternating-kg-adjustable-approx,eq:acqf-alternating-kg-fixed-approx}) and two-step KG (\cref{eq:acqf-twostep-kg-approx}) require more work to be parallelized, which is beyond the scope of this paper.

\begin{figure}[htbp]
    \centering
    \includegraphics[width=0.5\textwidth]{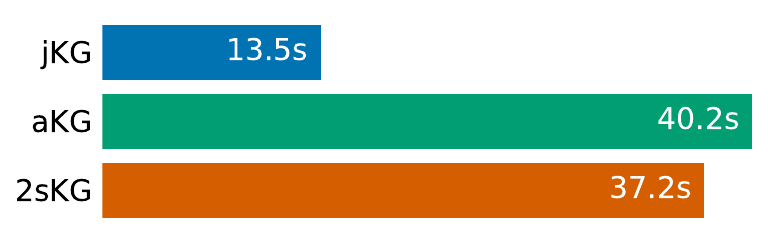}
    \caption{Median wall times required to optimize the acquisition functions on the (2,2,2)-dimensional GP sampled test problems without observation noise using 6 CPU cores.}
    \label{fig:timings-dims-222}
\end{figure}

\section{Discussion}

The results in \cref{sec:experiments} show the importance of jointly optimizing the fixed design and control strategy for the adjustable variables.
The joint knowledge gradient is simple, principled, and computationally feasible, with a theoretical guarantee of convergence and superior empirical results to the current state of the art \citep{xie2021globallocal}.
Furthermore, it provides not only the recommended fixed design but also a recommendation for the control policy for the adjustable variables.
The alternating knowledge gradient provides a natural alternative and demonstrates that the Monte Carlo approximation of the expectation over the hypothesized next observation does not adversely affect the results.

There are many possible future research directions. Stochastic programming typically involves explicit constraints, and while the algorithm can support optimizing over constrained spaces \(\sX\) and \(\sY\), it would be interesting to extend it to handle black-box constraints as in \citep{ungredda2024constrained}. Similarly, the algorithms could be extended to cover standard problems in stochastic programming such as multi-stage decisions, and different risk measures such as value-at-risk.

Finally, another limitation of the algorithm is its extensibility to higher dimensions. For example, in the wind farm layout problem \citep{chen2022windfarm}, there are typically hundreds of wind turbines leading to hundreds of parameters.
High-dimensional Bayesian optimization is an active area of research \citep{binois2022highdimensional} and recent ideas such as local search methods could be incorporated to tackle higher-dimensional two-stage stochastic programs.

\subsubsection*{Acknowledgments}
The authors would like to thank Andrew Nugent for his verification of the proofs in \cref{app:theoretical-results} and his suggestions for improving the clarity of our explanations.
We thank Wei Xie and Hua Zheng for providing their source code.

The first author was supported by the Engineering and Physical Sciences Research Council through the Mathematics of Systems II Centre for Doctoral Training at the University of Warwick (reference EP/S022244/1).
The second author was supported by the Flemish Government under the Flanders Artificial Intelligence Research program.

\subsubsection*{Data Availability Statement}
This work is entirely theoretical, there is no data underpinning this publication. 

\bibliography{main}
\bibliographystyle{tmlr}

\appendix
\section*{Appendix}
\section{Theoretical Results}\label{app:theoretical-results}
In this section, we establish the asymptotic consistency of the value associated with the recommendations of the joint knowledge gradient acquisition function presented in \cref{sec:algo-acqf} when viewed as an estimator for \(\sup_{\vx, \vg} \bar{h}(\vx, \vg)\). This was stated in the main text as \cref{thm:jointkg-consistency}.

\jointkgconsistency*

As noted in \cref{sec:acqf-consistency}, rather than proving consistency for a specific objective function \(h\), we instead prove it almost surely and in mean when \(h\) is a draw from a Gaussian process with the same distribution as the surrogate model \(f\).
Therefore, to simplify notation, we identify \(h=f\).

We also assume that the hyperparameters of the distribution of \(f\) remain fixed throughout the optimization, and are not refitted every iteration as described in \cref{alg:bo-basic} and \cref{alg:bo-codesign-joint-kg}.
This is the standard assumption in existing proofs of the asymptotic consistency of knowledge gradient acquisition functions \citep{bect2019supermartingale,toscanopalmerin2022integrands,buckingham2023latencies}.
However, it is important to note that it is likely that the equivalent theorem does not hold when the hyperparameters are fitted after each new observation.
Indeed, \citet{bull2011convergence} show that for standard GP priors with parameters sequentially re-estimated from the data, there exist smooth objectives \(f\) on which Bayesian Optimization using the expected improvement acquisition function does not converge.
Therefore, for the case of hyperparameters estimated from the data, we rely on experimental evidence that this does not happen in practice.

The proof of consistency is based on the work by \citep{bect2019supermartingale} and \citep{buckingham2023latencies}, but avoids the use of random measures which is found in \citep{bect2019supermartingale}.
It begins by showing that the joint knowledge gradient converges almost surely, uniformly to zero in \cref{thm:app-joint-kg-to-zero}. Further, the posterior mean and covariance of \(\bar{f}\) are shown to converge uniformly to continuous limits in \cref{thm:app-mubar-kbar-uniform-convergence} by observing that the posterior mean and covariance of \(f\) are martingales in Banach spaces. These two facts are combined to show in \cref{thm:app-fbar-minus-mubar-const-simultaneous} that the uniform limit, \(\bar{\mu}^\infty\), of the posterior means of \(\bar{f}\) has sample paths which are almost surely equal to the sample paths of \(\bar{f}\) (up to a constant). Whence, convergence of the values of the recommendations can be established in \cref{thm:app-jointkg-consistent}.

\subsection{Notation, statistical model and assumptions}
We assume that the spaces \(\sX \subset \R^{d_x}\), \(\sY \subset \R^{d_y}\) and \(\sU \subset \R^{d_u}\) are compact and that the objective \(f \sim \mathcal{GP}(\mu, k)\) is a Gaussian process with index set \(\tilde{\sX} = \sX \times \sY \times \sU\) defined on a probability space \((\Omega, \mathcal{F}, \mathbb{P})\).
We further assume that the prior mean function, \(\mu\), and covariance function, \(k\), are continuous, and that there exists a version of \(f\) with continuous sample paths.
A sufficient condition for this is that the mean function \(\mu\) is continuous and that the kernel satisfies \citep[Theorem~1.4.1]{adler2009randomfieldsgeometry}
\begin{multline}
    \exists 0 < c < \infty \text{ and } \exists \beta, \delta > 0 \text{ such that } \forall \tilde{\vx}, \tilde{\vx}' \in \tilde{\sX} \text{ with } \|\tilde{\vx} - \tilde{\vx}'\|_2 < \delta, \\
    \E[ |f(\tilde{\vx}) - f(\tilde{\vx}')|^2 ] = k(\tilde{\vx}, \tilde{\vx}) + k(\tilde{\vx}', \tilde{\vx}') - 2 k(\tilde{\vx}, \tilde{\vx}') \leq \frac{c}{\bigl| \log \|\tilde{\vx} - \tilde{\vx}' \|_2 \bigr|^{1+\beta}}.
\end{multline}

As in the main text, let \(\tilde{\rvx}^1, \tilde{\rvx}^2, \dots \in \tilde{\sX}\) denote the (data dependent) design points with \(\tilde{\rvx}^n = (\rvx^n, \rvy^n, \rvu^n)\) for each \(n\). The first \(n_0\) of these are selected according to a (possibly random) initial design such as a Latin Hypercube or scrambled Sobol' sequence, and the remainder are selected sequentially based on the observations made up to that point.
Let \(\rv^1, \rv^2, \dots \in \R\) be the corresponding (noisy) observations given by
\begin{equation}
    \forall n = 1,2,\dots \qquad \rv^n = f(\rvx^n, \rvy^n, \rvu^n) + \varepsilon^n \qquad\text{where}\qquad \forall n\quad \varepsilon^n \sim \mathcal{N}(0, \sigma^2) \;\text{i.i.d},
\end{equation}
where the \(\varepsilon^1, \varepsilon^2, \dots\) are also independent of \(f\).
This is \cref{eq:obs-model} from the main text, restated here for convenience.
Note that now we view the design points \(\tilde{\rvx}^1, \tilde{\rvx}^2\) as random vectors, since they depend on the random initial design and on the observation noise.
For each \(n\), let \(\mathcal{F}^n = \sigma(\tilde{\rvx}^1, \rv^1, \dots, \tilde{\rvx}^n, \rv^n)\) be the sigma-algebra generated by the first \(n\) design points and observations, so that \((\mathcal{F}^n)_{n=0}^\infty\) forms a filtration of \(\mathcal{F}\).

Our objective is not \(f\) but the expectation of \(f\) over the environmental variables \(\rvu\),
\begin{equation} \label{eq:app-objective}
    \bar{f}(\vx, \vg) = \E_\rvu[f(\vx, \vg(\rvu), \rvu)].
\end{equation}
Note here that \(\rvu\) and \(f\) are independent and so formally this is the expectation of \(f(\vx, \vg(\rvu), \rvu)\) conditional on the GP \(f\).

It will be useful later when talking about the posterior mean of \(\bar{f}\) to be able to freely interchange the order of integration over \(f\) and \(\rvu\). This is possible by Fubini's Theorem / the law of total expectation, provided that the expectation \(\E[f(\vx, \vg(\rvu), \rvu)]\) over both jointly is well defined.

\begin{lemma}
    We can exchange the order of integration over \(f\) and \(\rvu\) when calculating the posterior mean of \(\bar{f}\). That is,
    \begin{equation*}
        \bar{\mu}^n(\vx, \vg)
        = \E\bigl[\,\bar{f}(\vx, \vg) \,\big|\, \mathcal{F}^n \bigr]
        = \E_\rvu\bigl[ \mu^n(\vx, \vg(\rvu), \rvu) \bigr].
    \end{equation*}
\end{lemma}
\begin{proof}
    As observed, in order to apply Fubini's theorem, it is sufficient to check that \(\E[(f(\vx, \vg(\rvu), \rvu))_+] < \infty\) and \(\E[(f(\vx, \vg(\rvu), \rvu))_-] < \infty\), where \((\cdot)_+\) and \((\cdot)_-\) denote the positive and negative parts of the random variable, and the expectation is taken jointly over both \(f\) and \(\rvu\).
    This follows immediately from \citep[Theorem~2.9]{azais2009levelsets}.
    Indeed, write \(f = \mu + f_0\) with \(f_0\) a zero mean GP and \(\mu\) deterministic.
    Then \(f_0\) has almost surely continuous sample paths and \(\tilde{\sX}\) is compact, so \(\sup_{\tilde{\vx}} f_0(\tilde{\vx}) < \infty\) almost surely, and so the theorem gives \(\E|\sup_{\tilde{\vx}} f_0(\tilde{\vx})| < \infty\).
    Since \(f_0\) and \(-f_0\) have the same distribution, we also have \(\E|\inf_{\tilde{\vx}} f_0(\tilde{\vx})| < \infty\) and therefore
    \begin{alignat*}{4}
        \E[(f(\vx, \vg(\rvu), \rvu))_+]
        &\leq \E\biggl[\Bigl|\sup_{\tilde{\vx}} f(\tilde{\vx})\Bigr|\biggr]
        &&\leq \Bigl|\sup_{\tilde{\vx}} \mu(\tilde{\vx})\Bigr| &&+ \E\biggl[\Bigl|\sup_{\tilde{\vx}} f_0(\tilde{\vx})\Bigr|\biggr]
        &&< \infty, \\
        \E[(f(\vx, \vg(\rvu), \rvu))_-]
        &\leq \E\biggl[\Bigl|\inf_{\tilde{\vx}} f(\tilde{\vx})\Bigr|\biggr]
        &&\leq \Bigl|\inf_{\tilde{\vx}} \mu(\tilde{\vx})\Bigr| &&+ \E\biggl[\Bigl|\inf_{\tilde{\vx}} f_0(\tilde{\vx})\Bigr|\biggr]
        &&< \infty,
    \end{alignat*}
    as required.
\end{proof}

The sample paths of \(f\) are continuous and \(\tilde{\sX}\) is compact, so \(f\) is bounded and attains its bounds almost surely. It turns out that the objective \(\bar{f}\) also attains its bounds, so our optimization problem is well-posed.

\begin{lemma} \label{thm:app-fbar-attains-bound}
    The objective \(\bar{f}\) is bounded and attains its bounds almost surely.
    Furthermore, there exists a choice \(\rvx \in \sX, \rvg: \sU \to \sY\) maximizing (resp. minimizing) \(\bar{f}\) such that \(\rvg\) has Borel-measurable sample paths.
\end{lemma}
\begin{proof}
    We prove the result for the upper bound, but the equivalent for the lower bound is analogous.

    The sample paths of \(f\) are continuous and \(\tilde{\sX}\) is compact, so they are uniformly continuous.
    Therefore, the sample paths of the map \((\vx, \vu) \mapsto \max_\rvy f(\vx, \rvy, \vu)\) are also uniformly continuous.
    Hence, the map \(\vx \mapsto \int \max_\rvy f(\vx, \rvy, \vu) \;\mathbb{P}(\mathrm{d}\vu)\) also has (uniformly) continuous sample paths.
    Therefore, since \(\sX\) is compact, this map is bounded and attains its bounds, so let
    \begin{equation*}
        \rvx^* \in \argmax_{\rvx \in \sX} \E_\rvu \Bigl[ \max_\rvy f(\rvx, \rvy, \vu) \Bigr].
    \end{equation*}
    Now, for each \(\vu \in \sU\), the map \(\vy \mapsto f(\rvx^*, \vy, \vu)\) has continuous sample paths, and \(\sY\) is compact, so it is bounded and attains its bounds.
    Therefore, we can define a map \(\rvg^* : \sU \to \sY\) by
    \begin{equation*}
        \rvg^*(\vu) \in \argmax_{\rvy \in \sY} f(\rvx^*, \rvy, \vu).
    \end{equation*}
    By construction, the pair \(\rvx^*, \rvg^*\) are an upper bound for \(\bar{f}\).

    It remains to be shown that \(\rvg^*\) can be chosen to have measurable sample paths with respect to the Borel sigma-algebras on \(\sU\) and \(\sY\).
    This is the content of Proposition 7.33 in \citep{bertsekas1996soc}.
\end{proof}

In order to argue rigorously about the knowledge gradient, we must make the formula in \cref{eq:acqf-joint-kg-conceptual} concrete.
Let \(\rv_{\tilde{\vx}} = f(\tilde{\vx}) + \varepsilon\) with \(\varepsilon \sim \mathcal{N}(0, \sigma^2)\) represent a possible noisy observation at \(\tilde{\vx} \in \tilde{\sX}\).
For each \(n\), let \(\bar{\mu}^{n+1}(\vx', \vg';\, \tilde{\vx}, \rv_{\tilde{x}}) = \E[\bar{f}(\vx', \vg') \,|\, \mathcal{F}^n, \rv_{\tilde{\vx}}]\) be the expectation of \(\bar{f}\) conditional on the \(n\) observations so far, and the proposed new observation \(\rv_{\tilde{\vx}}\) taken at \(\tilde{\vx}\).
Note that \(\bar{\mu}^{n+1}(\vx', \vg') = \bar{\mu}^{n+1}(\vx', \vg';\, \tilde{\rvx}^{n+1}, \rv^{n+1})\), so this is consistent with our notation from the main text.
With this notation, \cref{eq:acqf-joint-kg-conceptual} becomes
\begin{equation} \label{eq:app-acqf-joint-kg-conceptual}
    \alpha_\text{jKG}^n(\tilde{\vx}) =
    \E\Biggl[\, \max_{\substack{\rvx' \in \sX \\ \rvg': \sU \to \sY }} \bar{\mu}^{n+1}(\rvx', \rvg';\, \tilde{\vx}, \rv_{\tilde{\vx}}) \,\Bigg|\, \mathcal{F}^n \Biggr]
    - \max_{\substack{\rvx' \in \sX \\ \rvg': \sU \to \sY }} \bar{\mu}^{n}(\rvx', \rvg').
\end{equation}
We are justified in using a maximum instead of a supremum here because the posterior means \(\bar{\mu}^n\) are all bounded and attain their bounds.
\begin{lemma} \label{thm:app-mubar-attains-bound}
    For all \(n\), the posterior mean \(\bar{\mu}^n\) is bounded and attains its bounds.
    Furthermore, there exists a choice \(\rvx^{*n} \in \sX\) and \(\rvg^{*n}: \sU \to \sY\) maximizing (resp. minimizing) \(\bar{\mu}^\infty\) such that \(\rvg^{*n}\) has Borel-measurable sample paths.
\end{lemma}
\begin{proof}
    Let \(n \in \mathbb{N}_0\). Then \(\mu^n\) is a stochastic process with continuous sample paths and the result follows by the same argument as \cref{thm:app-fbar-attains-bound}.
\end{proof}
This result also means that our recommendations \(\rvx^{*n}, \rvg^{*n} \in \argmax_{\rvx, \rvg} \bar{\mu}^n(\rvx, \rvg)\) from \cref{eq:recommendation-conceptual} in the main text are well defined.

So far, we have not specified the mechanism by which the design points \(\tilde{\rvx}^1, \tilde{\rvx}^2, \dots\) are chosen.
Intuitively, we wish to choose them to maximize the joint knowledge gradient in \cref{eq:app-acqf-joint-kg-conceptual}.
However, in the case \(\sigma = 0\) of noiseless observations, the knowledge gradient is not continuous at the observations and, therefore, it is not obvious that it attains its supremum (even though \(\tilde{\sX}\) is compact).
Furthermore, in practice, we can only hope to maximize the knowledge gradient approximately.
To that end, we assume that for each \(n \geq n_0\), we optimize \(\alpha_\text{jKG}^n\) to within some small \(\eta^n\) of the optimum,
\begin{equation} \label{eq:app-maximize-acqf}
    \alpha_\text{jKG}^n(\tilde{\rvx}^{n+1}) > \sup_{\tilde{\rvx}' \in \tilde{\sX}} \alpha_\text{jKG}^n(\tilde{\rvx}') - \eta^n,
\end{equation}
where \((\eta^n)_{n=n_0}^\infty\) is a positive sequence with \(\eta^n \to 0\) as \(n \to \infty\).
Further, we assume that each \(\tilde{\rvx}^{n+1}\) is \(\mathcal{F}^n\)-measurable. That is, that the optimization is deterministic given the first \(n\) observations.

We conclude this section with a short proof that the joint knowledge gradient is non-negative.
This was stated as \cref{thm:joint-kg-nonnegative} in the main text.
\begin{proposition}[\cref{thm:joint-kg-nonnegative} from the main text] \label{thm:app-jointkg-nonneg}
    The joint knowledge gradient is almost surely non-negative,
    \begin{equation*}
        \forall n \in \mathbb{N}_0\; \forall \tilde{\vx} \in \tilde{\sX} \quad \alpha_\text{jKG}^n(\tilde{\vx}) \geq 0 \qquad \text{a.s.}
    \end{equation*}
\end{proposition}
\begin{proof}
    Let \(\tilde{\vx} \in \tilde{\sX}\) and \(n \in \mathbb{N}_0\).
    Then, for all \(\vx'' \in \sX\) and \(\vg'' : \sU \to \sY\),
    \begin{equation*}
        \max_{\rvx', \rvg'} \E\bigl[\bar{f}(\rvx', \rvg') \,\big|\, \mathcal{F}^n, \rv_{\tilde{\vx}}\bigr]
        \geq
        \E\bigl[\bar{f}(\vx'', \vg'') \,\big|\, \mathcal{F}^n, \rv_{\tilde{\vx}}\bigr].
    \end{equation*}
    Taking expectation conditional on \(\mathcal{F}^n\) gives
    \begin{equation*}
        \E\biggl[
            \max_{\rvx', \rvg'} \E\bigl[\bar{f}(\rvx', \rvg' \,\big|\, \mathcal{F}^n, \rv_{\tilde{\vx}})\bigr]
        \,\bigg|\,
            \mathcal{F}^n
        \biggr]
        \geq
        \E\bigl[\bar{f}(\vx'', \vg'') \,\big|\, \mathcal{F}^n\bigr]
        = \bar{\mu}^n(\vx'', \vg'').
    \end{equation*}
    Whence, this holds for the maximum over \(\vx''\) and \(\vg''\),
    \begin{gather*}
        \E\biggl[
            \max_{\rvx', \rvg'} \E\bigl[\bar{f}(\rvx', \rvg') \,\big|\, \mathcal{F}^n, \rv_{\tilde{\vx}}\bigr]
        \,\bigg|\,
            \mathcal{F}^n
        \biggr]
        \geq
        \max_{\rvx'', \rvg''} \bar{\mu}^n(\rvx'', \rvg'') \\
        \Rightarrow\qquad \alpha_\text{jKG}^n(\tilde{\vx})
        = \E\biggl[
            \max_{\rvx', \rvg'} \E\bigl[\bar{f}(\rvx', \rvg') \,\big|\, \mathcal{F}^n, \rv_{\tilde{\vx}}\bigr]
        \,\bigg|\,
            \mathcal{F}^n
        \biggr]
        - \max_{\rvx'', \rvg''} \bar{\mu}^n(\rvx'', \rvg'')
        \geq 0.
    \end{gather*}
\end{proof}

\subsection{\texorpdfstring{Convergence of \(\alpha_\text{jKG}^n\) to zero}{Convergence of joint KG to zero}}

In this section we will prove that almost surely, as we collect more data, the joint knowledge gradient converges uniformly to zero. This is the first step in proving that the recommendations converge to optimal.
In the field of sequential uncertainty reduction \citep{bect2019supermartingale}, the concept of residual uncertainty is central.
\begin{definition} \label{defn:app-residual-uncert}
    For \(n \in \mathbb{N}_0\), the \emph{residual uncertainty} associated with the joint knowledge gradient is the expected difference between the maximum of the conditional expectation of \(\bar{f}\) and the maximum of \(\bar{f}\) itself,
    \begin{equation} \label{eq:app-residual-uncert}
        H^n = \E\Biggl[\, \max_{\substack{\rvx' \in \sX \\ \rvg' : \sU \to \sY}} \bar{f}(\rvx', \rvg') \,\Bigg|\, \mathcal{F}^n \Biggr] - \max_{\substack{\rvx' \in \sX \\ \rvg' : \sU \to \sY}} \underbrace{\E\bigl[\,\bar{f}(\rvx', \rvg') \,\big|\, \mathcal{F}^n \bigr]}_{\bar{\mu}^n(\rvx', \rvg')}.
    \end{equation}
    It measures how well the maximum of the posterior mean \(\bar{\mu}^n\) approximates the maximum of \(\bar{f}\).
\end{definition}
\begin{remark} \label{rmk:app-joint-kg-and-resudual-uncertainty}
    For each \(n \geq 0\), the joint knowledge gradient at \(\tilde{\rvx}^{n+1}\) is the expected reduction in residual uncertainty,
    \begin{equation} \label{eq:app-joint-kg-and-residual-uncertainty}
        \alpha_\text{jKG}^n(\tilde{\rvx}^{n+1}) = H^n - \E[H^{n+1} \,|\, \mathcal{F}^n].
    \end{equation}
    Indeed, the first term, \(\E[\max_{\rvx', \rvg'} \bar{f}(\rvx', \rvg') \,|\, \mathcal{F}^n]\), cancels and we have
    \begin{align*}
        H^n - \E[ H^{n+1} \,|\, \mathcal{F}^n ] = \E\biggl[ \max_{\rvx', \rvg'} \bar{\mu}^{n+1}(\rvx', \rvg') \,\bigg|\, \mathcal{F}^n \biggr] - \max_{\rvx', \rvg'} \bar{\mu}^n(\rvx', \rvg') = \alpha_\text{jKG}^n(\tilde{\rvx}^{n+1}).
    \end{align*}
    The last equality here uses the fact that \(\tilde{\rvx}^{n+1}\) is \(\mathcal{F}^n\)-measurable.
\end{remark}
\begin{lemma} \label{thm:resid-uncert-well-def}
    The residual uncertainty in \cref{defn:app-residual-uncert} is well defined and non-negative.
    That is, for all \(n\), \(0 \leq H^n < \infty\) almost surely.
\end{lemma}
\begin{proof}
    Let \(n \in \mathbb{N}_0\).
    We will first show that \(|H^n| < \infty\) almost surely.
    We will bound each term separately by the expected maximum of the absolute value of \(f\), then show that this is finite.
    Recalling that \(\bar{f}(\vx', \vg') = \E_{\rvu'}[f(\vx', \vy', \rvu')]\), we have for the first term in \cref{eq:app-residual-uncert},
    \begin{align*}
        \biggl| \E\Bigl[\, \max_{\rvx', \rvg'} \bar{f}(\rvx', \rvg') \,\Big|\, \mathcal{F}^n \Bigr] \biggr|
        &\leq \E\Bigl[ \max_{\rvx', \rvg'} \bigl| \bar{f}(\rvx', \rvg') \bigr| \,\Big|\, \mathcal{F}^n \Bigr] \\
        &\leq \E\Bigl[\, \max_{\rvx', \rvg'} \E_{\rvu'} \bigl[ | f(\rvx', \rvg'(\rvu'), \rvu')| \bigr] \,\Big|\, \mathcal{F}^n \Bigr] \\
        &\leq \E\Bigl[\, \max_{\rvx', \rvy', \rvu'} |f(\rvx', \rvy', \rvu')| \,\Big|\, \mathcal{F}^n \Bigr].
    \end{align*}
    Similarly, for the second term,
    \begin{align*}
        \biggl| \max_{\rvx', \rvg'} \E\bigl[ \bar{f}(\rvx', \rvg') \,\big|\, \mathcal{F}^n \bigr] \biggr|
        &\leq \max_{\rvx', \rvg'} \E\Bigl[ \bigl| \bar{f}(\rvx', \rvg') \bigr| \,\Big|\, \mathcal{F}^n \Bigr] \\
        &\leq \max_{\rvx', \rvg'} \E\Bigl[ \E_{\rvu'} \bigl[ | f(\rvx', \rvg'(\rvu'), \rvu')| \bigr] \,\Big|\, \mathcal{F}^n \Bigr] \\
        &\leq \max_{\rvx', \rvy'} \E\Bigl[ \max_{\rvu'} | f(\rvx', \rvy', \rvu')| \,\Big|\, \mathcal{F}^n \Bigr] \\
        &\leq \E\Bigl[\, \max_{\rvx', \rvy', \rvu'} |f(\rvx', \rvy', \rvu')| \,\Big|\, \mathcal{F}^n \Bigr].
    \end{align*}
    Thus,
    \begin{equation*}
        |H^n| \leq \biggl| \E\Bigl[ \max_{\rvx', \rvg'} \bar{f}(\rvx', \rvg') \,\Big|\, \mathcal{F}^n \Bigr] \biggr| + \biggl| \max_{\rvx', \rvg'} \E[ \bar{f}(\rvx', \rvg') \,|\, \mathcal{F}^n ] \biggr| \leq 2\, \E\biggl[ \max_{\rvx', \rvy', \rvu'} |f(\rvx', \rvy', \rvu')| \,\bigg|\, \mathcal{F}^n \biggr].
    \end{equation*}
    This upper bound is non-negative, so it suffices to show that its expectation is finite.
    That is, we wish to show
    \begin{equation*}
        \E\biggl[ \max_{\rvx', \rvy', \rvu'} |f(\rvx', \rvy', \rvu')| \biggr] < \infty.
    \end{equation*}
    This follows from \cref{thm:app-supf-finite-moments} which follows this proof (setting \(p=1\)).
    
    It remains to show that \(H^n > 0\) almost surely.
    This follows because the expectation of the maximum must be at least the maximum of the expectation, by a very similar argument to \cref{thm:app-jointkg-nonneg}.
\end{proof}

\begin{remark} \label{rmk:nonneg-sup-mart}
    It is possible to show that \((H^n)_{n=0}^\infty\) form a non-negative supermartingale with respect to \(\mathcal{F}^n\), and this is the route taken in \citep{bect2019supermartingale}. However, we will be able to get by with \cref{thm:app-jointkg-nonneg} and \cref{rmk:app-joint-kg-and-resudual-uncertainty}.
\end{remark}

\begin{lemma} \label{thm:app-supf-finite-moments}
    For any \(1 \leq p < \infty\),
    \begin{equation*}
        \E\biggl[ \biggl( \sup_{\tilde{\vx} \in \tilde{\sX}} |f(\tilde{\vx})| \biggr)^p \biggr] < \infty.
    \end{equation*}
\end{lemma}
\begin{proof}
    Let \(1 \leq p < \infty\).
    By Jensen's inequality,
    \begin{equation*}
        \left( \sup_{\tilde{\vx}' \in \tilde{\sX}} |f(\tilde{\vx}')| \right)^p
        = \left(\max\left\{ \sup_{\tilde{\vx}' \in \tilde{\sX}} f(\tilde{\vx}'),\, \sup_{\tilde{\vx}' \in \tilde{\sX}} -f(\tilde{\vx}') \right\} \right)^p
        \leq \left|\sup_{\tilde{\vx}'} f(\tilde{\vx}')\right|^p + \left|\sup_{\tilde{\vx}'} -f(\tilde{\vx}')\right|^p.
    \end{equation*}
    Assume first that \(f\) has mean zero.
    Then \(f\) and \(-f\) are identically distributed, and
    \begin{equation*}
        \E\left[ \left( \sup_{\tilde{\vx}' \in \tilde{\sX}} |f(\tilde{\vx}')| \right)^p \right] \leq 2 \E\left[ \left| \sup_{\tilde{\vx}'} f(\tilde{\vx}') \right|^p \right].
    \end{equation*}
    This upper bound is finite, by \citep[Equation~2.34]{azais2009levelsets}.
    
    In the case where \(f\) does not have mean zero, write \(f(\tilde{\vx}) = \mu(\tilde{\vx}) + f_0(\tilde{\vx})\) where \(f_0\) is a zero mean GP and \(\mu\) is deterministic.
    Then \(\sup_{\tilde{\vx}'} |f(\tilde{\vx}')| \leq \sup_{\tilde{\vx}'} |\mu(\tilde{\vx}')| + \sup_{\tilde{\vx}'} |f_0(\tilde{\vx}')|\) and so the integer moments of \(\sup_{\tilde{\vx}'} |f(\tilde{\vx}')|\) are bounded above by a linear combination of the integer moments of \(\sup_{\tilde{\vx}'} |f_0(\tilde{\vx}')|\) of degree at most \(p\), which are finite. For non-integer \(p\), the result follows by rounding up to the next largest integer.
\end{proof}

\begin{theorem} \label{thm:app-joint-kg-to-zero}
    The joint knowledge gradient converges almost surely, uniformly to zero. That is,
    \begin{equation*}
        \mathbb{P}\biggl(\, \sup_{\tilde{\rvx} \in \tilde{\sX}} \alpha_\text{jKG}^n(\tilde{\rvx}) \to 0 \quad\text{as}\quad n \to 0 \biggr) = 1.
    \end{equation*}
\end{theorem}
\begin{proof}
    For each integer \(n \geq 0\), let \(\Delta^{n+1} = H^n - H^{n+1}\) and \(\overline{\Delta}^{n+1} = \E[\Delta^{n+1} \,|\, \mathcal{F}^n]\).
    Therefore,
    \begin{equation*}
        \overline{\Delta}^{n+1} = \alpha_\text{jKG}^n(\tilde{\rvx}^{n+1}) \geq 0
    \end{equation*}
    by \cref{rmk:app-joint-kg-and-resudual-uncertainty} and \cref{thm:app-jointkg-nonneg}.
    Observe that the telescopic sum \(\sum_{j=1}^n \Delta^j = H^0 - H^n\).
    Therefore,
    \begin{equation*}
        \E\biggl[ \sum_{j=1}^n \overline{\Delta}^{j} \biggr] = \E\biggl[ \sum_{j=1}^n \Delta^j \biggr] = \E[H^0 - H^n] \leq \E[H^0] < \infty
    \end{equation*}
    by \cref{thm:resid-uncert-well-def}.
    Hence, since \(\overline{\Delta}^n\) is almost surely non-negative for all \(n\), we have \(\E[\sum_{n=1}^\infty \overline{\Delta}^n] < \infty\) by the monotone convergence theorem. Whence, \(\overline{\Delta}^n \to 0\) as \(n \to \infty\) almost surely.

    We chose our design points to maximize the knowledge gradient as in \cref{eq:app-maximize-acqf}. Therefore, for all \(n \geq n_0\),
    \begin{equation*}
        0 \leq \sup_{\tilde{\rvx} \in \tilde{\sX}} \alpha_\text{jKG}^n(\tilde{\rvx})
        < \alpha_\text{jKG}^n(\tilde{\rvx}^{n+1}) + \eta^n
        = \overline{\Delta}^{n+1} + \eta^n \to 0 \quad \text{as} \quad n \to \infty \quad \text{a.s.}
    \end{equation*}
    which establishes the result.
\end{proof}

\subsection{Uniform convergence of the posterior mean and covariance functions}
In this section, we establish the uniform convergence of the posterior mean and covariance functions of \(f\) and \(\bar{f}\).
This will be an important ingredient in proving convergence of the values of the recommendations.

Let \(\mathcal{F}^\infty = \sigma(\cup_{n=0}^\infty \mathcal{F}^n)\) be the sigma-algebra generated by the \(\mathcal{F}^n\).
Define processes
\begin{subequations}
\begin{alignat}{2}
    \mu^\infty &: \tilde{\sX} \to \R &\qquad \tilde{\vx} &\mapsto \E[f(\tilde{\vx}) \,|\, \mathcal{F}^\infty], \\
    k^\infty &: \tilde{\sX} \times \tilde{\sX} \to \R &\qquad (\tilde{\vx}, \tilde{\vx}') &\mapsto \mathrm{Cov}[f(\tilde{\vx}), f(\tilde{\vx}') \,|\, \mathcal{F}^\infty].
\end{alignat}
\end{subequations}

The proof uses the theory of martingales in Banach spaces.
The Gaussian process, \(f\), is defined on the probability space \((\Omega, \mathcal{F}, \mathbb{P})\) and can be viewed as a function \(f : \Omega \times \tilde{\sX} \to \R,\, (\omega, \tilde{\vx}) \mapsto f(\tilde{\vx})(\omega)\).
The sigma-algebra \(\mathcal{F}\) contains all the events generated by the finite-dimensional distributions of \(f\) and, since sigma-algebras are closed under countable intersections, \(\mathcal{F}\) also contains all events which can be described by the value of \(f\) at countably many locations in \(\tilde{\sX}\).
We are considering the version of \(f\) with continuous sample paths and \(\tilde{\sX}\) is separable, and so \(\mathcal{F}\) contains all events described by the sample paths of \(f\).
That is, we can view \(f\) as a measurable function \(f: \Omega \to \mathcal{C}(\tilde{\sX}),\, \omega \mapsto (\tilde{\vx} \mapsto f(\omega, \tilde{\vx}))\) taking values in the space \(\mathcal{C}(\tilde{\sX})\) of continuous, real-valued functions on \(\tilde{\sX}\).
In other words, \(f\) is a random element in \(\mathcal{C}(\tilde{\sX})\), defined on \((\Omega, \mathcal{F}, \mathbb{P})\).

Since \(\tilde{\sX}\) is compact, the space \(\mathcal{C}(\tilde{\sX})\) of continuous functions \(\tilde{\sX} \to \R\) forms a Banach space when equipped with the supremum norm, \(\|\cdot\|_\infty\).
For \(1 \leq p < \infty\), we denote by \(L_p(\Omega, \mathcal{F}, \mathbb{P};\, \mathcal{C}(\tilde{\sX}))\), the space of (Bochner-) measurable functions, \(f\), with \(\int \|f\|_\infty^p 
 \mathrm{d}\mathbb{P} = \E[\|f\|_\infty^p] < \infty\).
 Note that, since \(\mathcal{C}(\tilde{\sX})\) is separable, the notion of Bochner-measurability coincides with that of Borel-measurability, by Pettis' theorem \citep[Theorem~1.47]{vanNeerven2022funcanal}.

\begin{proposition} \label{thm:app-mu-k-uniform-convergence}
    For any choice of design points \(\tilde{\rvx}^1, \tilde{\rvx}^2, \dots\), the sequences of stochastic processes \(\mu^n \to \mu^\infty\) and \(k^n \to k^\infty\) converge uniformly as \(n \to \infty\), both almost surely and in \(L^p\) for all \(1 \leq p < \infty\). Furthermore, the limits \(\mu^\infty\) and \(k^\infty\) are continuous.
\end{proposition}
\begin{proof}
    For any \(1 \leq p < \infty\), \(\E[\|f\|_\infty^p] < \infty\) by \cref{thm:app-supf-finite-moments} and so \(f \in L_p(\Omega, \mathcal{F}, \mathbb{P};\, \mathcal{C}(\tilde{\sX}))\).
    The conditional means \(\mu^n\) and \(\mu^\infty\) have continuous sample paths since \(L_p(\Omega, \mathcal{F}, \mathbb{P};\, \mathcal{C}(\tilde{\sX}))\) is closed under conditional expectation \citep[Proposition 1.10]{pisier2016martingalesbanach}, and so \((\mu^n)_{n=0}^\infty\) is a martingale in \(\mathcal{C}(\tilde{\sX})\).
    Therefore, \(\mu^n \to \mu^\infty\) in \(L^p(\Omega, \mathcal{F}, \mathbb{P};\, \mathcal{C}(\tilde{\sX}))\) and almost surely by Theorems 1.14 and 1.30 in \citep{pisier2016martingalesbanach}, respectively.
    Since this space uses the supremum norm, this corresponds to uniform convergence of the processes.

    We can use the same argument to show convergence of the second moments, \(\mu_{(2)}^n(\tilde{\vx}, \tilde{\vx}') = \E[ f(\tilde{\vx}) f(\tilde{\vx}') \,|\, \mathcal{F}^n]\) in the Banach space of continuous, real-valued functions \(\mathcal{C}(\tilde{\sX} \times \tilde{\sX})\).
    The process \(f_{(2)}(\tilde{\vx}, \tilde{\vx}') = f(\tilde{\vx}) f(\tilde{\vx}')\) has \(\E[\|f_{(2)}\|_\infty^p] < \infty\) for all \(1 \leq p < \infty\) and so \(f_{(2)} \in L_p(\Omega, \mathcal{F}, \mathbb{P},\;\mathcal{C}(\tilde{\sX} \times \tilde{\sX}))\).
    Thus, by \citep[Theorems 1.14 and 1.30]{pisier2016martingalesbanach}, \(\mu_{(2)}^n \to \mu_{(2)}^\infty\) almost surely and in \(L_p(\Omega, \mathcal{F}, \mathbb{P},\; \mathcal{C}(\tilde{\sX} \times \tilde{\sX}))\), where \(\mu_{(2)}^\infty(\tilde{\vx}, \tilde{\vx}') = \E[f(\tilde{\vx}) f(\tilde{\vx}') \,|\, \mathcal{F}^\infty]\).
    Therefore, \(k^n \to k^\infty\) uniformly both almost surely and in \(L_p(\Omega, \mathcal{F}, \mathbb{P},\; \mathcal{C}(\tilde{\sX} \times \tilde{\sX}))\), since \(k^n(\tilde{\vx}, \tilde{\vx}') = \mu_{(2)}^n(\tilde{\vx}, \tilde{\vx}') - \mu^n(\tilde{\vx})\mu^n(\tilde{\vx}')\) for all \(n \in \mathbb{N}_0 \cup \{\infty\}\).
    Finally, \(k^\infty\) has continuous sample paths, since it is a random element in \(\mathcal{C}(\tilde{\sX} \times \tilde{\sX})\).
\end{proof}

We now prove the analogous result of \cref{thm:app-mu-k-uniform-convergence} for the average performance objective, \(\bar{f}\).
For each \(n \in \mathbb{N}_0 \cup \{\infty\}\), \(\vx, \vx' \in \sX\) and \(\vg, \vg': \sU \to \sY\), define the posterior mean and covariance functions of \(\bar{f}\) as
\begin{subequations}
    \begin{align}
        \bar{\mu}^n(\vx, \vg) &= \E[\bar{f}(\vx, \vg) \,|\, \mathcal{F}^n], \\
        \bar{k}^n(\vx, \vg;\, \vx', \vg') &= \mathrm{Cov}[\bar{f}(\vx, \vg), \bar{f}(\vx', \vg') \,|\, \mathcal{F}^n].
    \end{align}
\end{subequations}
Note that we do not include a statement of continuity of \(\bar{\mu}^\infty\) and \(\bar{k}^\infty\) since we have not defined a topology on the set \(\sX \times \{\sU \to \sY\}\).
\begin{proposition} \label{thm:app-mubar-kbar-uniform-convergence}
    For any choice of design points \(\tilde{\rvx}^1, \tilde{\rvx}^2, \dots\), the sequences of stochastic processes \(\bar{\mu}^n \to \bar{\mu}^\infty\) and \(\bar{k}^n \to \bar{k}^\infty\) converge uniformly as \(n \to \infty\), both almost surely and in \(L^p\) for all \(1 \leq p < \infty\).
\end{proposition}
\begin{proof}
    Let \(n \in \mathbb{N}_0\).
    Then
    \begin{equation*}
        \sup_{\substack{\rvx \in \sX \\ \rvg : \sU \to \sY}} \bigl|\bar{\mu}^n(\rvx, \rvg) - \bar{\mu}^\infty(\rvx, \rvg)\bigr| \leq \sup_{\substack{\rvx \in \sX \\ \rvy \in \sY \\ \rvu \in \sU}} \bigl|\mu^n(\rvx, \rvy, \rvu) - \mu^\infty(\rvx, \rvy, \rvu)\bigr| \to 0
    \end{equation*}
    as \(n \to \infty\), both almost surely and in \(L^p\) for all \(1 \leq p < \infty\) by \cref{thm:app-mu-k-uniform-convergence}.
    Similarly,
    \begin{equation*}
        \sup_{\substack{\rvx, \rvx' \in \sX \\ \rvg, \rvg' : \sU \to \sY}} \bigl| \bar{k}^n(\rvx, \rvg;\, \rvx', \rvg') - \bar{k}^\infty(\rvx, \rvg;\, \rvx', \rvg') \bigr|
        \leq
        \sup_{\tilde{\rvx}, \tilde{\rvx}' \in \tilde{\sX}} \bigl| k^n(\tilde{\rvx}, \tilde{\rvx}') - k^\infty(\tilde{\rvx}, \tilde{\rvx}') \bigr|
        \to 0
    \end{equation*}
    as \(n \to \infty\), both almost surely and in \(L^p\) for all \(1 \leq p < \infty\) again by \cref{thm:app-mu-k-uniform-convergence}.
\end{proof}

We finish this section with a lemma which will be useful later, and which is also proved using \cref{thm:app-supf-finite-moments}.
\begin{lemma} \label{thm:app-supfbar-finite-expectation}
    All moments of the supremum of \(|\bar{f}|\) are finite. That is,
    \begin{equation*}
        \forall 1 \leq p < \infty \qquad \E\Biggl[ \biggl( \sup_{\substack{\rvx \in \sX \\ \rvg : \sU \to \sY}} |\bar{f}(\rvx, \rvg)| \biggr)^p \Biggr] < \infty.
    \end{equation*}
\end{lemma}
\begin{proof}
    Let \(1 \leq p < \infty\).
    From \cref{thm:app-supf-finite-moments} we have \(\E[\sup_{\rvx, \rvy, \rvu} |f(\rvx, \rvy, \rvu)|^p] < \infty\).
    Therefore, 
    \begin{equation*}
        \E\Biggl[ \biggl( \sup_{\substack{\rvx \in \sX \\ \rvg : \sU \to \sY}} |\bar{f}(\rvx, \rvg)| \biggr)^p \Biggr]
        =
        \E\Biggl[ \biggl( \sup_{\substack{\rvx \in \sX \\ \rvg : \sU \to \sY}} \Bigl| \E_\rvu \bigl[ f(\rvx, \rvg(\rvu), \rvu) \bigr] \Bigr| \biggr)^p \Biggr]
        \leq
        \E\biggl[ \sup_{\rvx, \rvy, \rvu} |f(\rvx, \rvy, \rvu)|^p \biggr]
        < \infty.
    \end{equation*}
\end{proof}

\subsection{Convergence to optimal recommendations}
The main result in this section is to prove \cref{thm:jointkg-consistency} from the main text. That is, that the objective values of the recommendations generated using joint knowledge gradient converge to the optimal values.
The lions share of the work is done by the following proposition, which shows that the limiting covariance function \(\bar{k}^\infty\) is pointwise almost surely constant.
This is then used in \cref{thm:app-fbar-minus-mubar-const-pointwise} to show that the limiting posterior mean \(\bar{\mu}^\infty\) is pointwise almost surely equal to \(\bar{f}\) up to a constant.
The result is extended to show that the sample paths are almost surely equal up to a constant in \cref{thm:app-fbar-minus-mubar-const-simultaneous} and from this we deduce the convergence of the values of the recommendations the the supremum of \(\bar{f}\).

\begin{proposition} \label{thm:app-kbar-constant}
    There exists a random variable \(\rc_1 \in \R\) such that for all \(\vx, \vx' \in \sX\) and \(\vg, \vg': \sU \to \sY\),
    \begin{equation*}
        \mathbb{P}\biggl( \bar{k}^\infty(\vx, \vg;\, \vx', \vg') = \rc_1 \biggr) = 1.
    \end{equation*}
\end{proposition}
Intuitively, we might want to show that we can take \(\rc_1 = 0\) here, but this turns out not to be necessary for the proof of \cref{thm:jointkg-consistency}.
\begin{proof}
By \cref{thm:app-mubar-attains-bound}, the posterior mean \(\bar{\mu}^n(\vx, \vg) = \E_\rvu[f(\vx, \vg(\rvu), \rvu)]\) attains its bounds, for each \(n\).
Thus, for each \(n \in \mathbb{N}_0\), let \(\rvx^{*n}, \rvg^{*n} \in \argmax_{\rvx, \rvg} \bar{\mu}^n(\rvx, \rvg)\) and for each \(\tilde{\vx} \in \tilde{\sX}\), \(\vx' \in \sX\) and \(\vg': \sU \to \sY\) let
\begin{equation*}
    \rw_{\tilde{\vx}}^n(\vx', \vg') = \E\Bigl[ \bar{f}(\vx', \vg') - \bar{f}(\rvx^{*n}, \rvg^{*n}) \,\Big|\, \mathcal{F}^n, \rv_{\tilde{\vx}} \Bigr]
\end{equation*}
where \(\rv_{\tilde{\vx}} = f(\tilde{\vx}) + \varepsilon\) is a hypothesized noisy observation at \(\tilde{\vx}\) with \(\varepsilon \sim \mathcal{N}(0, \sigma^2)\).
Thus, for each \(\tilde{\vx} \in \tilde{\sX}\), \(\alpha_\text{jKG}^n(\tilde{\vx}) = \E[\sup_{\vx', \vg'} \rw_{\tilde{\vx}}^n(\vx', \vg') \,|\, \mathcal{F}^n]\).

We will show that in the limit \(n \to \infty\), the \(\rw_{\tilde{\vx}}^n(\vx', \vg')\) are uniformly not very positive, and use this to deduce the proposition statement.

\begin{claim}
    Let \(\delta > 0\). There exists a random sequence \((\rho_\delta^n)_{n=0}^\infty\) with \(\rho_\delta^n \to 0\) almost surely as \(n \to \infty\), such that for all \(\tilde{\vx} \in \tilde{\sX}\), \(\vx' \in \sX\) and \(\vg' : \sU \to \sY\) and for all \(n \in \mathbb{N}_0\), the conditional probabilities \(\mathbb{P}(\rw_{\tilde{\vx}}^n(\vx', \vg') > \delta \,|\, \mathcal{F}^n) \leq \rho_\delta^n\) almost surely.
\end{claim}
Let \(\tilde{\vx} \in \tilde{\sX}\), \(\vx' \in \sX\), \(\vg' : \sU \to \sY\) and \(\delta > 0\).
Certainly, \(\sup_{\rvx'', \rvg''} \rw_{\tilde{\vx}}^n(\rvx'', \rvg'') \geq 0\) and so \(\sup_{\rvx'', \rvg''} \rw_{\tilde{\vx}}^n(\rvx'', \rvg'') = \sup_{\rvx'', \rvg''} \max \{\rw_{\tilde{\vx}}^n(\rvx'', \rvg''), 0\}\) almost surely.
Therefore, by Markov's inequality,
\begin{align*}
    0 &\leq \delta\, \mathbb{P}\Bigl(\rw_{\tilde{\vx}}^n(\vx', \vg') > \delta \,\Big|\, \mathcal{F}^n \Bigr) \\
    &\leq \E\Bigl[ \max\{\rw_{\tilde{\vx}}^n(\vx', \vg'), 0\} \,\Big|\, \mathcal{F}^n \Bigr] \\
    &\leq \E\Bigl[ \sup_{\rvx'', \rvg''} \max\{\rw_{\tilde{\vx}}^n(\rvx'', \rvg''), 0\} \,\Big|\, \mathcal{F}^n \Bigr] \\
    &= \E\Bigl[ \sup_{\rvx'', \rvg''} \rw_{\tilde{\vx}}^n(\rvx'', \rvg'') \,\Big|\, \mathcal{F}^n \Bigr]
    = \alpha_\text{jKG}^n(\tilde{\vx}) \leq \sup_{\tilde{\rvx}} \alpha_\text{jKG}^n(\tilde{\rvx}) \to 0 \quad \text{a.s.}
\end{align*}
as \(n \to \infty\) by \cref{thm:app-joint-kg-to-zero}.
That is, \(\mathbb{P}(\rw_{\tilde{\vx}}^n(\vx', \vg') > \delta \,|\, \mathcal{F}^n) \leq \rho_\delta^n \to 0\)  almost surely as \(n \to \infty\), where \(\rho_\delta^n = \delta^{-1} \sup_{\tilde{\rvx}} \alpha_\text{jKG}^n(\tilde{\rvx})\). This establishes the claim.

As in the claim, let \(\tilde{\vx} \in \tilde{\sX}\), \(\vx' \in \sX\), \(\vg' : \sU \to \sY\). Also, let \(\rvu'\) be an independent copy of \(\rvu\).
Write \(\tilde{\rvx}' = (\vx', \vg'(\rvu'), \rvu')\) and \(\tilde{\rvx}^{*n} = (\vx^{*n}, \vg^{*n}(\rvu'), \rvu')\).

We will first show that \(\E_{\rvu'}[k^n(\tilde{\rvx}', \tilde{\vx})] - \E_{\rvu'}[k^n(\tilde{\rvx}^{*n}, \tilde{\vx})] \to 0\) uniformly in \(\tilde{\vx}\) as \(n \to 0\), almost surely.

If for some \(n \in \mathbb{N}\), \(k^n(\tilde{\vx}, \tilde{\vx}) = 0\), then \(k^{n'}(\tilde{\vx}, \tilde{\vx}) = 0\) for all \(n' > n\) as well.
Therefore, by Cauchy-Schwarz, \(k^{n'}(\tilde{\rvx}', \tilde{\vx}) = k^{n'}(\tilde{\rvx}^{*n'}, \tilde{\vx}) = 0\) almost surely and so \(\E_{\rvu'}[k^{n'}(\tilde{\rvx}', \tilde{\vx})] = \E_{\rvu'}[k^{n'}(\tilde{\rvx}^{*n'}, \tilde{\vx})] = 0\) for all \(n' > n\).

Otherwise assume that, for all \(n\), \(k^n(\tilde{\vx}, \tilde{\vx}) > 0\) and hence \(k^n(\tilde{\vx}, \tilde{\vx}) + \sigma^2 > 0\).
The posterior mean of \(f(\tilde{\rvx}') = f(\vx', \vg'(\rvu'), \rvu')\) conditional on \(\mathcal{F}^n\), \(\rv_{\tilde{\vx}}\) and \(\rvu'\) is \citep[Equation~2.38]{rasmussen2006gpml}
\begin{equation*}
    \E[f(\tilde{\rvx}') \,|\, \mathcal{F}^n, \rv_{\tilde{\vx}}, \rvu'] = \mu^n(\tilde{\rvx}') + k^n(\tilde{\rvx}', \tilde{\vx}) \bigl( k^n(\tilde{\vx}, \tilde{\vx}) + \sigma^2 \bigr)^{-1} \bigl(\rv_{\tilde{\vx}} - \mu^n(\tilde{\vx})\bigr).
\end{equation*}
Therefore,
\begin{align*}
    \rw_{\tilde{\vx}}^n(\vx', \vg') &= \E\Bigl[ \bar{f}(\vx', \vg') - \bar{f}(\vx^{*n}, \vg^{*n}) \,\Big|\, \mathcal{F}^n, \rv_{\tilde{\vx}} \Bigr] \\
    &= \E_{\rvu'} \left[ \mu^n(\tilde{\rvx}') - \mu^n(\tilde{\rvx}^{*n}) + \frac{k^n(\tilde{\rvx}', \tilde{\vx}) - k^n(\tilde{\rvx}^{*n}, \tilde{\vx})}{k^n(\tilde{\vx}, \tilde{\vx}) + \sigma^2} \bigl( \rv_{\tilde{\vx}} - \mu^n(\tilde{\vx}) \bigr) \right] \\
    &= \bar{\mu}^n(\vx', \vg') - \bar{\mu}^n(\rvx^{*n}, \rvg^{*n}) + \frac{\E_{\rvu'}[k^n(\tilde{\rvx}', \tilde{\vx})] - \E_{\rvu'}[k^n(\tilde{\rvx}^{*n}, \tilde{\vx})]}{k^n(\tilde{\vx}, \tilde{\vx}) + \sigma^2} \bigl( \rv_{\tilde{\vx}} - \mu^n(\tilde{\vx}) \bigr).
\end{align*}
Since \(\rv_{\tilde{\vx}} \,|\, \mathcal{F}^n \sim \mathcal{N}(\mu^n(\tilde{\vx}), k^n(\tilde{\vx}, \tilde{\vx}) + \sigma^2)\), we therefore have that \(\rw_{\tilde{\vx}}^n(\vx', \vg') \,|\, \mathcal{F}^n\) is Gaussian with
\begin{align*}
    \E[\rw_{\tilde{\vx}}^n(\vx', \vg') \,|\, \mathcal{F}^n] &= \bar{\mu}^n(\vx', \vg') - \bar{\mu}^n(\rvx^{*n}, \rvg^{*n}), \\
    \mathrm{Var}[\rw_{\tilde{\vx}}^n(\vx', \vg') \,|\, \mathcal{F}^n] &= \frac{\Bigl( \E_{\rvu'}[k^n(\tilde{\rvx}', \tilde{\vx})] - \E_{\rvu'}[k^n(\tilde{\rvx}^{*n}, \tilde{\vx})] \Bigr)^2}{k^n(\tilde{\vx}, \tilde{\vx}) + \sigma^2}.
\end{align*}
Thus, writing \(\Phi(\cdot)\) for the cumulative density function of a standard normal variable, and letting \(\delta > 0\), either \(\E_{\rvu'}[k^n(\tilde{\rvx}', \tilde{\vx})] - \E_{\rvu'}[k^n(\tilde{\rvx}^{*n}, \tilde{\vx})] = 0\) or
\begin{equation*}
    \mathbb{P}(\rw_{\tilde{\vx}}^n(\vx', \vg') > \delta \,|\, \mathcal{F}^n) = 1 - \Phi\left(
        \frac{\sqrt{k^n(\tilde{\vx}, \tilde{\vx}) + \sigma^2}}{\bigl|\E_{\rvu'}[k^n(\tilde{\rvx}', \tilde{\vx})] - \E_{\rvu'}[k^n(\tilde{\rvx}^{*n}, \tilde{\vx})]\bigr|} \Bigl[ 
\delta - \bigl(\bar{\mu}^n(\vx', \vg') - \bar{\mu}^n(\rvx^{*n}, \rvg^{*n})\bigr) \Bigr]
    \right).
\end{equation*}
By \cref{thm:app-mubar-kbar-uniform-convergence}, \(\bar{\mu}^n \to \bar{\mu}^\infty\) uniformly, almost surely as \(n \to \infty\).
Further, the same argument as in \cref{thm:app-mubar-attains-bound} shows that \(\bar{\mu}^\infty\) is bounded and attains its bounds.
Therefore, there exists a negative random variable \(\rb\) (not depending on \(n\)) such that for sufficiently large \(n\), \(\bar{\mu}^n(\vx', \vg') - \bar{\mu}^n(\rvx^{*n}, \rvg^{*n}) > \rb\).
For example, take any \(\rb < \min \bar{\mu}^\infty - \max \bar{\mu}^\infty\).
Then, almost surely,
\begin{equation*}
    \mathbb{P}(\rw_{\tilde{\vx}}^n(\vx', \vg') > \delta \,|\, \mathcal{F}^n) \geq
    1 - \Phi\left(
        \frac{\sqrt{k^n(\tilde{\vx}, \tilde{\vx}) + \sigma^2}}{\bigl|\E_{\rvu'}[k^n(\tilde{\rvx}', \tilde{\vx})] - \E_{\rvu'}[k^n(\tilde{\rvx}^{*n}, \tilde{\vx})]\bigr|} (\delta - \rb)
    \right).
\end{equation*}
In the claim, we showed that the left hand side here converges to zero uniformly in \(\tilde{\vx}\) almost surely as \(n \to \infty\).
Therefore, so must the right hand side and we must have \(\E_{\rvu'}[k^n(\tilde{\rvx}', \tilde{\vx})] - \E_{\rvu'}[k^n(\tilde{\rvx}^{*n}, \tilde{\vx})] \to 0\) uniformly in \(\tilde{\vx}\), almost surely.

Thus, in all cases, we have shown \(\E_{\rvu'}[k^n(\tilde{\rvx}', \tilde{\vx})] - \E_{\rvu'}[k^n(\tilde{\rvx}^{*n}, \tilde{\vx})] \to 0\) uniformly in \(\tilde{\vx}\), almost surely.
Let \(\vg: \sU \to \sY\). Since convergence is uniform in \(\tilde{\vx}\), we may integrate over \(\rvu\) in \(\tilde{\rvx} = (\vx, \vg(\rvu), \rvu)\) to give
\begin{equation*}
    \bar{k}^n(\vx', \vg';\, \vx, \vg) - \bar{k}^n(\rvx^{*n}, \rvg^{*n};\, \vx, \vg)
    = \E_{\rvu, \rvu'}[k^n(\tilde{\rvx}', \tilde{\rvx}) - k^n(\tilde{\rvx}^{*n}, \tilde{\rvx})]
    \to 0
\end{equation*}
almost surely as \(n \to \infty\).
Further, from \cref{thm:app-mubar-kbar-uniform-convergence}, \(\bar{k}^n \to \bar{k}^\infty\) uniformly, almost surely, so
\begin{equation*}
    \bar{k}^\infty(\vx', \vg';\, \vx, \vg) - \bar{k}^\infty(\rvx^{*n}, \rvg^{*n};\, \vx, \vg) \to 0 \qquad \text{a.s.}
\end{equation*}
This holds for any \(\vx', \vg'\), including \(\vx' = \vx\) and \(\vg' = \vg\), so
\begin{multline*}
    \bar{k}^\infty(\vx', \vg';\, \vx, \vg) - \bar{k}^\infty(\vx, \vg;\, \vx, \vg) = \\
    \Bigl( \bar{k}^\infty(\vx', \vg';\, \vx, \vg) - \bar{k}^\infty(\rvx^{*n}, \rvg^{*n};\, \vx, \vg) \Bigr) - \Bigl( \bar{k}^\infty(\vx, \vg;\, \vx, \vg) - \bar{k}^\infty(\rvx^{*n}, \rvg^{*n};\, \vx, \vg) \Bigr)
    \to 0 \quad \text{a.s.}
\end{multline*}
as \(n \to \infty\).
The left hand side is independent of \(n\), so therefore must equal zero almost surely.
Thus, by symmetry of \(\bar{k}^\infty\), we conclude that for any \(\vx_A, \vx_A', \vx_B, \vx_B' \in \sX\) and \(\vg_A, \vg_A', \vg_B, \vg_B' : \sU \to \sY\),
\begin{equation*}
    \bar{k}^\infty(\vx_A, \vg_A;\, \vx_A', \vg_A') = \bar{k}^\infty(\vx_B, \vg_B;\, \vx_B', \vg_B')
\end{equation*}
almost surely.
Taking \(\rc_1 = \bar{k}^\infty(\vx_A, \vg_A;\, \vx_A', \vg_A')\) for any choice of \(\vx_A, \vg_A, \vx_A', \vg_A'\) completes the proof.
\end{proof}

\begin{proposition} \label{thm:app-fbar-minus-mubar-const-pointwise}
    There exists a random variable \(\rc_2\) such that for all \(\vx \in \sX\) and all \(\vg : \sU \to \sY\), we have
    \begin{equation*}
        \mathbb{P}\biggl(\bar{f}(\vx, \vg) = \bar{\mu}^\infty(\vx, \vg) + \rc_2 \biggr) = 1.
    \end{equation*}
\end{proposition}
\begin{proof}
    Let \(\vx, \vx' \in \sX\) and \(\vg, \vg': \sU \to \sY\).
    Then
    \begin{multline*}
        \mathrm{Var}\Bigl[\bigl(\bar{f}(\vx, \vg) - \bar{\mu}^\infty(\vx, \vg)\bigr) - \bigl(\bar{f}(\vx', \vg') - \bar{\mu}^\infty(\vx', \vg')\bigr) \,\Big|\, \mathcal{F}^\infty \Bigr] \\
        = \underbrace{k^\infty(\vx, \vg;\, \vx, \vg) - k^\infty(\vx, \vg;\, \vx', \vg')}_{0 \text{ a.s.}}
         + \underbrace{k^\infty(\vx', \vg';\, \vx', \vg') - k^\infty(\vx', \vg';\, \vx, \vg)}_{0 \text{ a.s.}}
        = 0
    \end{multline*}
    almost surely by \cref{thm:app-kbar-constant}.
    Also, \(\E[(\bar{f}(\vx, \vg) - \bar{\mu}^\infty(\vx, \vg)) - (\bar{f}(\vx', \vg') - \bar{\mu}^\infty(\vx', \vg')) \,|\, \mathcal{F}^\infty] = 0\) almost surely.
    Write \(\rz = (\bar{f}(\vx, \vg) - \bar{\mu}^\infty(\vx, \vg)) - (\bar{f}(\vx', \vg') - \bar{\mu}^\infty(\vx', \vg'))\).
    The law of total variance gives
    \begin{equation*}
        \mathrm{Var}[\rz] = \E\bigl[\mathrm{Var}[\rz | \mathcal{F}^\infty] \bigr] + \mathrm{Var}\bigl[\E[ \rz \,|\, \mathcal{F}^\infty]\bigr] = 0.
    \end{equation*}
    Therefore, \(\rz\) is a random variable with zero mean and zero variance, and so is almost surely zero.
    That is, \(\bar{f}(\vx, \vg) - \bar{\mu}^\infty(\vx, \vg) = \bar{f}(\vx', \vg') - \bar{\mu}^\infty(\vx', \vg')\) almost surely.
    Taking \(\rc_2 = \bar{f}(\vx, \vg) - \bar{\mu}^\infty(\vx, \vg)\) completes the proof.
\end{proof}

Having established that \(\bar{f}\) and \(\bar{\mu}^\infty\) are pointwise almost surely equal up to a constant, we now extend this to show that this holds for all points \emph{simultaneously}. That is, that their sample paths are almost surely equal up to a constant.
\begin{proposition} \label{thm:app-fbar-minus-mubar-const-simultaneous}
    There exists a random variable \(c_2\) such that
    \begin{equation*}
        \mathbb{P}\biggl( \forall \vx \in \sX\; \forall \vg: \sU \to \sY \text{ measurable} \qquad \bar{f}(\vx, \vg) = \bar{\mu}^\infty(\vx, \vg) + \rc_2 \biggr) = 1.
    \end{equation*}
\end{proposition}
\begin{proof}
Our strategy for this proof will be to first extend the \cref{thm:app-fbar-minus-mubar-const-pointwise} to hold for countably many pairs \((\vx, \vg)\) simultaneously.
For any pair \((\vx, \vg)\), we will construct a sequence \((\vx^i, \vg^i) \to (\vx, \vg)\) as \(i \to \infty\) but only taking values in a countable set not depending on \((\vx, \vg)\).
We will then use continuity of \(f\) and \(\mu^\infty\) to conclude that the result must hold for all pairs \((\vx, \vg)\) simultaneously, which is the theorem statement.

The intersection of countably many events with probability 1 has probability 1 by countable subadditivity of measures.
Applying this to \cref{thm:app-fbar-minus-mubar-const-pointwise} with the Cartesian product of \(\sX \cap \mathbb{Q}^{d_x}\) and the set of simple functions \(\vg : \sU \to \sY \cap \mathbb{Q}^{d_y}\) (which is countable), we obtain
\begin{equation} \label{eq:app-fbar-minus-mubar-const-simultaneous-countable}
    \mathbb{P}\biggl(
        \forall \vx \in \sX \cap \mathbb{Q}^{d_x}\;
        \forall \vg : \sU \to \sY \cap \mathbb{Q}^{d_y} \text{ simple}
        \qquad
        \bar{f}(\vx, \vg) = \bar{\mu}^\infty(\vx, \vg) + \rc_2
    \biggr) = 1.
\end{equation}
Write \(A\) for this event.

Let \(\vx \in \sX\) and let \(\vg : \sU \to \sY\) be measurable.

Since \(\sX \cap \mathbb{Q}^{d_x}\) is dense in \(\sX\), we can take a sequence \(\vx^1, \vx^2, \dots \in \sX \cap \mathbb{Q}^{d_x}\) with \(\vx^i \to \vx\) as \(i \to \infty\)
Furthermore, since \(\vg\) is measurable, we can write \(\vg(\vu) = \lim_{i \to \infty} \vg^i(\vu)\) as a pointwise limit of simple functions, \((\vg^i)_{i=1}^ \infty\), by applying \citep[Theorem 1.96]{klenke2020probability} to the positive and negative parts of each component of \(\vg\).
Since \(\sY \cap \mathbb{Q}^{d_y}\) is dense in \(\sY\), we can assume the \(\vg^1, \vg^2, \dots\) take values in \(\sY \cap \mathbb{Q}^{d_y}\).
The sample paths of \(f - \mu^\infty\) are continuous, and therefore for any \(\vu \in \sU\),
\begin{equation*}
    f(\vx^i, \vg^i(\vu), \vu) - \mu^\infty(\vx^i, \vg^i(\vu), \vu) \to f(\vx, \vg(\vu), \vu) - \mu^\infty(\vx, \vg(\vu), \vu)
\end{equation*}
almost surely as \(i \to \infty\).
The expectation \(\E_\rvu\) is formally an expectation conditional on the sigma algebra \(\sigma(f, \mathcal{F}^\infty)\) (i.e. on everything except \(\rvu\)).
Therefore, the Dominated Convergence Theorem for conditional expectation \citep[Theorem 8.14 (viii)]{klenke2020probability} with \(\max |f - \mu^\infty|\) as a (constant) dominating function, the expectation over \(\rvu\) also converges almost surely
\begin{equation*}
    \underbrace{\E_\rvu\Bigl[ f(\vx^i, \vg^i(\vu), \vu) - \mu^\infty(\vx^i, \vg^i(\vu), \vu) \Bigr]}_{\bar{f}(\vx^i, \vg^i) - \bar{\mu}^\infty(\vx, \vg^i)} \to \underbrace{\E_\rvu\Bigl[ f(\vx, \vg(\vu), \vu) - \mu^\infty(\vx, \vg(\vu), \vu) \Bigr]}_{\bar{f}(\vx, \vg) - \bar{\mu}^\infty(\vx, \vg)}.
\end{equation*}
On the event \(A\) considered in \cref{eq:app-fbar-minus-mubar-const-simultaneous-countable}, each term in the sequence on the left hand side is equal to \(\rc_2\), and therefore the limit on the right must equal \(\rc_2\) as well.
Whence, on \(A\), \(\bar{f}(\vx, \vg) - \bar{\mu}^\infty(\vx, \vg) = \rc_2\).
Our choice of \(\vx\) and \(\vg\) was arbitrary, and so
\begin{equation*}
    \mathbb{P}\biggl(
        \forall \vx \in \sX\;
        \forall \vg : \sU \to \sY \text{ measurable}
        \qquad
        \bar{f}(\vx, \vg) = \bar{\mu}^\infty(\vx, \vg) + \rc_2
    \biggr) = 1.
\end{equation*}
\end{proof}

We can now prove \cref{thm:jointkg-consistency} from the main text, which we restate here for convenience.
\begin{theorem}[Consistency of Joint KG] \label{thm:app-jointkg-consistent}
    When using joint knowledge gradient as an acquisition function, the objective values associated with the recommendations converge to the optimal value both almost surely and in mean.
    That is,
    \begin{equation*}
        \bar{f}(\rvx^{*n}, \rvg^{*n}) \to \sup_{\substack{\rvx \in \sX \\ \rvg : \sU \to \sY}} \bar{f}(\rvx, \rvg) \qquad\text{as}\qquad n \to \infty
    \end{equation*}
    almost surely and in mean.
\end{theorem}
\begin{proof}
    By \cref{thm:app-fbar-attains-bound}, \(\bar{f}\) is bounded and attains its bounds, so let \(\rvx^*, \rvg^* \in \argmax_{\rvx, \rvg} \bar{f}(\rvx, \rvg)\) be random elements maximizing \(\bar{f}\).
    For any \(n\), \(\bar{f}(\rvx^{*n}, \rvg^{*n}) \leq \sup_{\rvx, \rvg} \bar{f}(\rvx, \rvg)\) almost surely.
    Therefore, it suffices to prove that \(\limsup_{n \to \infty} \bar{f}(\rvx^*, \rvg^*) - \bar{f}(\rvx^{*n}, \rvg^{*n}) = 0\).
    
    Splitting this limit, we have
    \begin{align*}
        \limsup_{n \to \infty} &\; \bar{f}(\rvx^*, \rvg^*) - \bar{f}(\rvx^{*n}, \rvg^{*n}) \\
        \leq & \limsup_{n \to \infty}
              \Bigl( \bar{f}(\rvx^*, \rvg^*) - \bar{\mu}^\infty(\rvx^*, \rvg^*) \Bigr)
            - \Bigl( \bar{f}(\rvx^{*n}, \rvg^{*n}) - \bar{\mu}^\infty(\rvx^{*n}, \rvg^{*n}) \Bigr) \\
        + & \limsup_{n \to \infty}
              \Bigl( \bar{\mu}^\infty(\rvx^*, \rvg^*) - \bar{\mu}^n(\rvx^*, \rvg^*) \Bigr)
            - \Bigl( \bar{\mu}^\infty(\rvx^{*n}, \rvg^{*n}) - \bar{\mu}^n(\rvx^{*n}, \rvg^{*n}) \Bigr) \\
        + & \limsup_{n \to \infty}
              \bar{\mu}^n(\rvx^*, \rvg^*) - \bar{\mu}^n(\rvx^{*n}, \rvg^{*n}).
    \end{align*}
    By \cref{thm:app-fbar-attains-bound,thm:app-mubar-attains-bound}, \(\rvg^*\) and the \(\rvg^{*n}\) have measurable sample paths, and so by \cref{thm:app-fbar-minus-mubar-const-simultaneous}, for each \(n\),
    \begin{equation*}
        \Bigl( \bar{f}(\rvx^*, \rvg^*) - \bar{\mu}^\infty(\rvx^*, \rvg^*) \Bigr)
            - \Bigl( \bar{f}(\rvx^{*n}, \rvg^{*n}) - \bar{\mu}^\infty(\rvx^{*n}, \rvg^{*n}) \Bigr)
        = \rc_2 - \rc_2 = 0
    \end{equation*}
    almost surely.
    Then by countable subadditivity of measures, this holds for all \(n\) simultaneously, and so the \(\limsup\) in the first line is almost surely zero.
    
    Since \(\bar{\mu}^n \to \bar{\mu}^\infty\) uniformly, almost surely, as \(n \to \infty\) (\cref{thm:app-mubar-kbar-uniform-convergence}), the second line is also almost surely zero.
    Finally, the third line line is at most zero by optimality of \(\rvx^{*n}, \rvg^{*n}\).
    Therefore,
    \begin{equation*}
        \limsup_{n \to \infty} \bar{f}(\rvx^*, \rvg^*) - \bar{f}(\rvx^{*n}, \rvg^{*n}) \leq 0
    \end{equation*}
    and since \(\bar{f}(\rvx^*, \rvg^*) - \bar{f}(\rvx^{*n}, \rvg^{*n}) \geq 0\) almost surely as well, we conclude that \(\bar{f}(\rvx^{*n}, \rvg^{*n}) \to \bar{f}(\rvx^*, \rvg^*)\) almost surely as \(n \to \infty\).

    To establish convergence in mean, we appeal to the Dominated Convergence Theorem with dominating variable \(\sup_{\rvx, \rvg}|\bar{f}(\rvx, \rvg)|\). Indeed, this has finite expectation by \cref{thm:app-supfbar-finite-expectation}.
\end{proof}

\section{Caching of the covariances} \label{app:caching}
Computation of the discrete knowledge gradient naively requires recomputing many covariances.
It is common in inference with Gaussian processes to cache the covariances between the training data to improve performance.
Performance can be further improved by caching the cross covariances between the training data and the discretization used to compute the discrete knowledge gradient, which is not done automatically by libraries such as GPyTorch \citep{gardner2018gpytorch}.

Recall the formula for the discrete knowledge gradient approximation applicable to standard (not two-stage) optimization problem \(h : \sX \to \R\) from \cref{eq:bg-kg-discrete} was given in the background section as
\begin{equation}
    \hat{\alpha}_\text{KG}^n(\vx)
    = \frac{1}{N_v} \sum_{i=1}^{N_v} \max_{\vx' \in \sXd} \mu^{n+1}(\vx' ; \vx, \tilde{v}^{n+1,i}) - \max_{\vx' \in \sXd} \mu^n(\vx').
\end{equation}
Here, \(\tilde{v}^{n+1, 1}, \dots, \tilde{v}^{n+1, N_v}\) are independent samples of the next observation \(\rv^{n+1} = f(\vx) + \varepsilon^{n+1}\), conditional on the the \(n\) observations made so far.
For any \(\vx' \in \sXd\) in the discretization, the posterior mean is computed using \citep[Equation~2.38]{rasmussen2006gpml} as
\begin{equation}
    \mu^{n+1}(\vx';\, \vx, \tilde{v}^{n+1,i}) = \mu^n(\vx') + k^n(\vx', \vx) \bigl(k^n(\vx, \vx) + \sigma^2 \bigr)^{-1} (\tilde{v}^{n+1, i} - \mu^n(\vx)).
\end{equation}
As the acquisition function is optimized, \(\hat{\alpha}_\text{KG}^n(\vx)\) and its gradient are evaluated at many different \(\vx \in \sX\), and there are three potentially costly computations involved: \(\mu^n(\vx')\), \(k^n(\vx', \vx)\) and \(k^n(\vx, \vx)\).

From Equation~2.24 in \citep{rasmussen2006gpml}
\begin{equation}
    k^n(\vx, \vx) = k(\vx, \vx) - \vk_\mX(\vx)^T \bigl( \mK_{\mX \mX} + \sigma^2 \mI \bigr)^{-1} \vk_\mX(\vx),
\end{equation}
where \(\vk_\mX(\vx)\) is the column vector with entries \([\vk_\mX(\vx)]_i = k(\vx^i, \vx)\), and where \(\mK_{\mX \mX}\) is the matrix with elements \([\mK_{\mX \mX}]_{ij} = k(\vx^i, \vx^j)\).
Recall \(\vx^1, \dots, \vx^n \in \sX\) are the observation locations.
The inversion \((\mK_{\mX \mX} + \sigma^2 \mI)^{-1}\) is potentially costly, and does not depend on the test input \(\vx\).
It is cached by default in GPyTorch.

The other terms are computed as \citep[equations 2.24 and 2.38]{rasmussen2006gpml}
\begin{align}
    \mu^n(\vx') &= \mu(\vx') + \vk_\mX(\vx')^T \bigl( \mK_{\mX \mX} + \sigma^2 \mI \bigr)^{-1} (\vv - \vmu_\mX), \\
    k^n(\vx', \vx) &= k(\vx', \vx) - \vk_\mX(\vx')^T \bigl( \mK_{\mX \mX} + \sigma^2 \mI \bigr)^{-1} \vk_\mX(\vx).
\end{align}
Here, \(\vmu_\mX = (\mu(\vx^1), \dots, \mu(\vx^n))\) is the column vector of prior means and \(\vv = (\evv^1, \dots, \evv^n)\) is the vector of noisy observations of \(f\).
As already mentioned, the inversion \((\mK_{\mX \mX} + \sigma^2 \mI)^{-1}\) does not depend on \(\vx\) and so should be cached.
In fact, for the computation of \(\mu^n(\vx')\), the whole product \((\mK_{\mX \mX} + \sigma^2 \mI)^{-1} (\vv - \vmu_\mX)\) is cached.

Further, the cross-covariances \(\vk_\mX(\vx')\) also do not depend on \(\vx\) and are computed for many \(\vx'\).
Repeating this computation every time the query location \(\vx\) is updated is costly and wasteful.
Writing \(\mX'\) for the matrix of a discretization points \(\vx' \in \sXd\), we may cache the matrix \(\mK_{\mX' \mX} = [k(\vx'^i, \vx^j)]_{i=1,\dots,N_x}^{j=1, \dots, n}\) during the optimization of \(\hat{\alpha}_\text{KG}^n(\vx)\).
Due to the design of the API, this is not done by default in GPyTorch.

In two-stage problems modeled by a GP \(f: \tilde{\sX} \to \R\) with \(\tilde{\sX} = \sX \times \sY \times \sU\), the quantities to be cached are simply \(\bigl(\mK_{\tilde{\mX} \tilde{\mX}} + \sigma^2 \mI \bigr)^{-1}\), \(\bigl(\mK_{\tilde{\mX} \tilde{\mX}} + \sigma^2 \mI \bigr)^{-1} (\vv - \vmu_{\tilde{\mX}})\) and \(\mK_{\tilde{\mX}' \tilde{\mX}}\).
In the experiments in this paper, these quantities were cached in the implementation of the joint and alternating knowledge gradient acquisition functions.

\section{Optimization of the acquisition functions using multi-start L-BFGS-B} \label{app:optimization-of-acqf}
The acquisition functions for joint, alternating and two-step KG in \cref{eq:acqf-joint-kg-approx,eq:acqf-alternating-kg-adjustable-approx,eq:acqf-alternating-kg-fixed-approx,eq:acqf-twostep-kg-approx} are optimized using multi-start L-BFGS-B using the heuristics implemented in the function \texttt{optimize\_acqf} in BoTorch \citep{balandat2020botorch}. The function is called with \texttt{nonnegative=True} since knowledge gradient is non-negative (\cref{thm:app-jointkg-nonneg}).

In the following description of the BoTorch implementation, we use \(\tilde{\sX} = \sX \times \sY \times \sU\) for the search space, which is appropriate for the joint and alternating KG.
For two-step KG, the reader should mentally replace this with \(\sY \times \sU\) or \(\sX \times \sU\) for the first and second steps respectively.
We simply write \(\alpha(\cdot)\) for the acquisition function, dropping the superscript \(n\) so that the number of samples collected so far is implicit.

The start points for L-BFGS-B are generated by first evaluating the acquisition function at 256 locations in \(\tilde{\sX}_\text{raw} \subset \tilde{\sX}\) determined by a scrambled Sobol' sequence.
These are referred to as `raw samples'.
The raw samples are subsampled to those at least 0.01\% of the maximum, \(\alpha_\text{max} = \max_{\tilde{\vx} \in \tilde{\sX}_\text{raw}} \alpha(\tilde{\vx})\), to remove values close to zero.
This gives, \(\tilde{\sX}_{\text{raw}, \gg 0} = \{\tilde{\vx} \in \tilde{\sX}_\text{raw} \,:\, \alpha(\tilde{\vx}) \ge 10^{-4} \alpha_\text{max}\}\).
Finally, 10 restart locations are chosen from \(\tilde{\sX}_{\text{raw}, \gg 0}\) using Boltzmann sampling.
That is, a subsample is taken independently, without replacement using probabilities
\begin{equation*}
    \forall \tilde{\vx} \in \tilde{\sX}_{\text{raw}, \gg 0} \quad p(\tilde{\vx}) \propto e^{\alpha(\tilde{\vx}) / \alpha_\text{max}}.
\end{equation*}
If the largest raw sample is not in the subsample then the last restart is replaced by this value.
If \(\tilde{\sX}_{\text{raw}, \gg 0}\) does not contain at least 10 values, then the threshold of 0.01\% is repeatedly reduced by a factor of 10 until it does.

\Cref{tab:params-acqf-optim} summarizes the important parameters.

There are some differences to this approach for the supply chain problem, which are detailed separately in \cref{app:further-exp-details-supply-chain-constrained-optim}.

\begin{table}[ht]
    \centering
    \begin{tabular}{cc}
    \toprule
        Parameter & Value \\
    \midrule
        Number of restarts & 10 \\
        Number of raw samples & 256 \\
        Max number of L-BFGS-B iterations & 200 \\
    \bottomrule
    \end{tabular}
    \caption{Parameters for optimization of the knowledge gradient acquisition functions.}
    \label{tab:params-acqf-optim}
\end{table}

\section{Optimization for the final recommendation} \label{app:optimization-of-recommendation}
Knowledge gradient acquisition functions are one-step Bayes-optimal when the final recommendation is given by the maximum of the posterior mean.
For the joint and alternating knowledge gradient, this is given by \cref{eq:recommendation-bilevel}.
The optimization is done in two stages.
First, the optimal fixed design, \(\hat{\vx}^{*n}\), is chosen to maximize \cref{eq:recommendation-bilevel-x-approx} using multi-start L-BFGS-B.
Then, the optimal control policy, \(\hat{\vg}^{*n}\), is optimized for each \(\vu\) of interest using single-start L-BFGS-B.
This section explains in detail how these optimizations are performed.

Note that this approach is not used for the supply chain problem, where instead recommendations are generated using an exhaustive search following \citep{xie2021globallocal}.

\subsection{One-shot optimization of fixed design via multi-start L-BFGS-B} \label{app:optimization-of-recommendation-x}
The optimal fixed design is approximated by rewriting \cref{eq:recommendation-bilevel-x-approx} as
\begin{equation}
    \hat{\vx}^{*n}
    \in \argmax_{\vx \in \sX} \sum_{\vu \in \sU_{\mathrm{MC}, \text{rec}}} \max_{\vy \in \sY} \mu^{n}(\vx, \vy, \vu)
    = \argmax_{\substack{\vx \in \sX \\ \vy^1, \dots, \vy^{N_{u,\text{rec}}} \in \sY}} \sum_{i=1}^{N_{u,\text{rec}}} \mu^n(\vx, \vy^i, \vu^i),
\end{equation}
where we have enumerated \(\sU_{\mathrm{MC}, \text{rec}} = \{\vu^1, \dots, \vu^{N_{u,\text{rec}}}\}\).

This is optimized as one large \(d_x d_y^{N_{u,\text{rec}}}\)-dimensional problem, which we term a `one-shot' optimization after the one-shot knowledge gradient \cite{balandat2020botorch}. Here \(d_x\) and \(d_y\) are the dimensions of \(\sX \subset \R^{d_x}\) and \(\sY \subset \R^{d_y}\), respectively.
It is optimized using multi-start L-BFGS-B with 10 restarts, as implemented in \texttt{optimize\_acqf} in BoTorch.
A \texttt{batch\_limit} of 4 is used, meaning up to 4 restarts are grouped in a single \(4 d_x d_y^{N_{u,\text{rec}}}\) optimization.

Care must be taken over the choice of the starting point for each of the 10 restarts.
Simply scattering a large number of points over \(\sX \times \sY^{N_{u,\text{rec}}}\) using a Sobol' sequence will lead the algorithm to converge to a local rather than global minimum.
Indeed, it is highly unlikely that the starts will be well spread over \(\sY\) for all of the \(N_{u,\text{rec}}\) realizations of \(\rvu\).
Instead, we take Sobol' samples \(\sX_\text{raw} \subset \sX\) and \(\sY_\text{raw} \subset \sY\) of sizes \(N_{\text{raw},x}\) and \(N_{\text{raw},y}\), respectively, and combine them in a Cartesian fashion.
Finally, Boltzmann sampling is used to select promising initial values for the 10 restarts.
Explicitly, the steps are
\begin{enumerate}
    \item Take Sobol' samples \(\sX_\text{raw} \subset \sX\) and \(\sY_\text{raw} \subset \sY\) of sizes \(N_{\text{raw},x}\) and \(N_{\text{raw},y}\), respectively;
    \item For each pair \((\vx, \vu) \in \sX \times \sU_{\mathrm{MC}, \text{rec}}\), find the best \(\vy \in \sY_\text{raw}\),
    \begin{equation*}
        \vy(\vx, \vu) \in \argmax_{\vy \in \sY_\text{raw}} \mu^n(\vx, \vy(\vx, \vu), \vu);
    \end{equation*}
    \item Assign a score to each \(\vx \in \sX_\text{raw}\) as
    \begin{equation*}
        s(\vx) = \frac{1}{N_u} \sum_{\vu \in \sU_{\mathrm{MC}, \text{rec}}} \mu^n(\vx, \vy(\vx, \vu), \vu);
    \end{equation*}
    \item Use Boltzmann sampling to randomly choose 10 initial locations from \(\sX_\text{raw}\) from the restarts without replacement, with independent probabilities
    \begin{equation*}
        \forall \vx \in \sX_\text{raw} \quad p(\vx) \propto e^{z(x)}
        \qquad\text{where}\qquad
        \forall \vx \in \sX_\text{raw} \quad z(\vx) = \frac{s(\vx) - \overline{s}}{\sigma_s},
    \end{equation*}
    where \(\overline{s}\) and \(\sigma_s\) are the mean and sample standard deviation of the \(s(\vx)\).
    If the \(\vx \in \sX_\text{raw}\) with the highest score was not selected, then the last restart is replaced by this value.
\end{enumerate}

The number of raw samples used depends on the dimension of the problem.
In the test problems with \(d_x = 1, 2\), we used \(N_{\text{raw},x} = 32\), while when \(d_x = 4\) we used \(N_{\text{raw},x} = 128\).
Similarly, we used \(N_{\text{raw},y} = 32\) when \(d_y=1,2\) and \(N_{\text{raw},y} = 128\) when \(d_y=4\).
We use an \texttt{init\_batch\_limit} of \(2^{17}\), meaning we will evaluate \(\mu^n(\vx, \vy, \vu)\) for all \((\vx, \vy, \vu) \in \sX_\text{raw} \times \sY_\text{raw} \times \sU_{\mathrm{MC}, \text{rec}}\) in batches of \(2^{17}\) points, which is significantly faster than using a for loop in python.
In our experiments with the largest dimensions, we have \(|\sX_\text{raw} \times \sY_\text{raw} \times \sU_{\mathrm{MC}, \text{raw}}| = 32 \times 128 \times 128 = 2^{19}\), so this equates to four batches.
Limits on memory prevent us from increasing the \texttt{init\_batch\_limit} arbitrarily high.
\Cref{tab:params-recx-optim} summarizes the key parameters for the optimization of \(\hat{\vx}^{*n}\).

\begin{table}[ht]
    \centering
    \begin{tabular}{cc}
    \toprule
        Parameter & Value \\
    \midrule
        Number of restarts & 10 \\
        Number of raw samples for \(\sX\) & 32 or 128 (dimension dependent) \\
        Number of raw samples for \(\sY\) & 32 or 128 (dimension dependent) \\
        Batch limit & 4 \\
        Batch limit for raw samples & \(2^{17}\) \\
        Max number of L-BFGS-B iterations & 200 \\
    \bottomrule
    \end{tabular}
    \caption{Parameters for optimization of the recommended fixed design, \(\hat{\vx}^{*n}\).}
    \label{tab:params-recx-optim}
\end{table}

\subsection{Single-start L-BFGS-B for optimization of control policy} \label{app:optimization-of-recommendation-y}
Once the recommended fixed design \(\hat{\vx}^{*n}\) has been found, we must optimize to find the recommended control policy \(\hat{\vg}^{*n}(\vu) \in \sY\) for the \(\vu \in \sU\) of interest.
For example, when computing the simple regret from \cref{eq:simple-regret-approx}, we wish to compute \(\hat{\vg}^{*n}(\vu^j)\) for a sample \(\vu^1, \dots, \vu^{N_u}\) of the environmental variable \(\rvu\), independent of the one used to optimize \(\hat{\vx}^{*n}\).
Empirically, we found that a single start L-BFGS-B is sufficient for this problem.
Again, we use \texttt{optimize\_acqf} from BoTorch, with \(N_{\text{raw},y} = 32\) or \(128\) raw samples depending on the problem dimension, as in \cref{app:optimization-of-recommendation-x}.
The best of the raw samples is used as the initial point for the L-BFGS-B.
Optimizations for the different \(\vu\) are run separately (corresponding to a \texttt{batch\_limit} of 1).

\begin{table}[ht]
    \centering
    \begin{tabular}{cc}
    \toprule
        Parameter & Value \\
    \midrule
        Number of restarts & 1 \\
        Number of raw samples & 32 or 128 (dimension dependent) \\
        Batch limit & 1 \\
        Batch limit for raw samples & \(> N_{\text{raw},y}\) \\
        Max number of L-BFGS-B iterations & 200 \\
    \bottomrule
    \end{tabular}
    \caption{Parameters for optimization of the recommended control policy, \(\hat{\vg}^{*n}(\vu)\).}
    \label{tab:params-recy-optim}
\end{table}

\section{Further Experimental Details} \label{app:further-exp-details}
In this appendix we give details and parameters used in the experiments necessary for reproducing the results.

\subsection{Experimental parameters} \label{app:further-exp-details-params}
\Cref{tab:experiment-parameters} summarizes the parameters used for the optimization of the acquisition function and estimation of the the (average) simple regret in the empirical studies. These are in addition to the parameters which have already been specified in \cref{tab:params-acqf-optim,tab:params-recx-optim,tab:params-recy-optim}.

\begin{table}[ht]
    \centering
    \begin{tabular}{p{8cm}cc}
        \toprule
        Parameter & Symbol & Value \\
        \midrule
         Number of fantasy samples & \(N_v\) & 64 \\
         Fixed design discretization size & \(N_x\) & 20 \\
         Adjustable variable discretization size & \(N_y\) & 20 \\
         Number of qMC points for environmental variable when optimizing acquisition functions & \(N_u\) & 64 \\
         Number of qMC points for environmental variable when optimizing \(\vx^{*n}\) and estimating simple regret & \(N_{u,\text{rec}}\) & 128 \\
        \bottomrule
    \end{tabular}
    \caption{Parameters used for the optimization of the acquisition function and estimation of the simple regret in the empirical studies.}
    \label{tab:experiment-parameters}
\end{table}

\Cref{tab:experiment-dimensions-lengthscales} summarizes the dimensions and length scales used for the different GP sampled synthetic problems. All GP generated test problems used a Mat\'ern-5/2 kernel with an output scale of 10, defined on the unit hypercube, \([0, 1]^d\). They were generated with a random Fourier feature approximation with 1024 features, as described in \citep{rahimi2007rff} and implemented in BoTorch \citep{balandat2020botorch}.

For the problems with a combined dimension of \(d = d_x + d_y + d_u = 6\) an initial design of \(n_0 = 50\) points was used, with a total budget of \(n_\mathrm{tot} = 400\) evaluations.
The two-step algorithms used \(n_1 = n_2 = 50\) points for the initial design of the first and second phases and had a total evaluation budget of \(n_{\mathrm{tot},1} = n_{\mathrm{tot},2} = 200\) for each phase.
These problems were used for the examining the effects of dimension and observation noise, and used a length scale of \(0.4\) in all dimensions.

The problems used to examine the effect of length scale had a combined dimension of \(d = d_x + d_y + d_u = 3\).
For these, an initial design of \(n_0 = 10\) points was used, with a total budget of \(n_\mathrm{tot} = 100\) evaluations.
The two-step algorithms used \(n_1 = n_2 = 10\) points for the initial design and had a total evaluation budget of \(n_{\mathrm{tot},1} = n_{\mathrm{tot},2} = 50\) for each step.

For all GP generated test problems, the fixed design used by the two-step algorithm in the first step is \(\vx_\text{step-1} = (0.5, \dots, 0.5) \in \R^{d_x}\).

For the optical table problem, the fixed design is the spring constant \(k\), the adjustable variable is the damping coefficient \(c\), and the uncertain environmental parameter is the angular frequency of the floor vibration \(\omega\).
The search space for the optimization is \(12\,N\,mm^{-1} \leq k \leq 50\,N\,mm^{-1}\), \(1\,Ns\,mm^{-1} \leq c \leq 10\,Ns\,mm^{-1}\) and \(1\,Hz \leq \omega/2\pi \leq 100\,Hz\).
The fixed design used in the first step of the two-step algorithms is \(k_\text{step-1} = 31\,N\,mm^{-1}\).

The optical table test problem used an initial design of \(n_0 = 6\) points and a total budget of \(n_\mathrm{tot} = 100\) evaluations. The two-step algorithms used \(n_1 = n_2 = 6\) points in the initial design and had a total evaluation budget of  \(n_{\mathrm{tot},1} = n_{\mathrm{tot},2} = 50\) for each step.

\begin{table}[ht]
    \centering
    \begin{tabular}{ccccccc}
        \toprule
        \multirow{2}*{Experiment} & \multicolumn{3}{c}{Dimensions} & \multicolumn{3}{c}{Length scales} \\
        \cmidrule(r){2-4} \cmidrule(lr){5-7}
        & \(d_x\) & \(d_y\) & \(d_u\) & \(\ell_x\) & \(\ell_y\) & \(\ell_u\) \\
        \midrule
        Dimensions        & 2 & 2 & 2 & 0.4 & 0.4 & 0.4 \\
                          & 4 & 1 & 1 & 0.4 & 0.4 & 0.4 \\
                          & 1 & 4 & 1 & 0.4 & 0.4 & 0.4 \\
                          & 1 & 1 & 4 & 0.4 & 0.4 & 0.4 \\
        Length scales     & 1 & 1 & 1 & 0.1 & 2   & 2   \\
                          & 1 & 1 & 1 & 2   & 0.1 & 2   \\
                          & 1 & 1 & 1 & 2   & 2   & 0.1 \\
        Observation noise & 2 & 2 & 2 & 0.4 & 0.4 & 0.4 \\
        \bottomrule
    \end{tabular}
    \caption{Dimensions and length scales used in the GP sampled test problems in the empirical studies. All problems were generated with a Mat\'ern-5/2 kernel with an output scale of 10.}
    \label{tab:experiment-dimensions-lengthscales}
\end{table}

\subsection{Analytical solution for the optical table test problem} \label{app:further-exp-details-optical-table}
The optical table system described in \cref{sec:experiments-optical-table} obeys the following differential equation
\begin{equation} \label{eq:app-optical-table-ode}
    m \ddot{y} + c \dot{y} + 4ky = c\dot{y}_f + 4ky_f + 4k\ell - gm
\end{equation}
where \cref{tab:optical-table-params} summarizes the different constants and variables used here.
We assume that the floor is undergoing simple harmonic motion
\begin{equation} \label{eq:app-optical-table-floorshm}
    y_f = A\cos(\omega t + \phi_f).
\end{equation}
Then the solution to the differential equation is the superposition of a steady state vibration, a constant offset and a transient response which depends on the initial conditions and decays over time,
\begin{equation} \label{eq:app-optical-table-solution}
    y = \underbrace{B\cos(\omega t + \phi)}_\text{steady state vibration} + \underbrace{\ell - \frac{gm}{2k}}_\text{constant offset} + \underbrace{y_c}_\text{transient}.
\end{equation}
The transient part is given by
\begin{equation} \label{eq:app-optical-table-solution-transient}
    y_c = \begin{cases}
        C_1 e^{\alpha_1 t} + C_2 e^{\alpha_2 t} & \text{if } c^2 > 8km \text{ (overdamped)}, \\
        (C_1 + C_2 t) e^{-ct / 2m} & \text{if } c^2 = 8km \text{ (critically damped)}, \\
        C e^{-ct/2m} \cos(\beta t + \phi_c) & \text{if } c^2 < 8km \text{ (underdamped)}.
    \end{cases}
\end{equation}
Here \(\alpha_1, \alpha_2 = \frac{-c \pm \sqrt{c^2 - 8km}}{2m}\) are the roots of the characteristic equation and \(\beta = \frac{\sqrt{8km - c^2}}{2m}\) is the imaginary part of \(\alpha_1\) when it has one. The constants \(C_1\), \(C_2\), \(C\) and \(\phi_c\) are arbitrary.

Thus, the amplitude ratio for the steady state solution is
\begin{equation} \label{eq:app-optical-table-ampratio}
    \frac{B}{A} = \sqrt{\frac{16k^2 + c^2 \omega^2}{(4k - m \omega^2)^2 + c^2 \omega^2}}.
\end{equation}

\begin{table}[ht]
    \centering
    \begin{tabular}{cp{10cm}}
        \toprule
        Parameter & Description \\
        \midrule
        \(k\) & The spring constant associated with each of the four springs \\
        \(c\) & The coefficient of the damper \\
        \(y\) & The vertical displacement of the table \\
        \(y_f\) & The vertical displacement of the floor \\
        \(\omega\) & The frequency of the simple harmonic motion of the floor \\
        \(\phi_f\) & The phase offset of the simple harmonic motion of the floor \\
        \(\phi\) & The phase offset of the steady state vibration of the table \\
        \(A\) & The amplitude of the simple harmonic motion of the floor \\
        \(B\) & The amplitude of the steady state vibration of the table \\
        \(m = m_1 + m_2\) & The combined mass of the table, \(m_1\), and the equipment, \(m_2\) \\
        \(\ell\) & The natural height of the table above the floor (when the springs are at their natural length, in the absence of gravity) \\
        \(g\) & The local coefficient of gravity \\
        \bottomrule
    \end{tabular}
    \caption{Constants and variables used in the model of the optical table.}
    \label{tab:optical-table-params}
\end{table}

\subsection{Supply chain management problem} \label{app:further-exp-details-supply-chain}

The supply chain problem has several features not considered in the other test problems. Namely, it has normally distributed environmental variables, it has constraints coupling the fixed and adjustable variables, and as defined in \citep{xie2021globallocal} it has a finite search space (\(\vx\) and \(\vy\) can only take integer values and \(\vx\) must be a multiple of \(20\)). These lead to a some natural modifications to the implementation of the algorithm.

\subsubsection{Normally distributed environmental variables} \label{app:further-exp-details-supply-chain-normal-u}
The approximations to the knowledge gradient acquisition functions in \cref{eq:acqf-joint-kg-approx,eq:acqf-alternating-kg-fixed-approx,eq:acqf-alternating-kg-adjustable-approx,eq:acqf-twostep-kg-approx} are defined using an MC or qMC estimate of the expectation \(\E_\rvu[\cdot]\) over the environmental variables. Since qMC estimates are lower variance unbiased estimators, we continue to use a qMC estimate for the expectation over the normally distributed environmental variables \(\rvu\). Specifically, we generate the sample by first generating a Sobol' sample of \(N_u\) points in \([0, 1]^{d_u}\), and then pass these point-wise and coordinate-wise through the inverse CDF \(F_u^{-1}\) of the normal distribution. That is,
\begin{gather*}
    (\vz^1, \dots, \vz^{N_u}) \sim \mathrm{Sobol}([0, 1]^{d_u}, N_u), \\
    F_u(u') = F_{u,j}(u') = \mathbb{P}(\ervu_j \leq u') \quad \forall j, \\
    \forall i=1,\dots,N_u \, \forall j=1,\dots,4 \quad \evu_j^i = F_u^{-1}(\evz_j^i).
\end{gather*}
We use the implementation in BoTorch \citep{balandat2020botorch} and use the same method for the qMC estimates over \(\rvu\) when generating recommendations and estimating regret.

In addition to requiring a qMC estimate of the expectation over the normal environmental variables, we also choose to use normally distributed \(\rvu\) in the initial sample. For the other test problems the initial sample is a Sobol' sample in \([0, 1]^{d_x + d_y + d_u}\), rescaled to the bounds of the test problem in the case of the optical table.
For the supply chain problem, we instead independently generate a uniform Sobol' sample on \([0, 1]^{d_x + d_y}\) for \(\sX \times \sY\) and a normal Sobol' sample on \(\R^d_u\) for \(\sU\). These are concatenated to give a sample for \([0, 1]^{d_x + d_y} \times  \sR^{d_u}\).
Since the problem has constraints, we discard infeasible points from the initial design and continue to generate samples using the Sobol' sequences until sufficiently many feasible initial points have been sampled.

Finally, we observe that, unlike uniformly distributed environmental variables, normally distributed \(\rvu\) have no natural bounds. When optimizing the acquisition function, we use the \(1\%\) and \(99\%\) marginal quantiles to give bounds on \(\sU\).
When normalizing the input space before modeling with a GP, we use a box with edges one standard deviation either side of the mean. This is preferable to using the \(1\%\) and \(99\%\) quantiles which were used for the bounds, because the latter would concentrate the majority of probability mass of \(\rvu\) in a very small volume relative to the priors used for the length scales of the GP.

\subsubsection{Optimization of the acquisition function} \label{app:further-exp-details-supply-chain-constrained-optim}
The usual process we have followed for optimizing the acquisition function is explained in \cref{app:optimization-of-acqf}. We first generate promising initial samples by evaluating the acquisition function at a large number of points, then selecting a random sample of the best with Boltzmann sampling, always ensuring the best evaluated point is among the chosen ones. Finally, we optimize each of these using L-BFGS-B and select the best.

However, for the supply chain problem we must satisfy the constraint \(y_1 \leq x/20\), we must restrict to only 10 different \((s, S)\)-ordering policies for the raw chemical, and we must ensure the soy \(x\) and daily production \(y_1\) are chosen from among the finite set detailed in \cref{tab:supply-chain-variables} of the main text.
We achieve this with the following modifications to the process in \cref{app:optimization-of-acqf}. First, when generating `raw samples', we use the MCMC approach implemented in BoTorch \citep{balandat2020botorch} to generate a sample for over all but the \((s, S)\) variables. We then take the Cartesian product with the 10 possible values of \((s, S)\) and proceed with the usual Boltzmann sampling to choose promising candidates for the 10 restarts, again ensuring the best of the raw samples is one of the restarts.
Instead of L-BFGS-B, we use the SLSQP implementation in scipy to optimize these restarts subject to the constraint \(y_1 \leq x/20\). We keep \((s, S)\) fixed during this optimization.
Finally, the best of these optimized candidates is rounded to give the next sample location as follows. First the soy quantity \(x\) is rounded to the nearest 20, then the daily production \(y_1\) is clipped to the nearest integer in the range \([0, x/20]\).

\subsubsection{Modifications to the two-step algorithms}
The two-step knowledge gradient and random sampling algorithms generate a control policy in the first step which is then held fixed in the second step. For our other experiments, this policy is evaluated by running an optimization using L-BFGS-B to find the maximum of the final posterior mean from step 1. For the supply chain problem, we instead use an exhaustive search to choose the best \(\rvy = (y_1, s, S)\) for each \(\vu\).

However, the soy \(x\) is optimized in the second step, which affects the feasible region for \(y_1\). For convenience of the implementation, if the chosen \(y_1\) lies outside the feasible region, then we simply round it to the nearest integer in \([0, x/20]\).

\subsection{Fitting the GP Surrogate Model} \label{app:further-exp-details-gpsurrogate}

The Gaussian process \(f \sim \mathcal{GP}(\mu, k)\) used to model the unknown, expensive function \(h : \sX \times \sY \times \sU \to \R\) (or a restriction of \(h\) in the case of the two-step algorithms) uses a constant mean, \(\mu\), and a Mat\'ern-5/2 kernel, \(k\), with separate length scales for each input dimension (known as \emph{automatic relevance determination}, \citep{williams1995gp}).
The constant mean, length scales, output scale and standard deviation of the observation noise of the GP are collectively referred to as the hyperparameters of the GP.

The hyperparameters are fitted to the observations using maximum a posterior (MAP) estimation.
Gamma prior distributions are placed on the length scales, output scale and noise standard deviation, while the mean uses an improper, uninformative prior equivalent to a Gaussian with infinite variance.

The hyperparameters are updated at the start of each iteration, to incorporate the new observations.
The data is normalized to have zero mean and unit variance, and MAP estimates are generated for all hyperparameters to fit the normalized observations.
When using the GP in the acquisition function or when making recommendations, the GP is untransformed.
That is, if \(f_\text{norm} \sim \mathcal{GP}(\mu_\text{norm}, k_\text{norm})\) is the GP fitted to the normalized data, then \(f = a + b f_\text{norm}\) is the GP referred to in the main text, where \(a\) and \(b\) are the mean and sample standard deviation of the observations.

On problems where the observation noise is known to be zero, the observation noise variance, \(\sigma^2\), is fixed at \(10^{-8}\) instead of being estimated from the observations.
In our experiments, the observation noise variance is only fitted for the experiment in \cref{fig:results-noisy}.

The Gamma distributions \(\mathrm{Gamma}(\alpha, \beta)\) in \cref{tab:gp-hyperparameters} are parameterized in terms of a concentration parameter, \(\alpha\), and rate parameter, \(\beta\), and have density function
\begin{equation} \label{eq:app-gamma-pdf}
    p(z) = \frac{\beta^\alpha}{\Gamma(\alpha)} z^{\alpha - 1} e^{-\beta z}.
\end{equation}

\begin{table}[ht]
    \centering
    \begin{tabular}{ccp{7.8cm}}
        \toprule
        Hyperparameter & Prior distribution & Notes \\
        \midrule
        Constant mean, \(\mu\) & None & The mean is only fitted on the first iteration \\
        Length scales, \(\ell_1, \dots, \ell_d\) & \(\mathrm{Gamma}(3, 10)\) & A separate length scale is fitted for each dimension \\
        Output scale, \(s\) & \(\mathrm{Gamma}(2, 0.15)\) & \\
        Observation noise variance, \(\sigma^2\) & \(\mathrm{Gamma}(1.1, 0.05)\) & On problems where the noise is known to be zero, this is fixed to zero rather than being fitted \\
        \bottomrule
    \end{tabular}
    \caption{Prior distributions and fitting notes on the GP hyperparameters. Gamma distributions, \(\mathrm{Gamma}(\alpha, \beta)\), are specified in terms of their concentration, \(\alpha\), and rate, \(\beta\), parameters.}
    \label{tab:gp-hyperparameters}
\end{table}

\end{document}